\documentclass{article}
\usepackage{amsmath,amssymb,amsthm,mathtools}
\usepackage{xspace}

\theoremstyle{plain}
\newtheorem{theorem}{Theorem}
\newtheorem{lemma}{Lemma}
\newtheorem{proposition}{Proposition}
\newtheorem{corollary}{Corollary}
\theoremstyle{definition}

\usepackage{latex_macros}
\renewcommand{\figdir}{.}

\providecommand{\real}{\mathbb{R}}
\providecommand{\Exp}{\mathbb{E}}

\providecommand{\mycomment}[1]{}

\usepackage{fullpage}
\usepackage{xcolor}
\IfFileExists{fontawesome5.sty}{
\usepackage{fontawesome5}

}{

}

\usepackage[
colorlinks=true,
linkcolor=magenta,
citecolor=blue,
urlcolor=blue
]{hyperref}

\usepackage[capitalize,noabbrev,nameinlink]{cleveref}
\crefname{assumption}{Assumption}{Assumptions}
\Crefname{assumption}{Assumption}{Assumptions}

\usepackage{enumitem}

\newcommand{\Xvar}{\ensuremath{X}}
\newcommand{\usedim}{d}

\newcommand{\defn}{\coloneqq}

\newcommand{\Yvar}{Y}

\newcommand{\Bit}{V}
\newcommand{\lam}{\lambda}

\newcommand{\Denoise}{\mu}
\newcommand{\Info}{\mathsf{I}}

\newcommand{\Hinfo}{\mathsf{H}}

\newcommand{\UnmaskKL}{\Gamma_{\scaleto{\operatorname{umask}}{4pt}}}
\newcommand{\UnmaskCard}{\Gamma_{\scaleto{\operatorname{card}}{3pt}}}

\newcommand{\Law}{\mathcal{L}}

\newcommand{\KL}{\ensuremath{D_{\scaleto{\operatorname{KL}}{4pt}}}}

\newcommand{\Jtot}{\mathsf J_{\mathrm{tot}}}

\newcommand{\Score}{\mathsf{s}}

\newcommand{\Hblk}{\mathsf H}
\newcommand{\hfun}{\ensuremath{\mathsf{h}}}

\newcommand{\Zvar}{Z}
\newcommand{\Mask}{\ensuremath{M}}

\newcommand{\NewMask}[2]{\ensuremath{\Mask_{#2 \mid #1}}}
\newcommand{\NewXvar}[2]{\ensuremath{\Xvar_{#2 \mid #1}}}
\newcommand{\XvarForce}[1]{\Xvar_{#1 \mid i}^{(\ell)}}

\newcommand{\Zmask}[1]{\Zvar_{#1}}

\newcommand{\Nscore}{N}

\newcommand{\Zhat}{\ensuremath{\widehat{Z}}}

\newcommand{\Xhat}{\ensuremath{\widehat{X}}}
\newcommand{\XhatIt}[1]{\Xhat^{#1}}

\newcommand{\ShannonInfo}{\ensuremath{\operatorname{Info}}}
\newcommand{\ShanInfo}{\ShannonInfo}

\newcommand{\Prob}{\ensuremath{\mathbb{P}}}
\newcommand{\Qprob}{\ensuremath{\mathbb{Q}}}

\newcommand{\Ent}{\ensuremath{\operatorname{Ent}}}

\newcommand{\Partition}{\ensuremath{\mathcal{P}}}
\newcommand{\PartH}{\ensuremath{\Partition(\Hinfo)}}

\newcommand{\Card}{\ensuremath{\operatorname{Card}}}
\renewcommand{\Xhat}{\widehat{X}}

\newcommand{\Alphabet}{\ensuremath{\mathcal{A}}}

\newcommand{\weight}{w}

\newcommand{\jind}{i}
\newcommand{\Reveal}{\ensuremath{R}}

\newcommand{\Masked}{\ensuremath{\mathcal{M}}}
\newcommand{\Event}{\ensuremath{\mathcal{E}}}

\newcommand{\ind}{i}
\newcommand{\iter}{j}

\makeatletter
\let\MaskOriginalNewCommand\newcommand
\def\newcommand#1{%
\@ifundefined{\expandafter\@gobble\string#1}
{\MaskOriginalNewCommand{#1}}
{\renewcommand{#1}}}
\makeatother

\newcommand{\DenoiseHat}{\widehat{\Denoise}}
\newcommand{\Info}{\ShanInfo}

\newcommand{\Hinfo}{\mathsf{H}}
\newcommand{\Dinfo}{\mathsf{D}}
\newcommand{\Escore}{\mathsf{E}_{\scaleto{\operatorname{den}}{3pt}}}

\newcommand{\Amat}{\ensuremath{\mathbf{A}}}

\newcommand{\Law}{\mathcal{L}}

\newcommand{\KL}{\ensuremath{D_{\scaleto{\operatorname{KL}}{4pt}}}}

\newcommand{\Score}{\mathsf{s}}
\newcommand{\scorehat}{\widehat{\score}}
\newcommand{\score}{\Score}

\newcommand{\Hblk}{\mathsf H}
\newcommand{\ValFun}{\mathsf{V}}
\newcommand{\hfun}{\ensuremath{\mathsf{h}}}
\newcommand{\ufun}{\ensuremath{\mathsf{g}}}

\newcommand{\IntLeft}{\ensuremath{\mathcal I_{\scaleto{\operatorname{pre}}{3pt}}}}
\newcommand{\IntMid}{\ensuremath{\mathcal I_{\scaleto{\operatorname{trans}}{3pt}}}}
\newcommand{\IntRight}{\ensuremath{\mathcal I_{\scaleto{\operatorname{post}}{3pt}}}}

\newcommand{\hcard}{\ensuremath{\hfun^{\scaleto{\operatorname{card}}{4pt}}}}
\newcommand{\Hcard}{\ensuremath{\Hunmask^{\mathrm{card}}}}

\newcommand{\rhat}{\widehat r}

\newcommand{\Zvar}{Z}

\newcommand{\Nscore}{N}

\newcommand{\Zhat}{\ensuremath{\widehat{Z}}}

\newcommand{\ShannonInfo}{\ensuremath{\operatorname{Info}}}
\newcommand{\ShanInfo}{\ShannonInfo}

\providecommand{\Order}{\mathcal{O}}

\newcommand{\Prob}{\ensuremath{\mathbb{P}}}
\newcommand{\Qprob}{\ensuremath{\mathbb{Q}}}

\newcommand{\Ent}{\ensuremath{\operatorname{Ent}}}

\newcommand{\Partition}{\ensuremath{\mathcal{P}}}
\newcommand{\PartH}{\ensuremath{\PartComp(\Partition)}}
\newcommand{\PartComp}{\mathsf{C}_{\scaleto{\operatorname{\ugc}}{4pt}}}
\newcommand{\PartCompHat}{\widehat{\mathsf{C}}_{\scaleto{\operatorname{\ugc}}{4pt}}}

\newcommand{\PlainPartComp}{\ensuremath{\mathsf{C}}}

\newcommand{\Card}{\ensuremath{\operatorname{Card}}}
\newcommand{\Xhat}{\widehat{X}}

\newcommand{\iter}{j}

\newcommand{\bind}{k}
\newcommand{\Btot}{K}

\newcommand{\block}{b}
\newcommand{\jind}{i}

\newcommand{\Hhat}{\widehat{\Hinfo}_\numsam}
\newcommand{\HhatBind}[1]{\widehat{\Hinfo}_{#1, \numsam}}

\newcommand{\ugc}{\textsf{UGC}\xspace}
\newcommand{\dgc}{\textsf{DGC}\xspace}

\newcommand{\Ber}{\textsf{Ber}\xspace}
\newcommand{\Card}{\textsf{Card}\xspace}
\newcommand{\card}{\ensuremath{\operatorname{card}}}

\newcommand{\Hunmask}{\ensuremath{\mathsf{H}}}

\newcommand{\numobs}{\ensuremath{m}}

\newcommand{\Qfunup}[1]{\Qfun^{(#1)}}

\newcommand{\HackErr}{\widehat{r}_\numobs(\eta)}
\newcommand{\NewHackErr}{\widehat{r}_{\bind, \numobs}(\eta/\Btot)}

\newcommand{\Sinfo}{\ensuremath{\mathsf{S}}}

\newcommand{\Vhat}{\ensuremath{\widehat{V}}}
\newcommand{\Nhat}{\ensuremath{\widehat{N}}}

\newcommand{\qdens}{\mathsf{q}}

\newcommand{\rhohat}{\ensuremath{\widehat{\rho}}}
\newcommand{\PartHfine}{\mathsf{P}_{\scaleto{\mathrm{\ugc}}{4pt}}}

\newcommand{\Jtot}{J}

\newcommand{\fdens}{\ensuremath{\mathsf{f}}}

\newcommand{\UniK}{\ensuremath{\Uni_\Btot}}
\newcommand{\Uni}{\ensuremath{\mathcal{U}}}
\newcommand{\Lip}{\ensuremath{\operatorname{Lip}}}

\makeatletter
\long\def\@makecaption#1#2{
\vskip 0.8ex
\setbox\@tempboxa\hbox{\small {\bf #1:} #2}
\parindent 1.5em  
\dimen0=\hsize
\advance\dimen0 by -3em
\ifdim \wd\@tempboxa >\dimen0
\hbox to \hsize{
\parindent 0em
\hfil 
\parbox{\dimen0}{\def\baselinestretch{0.96}\small
{\bf #1.} #2
} 
\hfil}
\else \hbox to \hsize{\hfil \box\@tempboxa \hfil}
\fi
}
\makeatother

\newenvironment{researchquestion}
{\begin{center}\begin{minipage}{0.98\textwidth}}
{\end{minipage}\end{center}}
\let\newcommand\MaskOriginalNewCommand

\newcommand{\Uvar}{\ensuremath{U}}

\newcommand{\ZhatSmall}{\hat{\Zvar}}
\newcommand{\ZhatCard}{\ensuremath{\Zhat^{\mathrm{card}}}}

\newcommand{\Arand}[1]{\ensuremath{A^{#1}}}
\newcommand{\RandSub}{\Prob_{\scaleto{\operatorname{Ber}}{4pt}}}
\newcommand{\RandUmask}{\Prob_{\scaleto{\operatorname{umask}}{4pt}}}

\newcommand{\Geo}{\ensuremath{\operatorname{Geo}}}

\newcommand{\newrfinal}{T}
\newcommand{\newrinit}{{\rtime_0}}

\newcommand{\miss}{\star}

\newcommand{\MyMask}[1]{\Zvar^\miss_{#1}}

\newcommand{\PlainKer}{\ensuremath{\mathbb{K}}}

\newcommand{\Kexact}[2]{\PlainKer_{#1, #2}}
\newcommand{\Khat}[2]{\widehat{\PlainKer}_{#1, #2}}
\newcommand{\CompKernel}{\ensuremath{\mathsf C}}

\newcommand{\rtime}{t}
\newcommand{\odds}{\ensuremath{\psi}}

\newcommand{\logit}{\ensuremath{\varphi}}
\newcommand{\revodds}{\logit}

\newcommand{\revinv}{\logit^{-1}}

\newcommand{\DTC}{\ensuremath{\mathsf{DTC}}}
\newcommand{\TC}{\ensuremath{\mathsf{TC}}}

\newcommand{\DHW}{\ensuremath{\mathsf{DHW}}}
\newcommand{\TSE}{\ensuremath{\mathsf{TSE}}}

\newcommand{\ANNOY}{\log
  \Big(\frac{\odds(\newrfinal)}{\odds(\newrinit)}\Big)}

\newcommand{\HACKT}{\KL(\Prob_{\rtime_0} \| \Qprob_{\rtime_0}) +
  \CompletionDefect_{\CompKernel}(T)}
\newcommand{\CompletionDefect}{\ensuremath{\Gamma}}

\newcommand{\Elld}{\ell_\usedim}
\newcommand{\ThetaTil}{\ensuremath{\widetilde{\Theta}}}
\newcommand{\OmegaTil}{\ensuremath{\widetilde{\Omega}}}

\newcommand{\SB}{\PartComp}
\newcommand{\FP}{\PartHfine}

\newcommand{\myrho}{c_{p,q}}

\newcommand{\Zclean}[1]{\ensuremath{\Zvar^{(#1)}}}

\newcommand{\samind}{\ell} \newcommand{\numsam}{m}
\newcommand{\QhatSam}{\ensuremath{Q^{(\samind)}}}

\newcommand{\BOUNDARY}{ \KL(\Prob_{\rtime_0} \| \Qprob_{\rtime_0}) +
  \CompletionDefect_{\CompKernel}(T)}

\newcommand{\Yup}[1]{\ensuremath{Y^{(#1)}}}
\newcommand{\ashort}{a}
\newcommand{\Difun}{\ensuremath{\mathcal{I}}}

\newcommand{\Ker}{\ensuremath{\mathbb{K}}}
\newcommand{\IntStar}{\ensuremath{\Int_{\scaleto{\operatorname{full}}{4pt}}}}
\newcommand{\MyLen}{\operatorname{Len}}
\newcommand{\Int}{\mathcal{I}}

\newcommand{\Ratio}{\operatorname{Ratio}}
\newcommand{\ProbZ}{\ensuremath{\Prob_\Zvar}}

\newcommand{\deld}{\ensuremath{\delta_d}}
\newcommand{\sneg}{\ensuremath{s_{\scaleto{\operatorname{left}}{3pt}}}}
\newcommand{\spos}{\ensuremath{s_{\scaleto{\operatorname{right}}{3pt}}}}
\newcommand{\splus}{\spos}

\newcommand{\Lint}{\mathcal{I}}

\newcommand{\esize}{\ensuremath{s}}

\newcommand{\order}{\Order}
\newcommand{\Fstar}{F_\star}
\newcommand{\newrevinv}{R}
\newcommand{\kdim}{\ensuremath{k}}
\newcommand{\Bmat}{\mathbf{B}}

\newcommand{\Nstar}{\ensuremath{N^\star}}

\begin{document}

\begin{center}
{\Large\bfseries The data geometry of masking diffusion:
  \\ Certified-optimal schedules via unmasking growth complexity }

\vspace*{0.3in}

\begin{tabular}{c}
Martin J. Wainwright \\ \texttt{mjwain@mit.edu}
\end{tabular}

\vspace*{0.2in}
\begin{tabular}{c}
Lab for Information and Decision Systems \\
Statistics and Data Science Center \\
EECS and Mathematics, \\
Massachusetts Institute of Technology
\end{tabular}

\vspace*{0.25in}
\today
\vspace*{0.25in}

\begin{abstract}
We study masking diffusion for discrete sampling and introduce a
path-resolved measure of data geometry called the \emph{unmasking
growth complexity} ({\textsf{UGC}\xspace}).  Its local increments
directly control Kullback--Leibler (KL) discretization error, yielding
a unified analysis of Bernoulli-subset and fixed-cardinality unmasking
schemes.  In log-reveal-odds coordinates, this structure yields
optimized single-block and multi-block schedules, and quantifies the
gains from adapting computational effort to data geometry.  Crucially,
we show how {\textsf{UGC}\xspace} increments can be estimated from
samples via KL increments along coupled reveal trajectories.  This
leads to \emph{certified-optimal} samplers that achieve a prescribed
KL error with high probability and iteration complexity within a
constant factor of the corresponding oracle procedure.  Collapsing the
\ugc path yields the aggregate {\textsf{UGC}\xspace} mass, which
connects to classical multivariate dependence measures, and complexity
measures from previous analyses of discrete diffusion.  In the
fine-partition limit, the squared integral of the square-root
{\textsf{UGC}\xspace} density determines the sharp leading-order
optimal Euler discretization error.  Examples exhibit substantial
dimension-dependent gains over coarse schedules, including
$\widetilde{\Omega}(\sqrt{d})$ improvements achievable with a constant
number of adaptively placed blocks.
\end{abstract}
\end{center}



\section{Introduction}

The problem of sampling from a high-dimensional distribution is
fundamental in nature, and has a wide range of applications.
Efficient sampling algorithms underpin the utility of Monte Carlo
approximation~\cite{RobCase04,RubKroe08}; support uncertainty
quantification in Bayesian models~\cite{GelEtAl13,BroEtAl11}; and lie
at the heart of generative AI~\cite{RomEtAl22,CroEtAl23,YanEtAl25,
  CheEtAl24}.  Recent years have witnessed tremendous practical and
theoretical advances in the use of diffusion\footnote{In full
generality, the diffusion terminology is a misnomer, since not all
such models involve a diffusion process in the formal probabilistic
sense; this includes the unmasking models studied in this
paper. Nonetheless, we adopt the conventional terminology.}  sampling
algorithms, which generate samples via a sequence of ``denoising''
operations~\cite{SohEtAl15,SonErmo19,HoEtAl20,SonEtAl21,
  CheEtAl23c,CheEtAl23a,BenEtAl24}.  Initial work in the area focused
on sampling from continuous distributions in $\real^\usedim$, in which
case the denoising step corresponds to estimating a signal embedded in
Gaussian noise~\cite{HoEtAl20,SonEtAl21}.  A more recent line of work,
and the general focus of this paper, has focused on sampling from
discrete distributions, using various kinds of denoising
processes~\cite{HooEtAl21,AusEtAl21,CamEtAl22,LouEtAl24,
  ShiEtAl24b,SahooEtAl2024MDLM,CheEtAl25,DmiEtAl26}.

In this paper, we study the problem of sampling a random vector $\Zvar \in
(\Alphabet)^\usedim$, where $\Alphabet$ is a discrete alphabet, and
$\usedim$ is the ambient dimension. While various samplers have been
devised, we focus on \emph{unmasking samplers}, which traverse a path
from a fully unobserved vector, denoted by $\Zvar^* = (\star, \ldots,
\star)$ with $\star$ meaning masked or unobserved, back to a sample
$\Zvar \sim \ProbZ$ from the target distribution.  The path is
traversed via a sequence of unmasking operations, in which a subset of
the masked coordinates are revealed.  Unmasking algorithms differ in
the form and sequence by which these unmasking steps take place.
Unmasking and closely related mask-based generative samplers have
proven effective across a range of applications, including machine
translation~\cite{GhazvininejadEtAl2019MaskPredict}, image and video
synthesis~\cite{ChangEtAl2022MaskGIT,YuEtAl2023MAGVIT}, text-to-image
generation~\cite{ChangEtAl2023Muse}, speech
synthesis~\cite{WangEtAl2025MaskGCT}, language
modeling~\cite{SahooEtAl2024MDLM,NieEtAl2025LLaDA}, and protein
design~\cite{WangEtAl2024DPLM}.

\subsection{Overview}

This paper is motivated by two broad questions associated with
diffusion samplers, as articulated in our companion paper on Gaussian
diffusion~\cite{Wai26}, which we paraphrase here:
\begin{researchquestion}
{\bf{Q1:}} Can the performance of masked diffusion sampling be
explained and quantified, in some generality, by a measure tied to
data geometry?  \par\smallskip
{\bf{Q2:}} Is it possible to exploit a masked-diffusion measure of
data geometry to design, optimize and certify practical sampling
schemes?
\end{researchquestion}
Recent work has made substantial progress on both questions. Chen et
al.~\cite{CheEtAl25} analyze fixed-cardinality unmasking and derive an
exact information-profile representation of its expected KL
discretization error, with consequences for schedule design and bounds
based on total and dual total correlation.  Lavenant and
Zanella~\cite{LavZan25} derive an information-profile representation
of the factorization error for random-order unmasking, and study
optimal schedules in asymptotic scaling regimes. From a different
direction, Dmitriev et al.~\cite{DmiEtAl26} analyze discrete diffusion
via a continuous-time Markov chain (CTMC) formulation; for masking
diffusion, their modified $\tau$-leaping guarantees are governed by an
effective total correlation that can be substantially smaller than the
classical worst-case measures. Collectively, these results provide
important answers to {\bf{Q1}} and partial answers to {\bf{Q2}}, but
fall short of finite-sample guarantees for learning and certifying
geometry-adaptive schedules from data.

Our answer to {\bf{Q1}} is based on the \emph{unmasking growth
complexity} or \ugc for short.  It is a path-resolved measure of data
dependence along the reveal process. We show that its local increments
directly control KL discretization error, thereby obtaining a common
analysis of both fixed-cardinality unmasking and a natural
Bernoulli-subset variant. When collapsed over the full reveal path,
the \ugc complexity reduces to a coarse measure that closely connected
to classical multivariate dependence
measures~\cite{Wat60,Han75,Han78,{TonSpoEde94}}, as well as to the
effective total correlation appearing in the CTMC
analysis~\cite{DmiEtAl26}. Consequently, the resulting single-block
guarantees sharpen existing fixed-cardinality bounds while giving
Bernoulli unmasking guarantees comparable to those available for CTMC
samplers.

The main contributions of our paper are in the context of {\bf{Q2}},
and in particular, our use of the full \ugc complexity path to design
and analyze sampling algorithms that are certified-optimal. More
precisely, we show that \ugc complexity has a natural additive
structure along the reveal path, so that its local increments can be
estimated from samples and used to allocate computational effort where
the target distribution is most difficult to unmask. We exploit this
structure to derive and optimize blockwise schedules, construct
data-dependent samplers with high-probability KL certificates, and
finally, to optimize the block boundaries themselves. In the
fine-partition limit, the log-reveal-odds \ugc density emerges as the
intrinsic local geometry: its square-root integral governs the
limiting partition complexity and the sharp leading-order optimal
Euler discretization error.


\subsection{Related work}
\label{SecRelated}

So as to put our results in context, we now provide a broader
discussion of related work.  There are various types of discrete
diffusion models, all of which replace the additive Gaussian noise for
continuous-space diffusions by stochastic corruption kernels on
discrete state spaces.  Early work developed multinomial and more
general structured discrete diffusion
processes~\cite{HooEtAl21,AusEtAl21}, among them the absorbing-state
corruption that underlies a masking diffusion process.  Other
work~\cite{CamEtAl22} formulated discrete diffusion models using the
formalism of continuous-time Markov chains, with subsequent work
developing discrete analogues of score-based modeling in which the
reverse CTMC is parameterized through learned probability
ratios~\cite{LouEtAl24}.  Masked diffusion is a particularly important
subclass, in which coordinates are corrupted by replacement with a
distinguished mask symbol and generation proceeds by progressively
reconstructing masked coordinates~\cite{ShiEtAl24b,
  SahooEtAl2024MDLM}.

As described above, recent work has made progress on characterizing
the accuracy--parallelism trade-off in random-order unmasking with
fixed-cardinality subsets.  In particular, Li and Cai~\cite{LiCai25}
established information-theoretic convergence guarantees for parallel
masked-diffusion sampling.  Chen et al.~\cite{CheEtAl25} derived an
exact characterization of the expected KL divergence in terms of a
one-dimensional information profile; they also gave explicit sampling
guarantees involving classical multivariate dependence
measures~\cite{Wat60,Han75,Han78}, namely the total correlation
($\TC$) and dual total correlation ($\DTC$). A portion of our analysis
exploits their exact KL representation. Lavenant and
Zanella~\cite{LavZan25} derived an information-profile representation
for the factorization error of random-order unmasking, and then
studied optimal scheduling via a continuum scaling limit. We compare
their asymptotic schedule analysis more closely with our
fine-partition and Euler results following~\Cref{ThmFineEuler}.

The $\tau$-leaping method for continuous-time Markov chains (CTMC)
originates in stochastic chemical kinetics, where its consistency and
approximation errors have been studied
extensively~\cite{Gil01,RatEtAl05,Li07,AndEtAl11}.  For discrete
diffusion, recent work has developed convergence guarantees for
$\tau$-leaping under both standard and absorbing corruption
processes~\cite{RenEtAl25b,LiaEtAl25a,LiaEtAl25b,DmiEtAl26}, as well
as higher-order variants of CTMC discretization~\cite{RenEtAl25a}.
Among CTMC analyses, most related to our work is the paper of Dmitriev
et al.~\cite{DmiEtAl26}, who analyze a modified $\tau$-leaping sampler
for a general class of CTMC discrete diffusion models.  For masking
diffusion, they give guarantees in terms of a measure that they call
effective total correlation, which turns out to be closely related to
the coarse \ugc complexity that we introduce.  In contrast to
CTMC-based analyses, we study Bernoulli and fixed-cardinality
unmasking; our coarse \ugc bounds yield matching guarantees for these
direct schemes.  The coarse \ugc complexity is also connected to
several multivariate dependence measures, including $\TC$, $\DTC$, and
the finer-grained measure of Tononi et al.~\cite{TonSpoEde94}.  We
develop these connections in~\Cref{AppComplements}.

Finally, in our companion paper~\cite{Wai26} on Gaussian diffusion
sampling, we introduced the denoising growth complexity ($\dgc$), a
pathwise complexity that controls KL discretization error, and whose
log-scale \dgc density specifies optimal sampling schedules.  Despite
the substantial differences between Gaussian and masking
diffusion, the two theories exhibit a remarkable degree of
parallelism, suggesting a common underlying principle linking
denoising growth, KL discretization error, and optimal scheduling.


\section{Overview: unmasking growth complexity and sampling}
\label{SecOverview}

So as to orient the reader, we begin by providing a high-level overview
of the main ideas and quantities underlying our analysis.  We first
describe the canonical unmasking process.  Using it, we define the
unmasking growth complexity (\ugc), which measures how informational
difficulty is distributed along the reveal path. We then pass to
log-reveal-odds coordinates, where this complexity is represented by a
density whose geometry determines how sampling effort should be
allocated along the path.

Our goal here is primarily conceptual: to explain the distinction
between coarse and geometry-aware schedules, to illustrate the data
geometry of the \ugc density on several examples, and to preview how
data-driven adaptive partitioning can exploit this geometry. Precise
forms of our guarantees for samplers, estimation procedures, and
supporting technical results are developed in the subsequent sections.


\subsection{Unmasking reveal process}
\label{SecRevealProcess}
Given a discrete alphabet $\Alphabet$, our goal is to sample a
$\usedim$-dimensional random vector $\Zvar \in \Alphabet^\usedim$ with
distribution $\ProbZ$.  Underlying the sampling algorithms that we
analyze is an unmasking stochastic process $\{ \Xvar_t, t \in [0,1 ]
\}$, where $\Xvar_t \in \big(\Alphabet \cup \{\miss \} \big)^\usedim$
and the new symbol $\miss$ denotes a masked or unobserved entry. At
time $t = 0$, we have $\Xvar_0 = (\miss, \ldots, \miss)$ almost
surely, whereas at time $t = 1$, we have $\Xvar_1 \sim \Prob_\Zvar$,
where $\Prob_\Zvar$ is the target distribution from which we would
like to sample.

The evolution of $\Xvar_t$ along the path is controlled by a sequence
$\{ \Uvar_\jind \}_{\jind=1}^\usedim$ of i.i.d.
$\operatorname{Unif}[0, 1]$ random variables, independent of $\Zvar$.
These variables define the \emph{reveal subset sequence}, indexed by
$t \in [0,1]$, via
\begin{subequations}
\begin{align}
\label{EqnMaskingProcess}
\Reveal_t & \defn \big\{ \jind \in [\usedim] \mid \Uvar_\jind \leq t
\big\}, \quad \mbox{and its complement} \quad \Reveal_t^c \defn
      [\usedim] \setminus \Reveal_t.
\end{align}
The subset $\Reveal_t$ corresponds to the subset of coordinates
$\jind$ for which the hidden value $\Zvar_\jind$ has been revealed by
time $t$.  More formally, the \emph{unmasking process} at time $t$ is
given by
\begin{align}
\label{EqnDefnRevealProcess}
X_t & \defn \Big( \Zmask{\Reveal_t}, \; \MyMask{\Reveal^c_t} \Big) \;
\in \; \big(\Alphabet \cup \{\miss \} \big )^\usedim \qquad \mbox{for
  $t \in [0,1]$,}
\end{align}
\end{subequations}
where $\Zmask{\Reveal_t} = (\Zvar_\jind, \jind \in \Reveal_t)$ are the
variables revealed by time $t$, and $\MyMask{\Reveal^c_t} = (\miss,
\ldots, \miss)$ is a sub-vector of missing values in positions indexed
by $\Reveal_t^c$.  By construction, the unmasking
process~\eqref{EqnDefnRevealProcess} defines a family of probability
distributions $\{ \Prob_t, t \in [0,1] \}$, where $\Prob_0$ denotes a
degenerate distribution with all its mass on the masked sequence,
whereas $\Prob_1 \equiv \Prob_\Zvar$ is the probability distribution
of the target variable $\Zvar$.


\subsection{Unmasking growth complexity and its geometry}
\label{SecUGC}
For a reveal time $\rtime \in [0,1]$, we define
the \emph{Bernoulli unmasking gain}
\begin{subequations}
\begin{align}
\label{EqnDefnHfun}
\hfun(\rtime) &\defn \; \sum_{i = 1}^{\usedim} \Info\big( \Zvar_i;
X_\rtime \mid i \in \Masked(X_\rtime) \big),
\end{align}
where $\Masked(X_\rtime) \subseteq \{1, \ldots, \usedim \}$ is the
subset of indices that are masked at reveal time $\rtime$, and $\Info$
denotes the (conditional) mutual information. The function $\hfun$ has
a denoising interpretation, since when $i \in \Masked(X_\rtime)$, all
other coordinates in $X_\rtime$ are revealed independently with
probability $\rtime$; the conditional mutual information term for the
$i^{th}$ coordinate measures how much these revealed variables reduce
uncertainty about $\Zvar_i$.

Moreover, the derivative $\hfun'$ turns out to have a simple
information-theoretic representation: in particular, it corresponds to
the second derivative
\begin{align}
  \label{EqnHderivativeMI}
  \hfun'(t) & = - \frac{d^2}{d t^2} \Info(\Zvar; \Xvar_t),
\end{align}
where $\Info(\Zvar; \Xvar_t)$ denotes the mutual information between
the target vector $\Zvar \in \Alphabet^\usedim$ and the partially
unmasked vector $\Xvar_t \in \big( \Alphabet \cup \{ \miss \}\big)^d$
at reveal time $t \in (0,1)$.  See equation~\eqref{EqnMaskInfoRate} in
the proof of~\Cref{ThmMaster} for the underlying details.

In terms of this mutual information derivative, the \emph{unmasking
growth complexity} assigns a non-negative number to any sub-interval
$[p,q]$ of the unit interval $[0,1]$ via
\begin{align}
\label{EqnDefnUGC}
\Hinfo(p, q) & \defn \int_p^q \rtime (1 - \rtime) \hfun'(\rtime) \, d
\rtime \quad \mbox{for any $0 \leq p < q \leq 1$.}
\end{align}
\end{subequations}
Based on the identity~\eqref{EqnHderivativeMI}, we see that
$\Hinfo(p,q)$ is a weighted integral of the information curvature.

The essential feature of this complexity measure is that it is a
\emph{path-resolved quantity}: rather than assigning a single
complexity to the target distribution $\ProbZ$, it assigns a
complexity to every interval of the reveal path. Our analysis shows
how these local increments directly control the associated KL
discretization error.  Moreover, for any triple $0 \leq p < q < r \leq
1$, we have the additivity property
\begin{align}
\label{EqnHinfoAdditive}  
  \Hinfo(p,r) & = \Hinfo(p, q) + \Hinfo(q, r).
\end{align}
This additivity allows the global sampling problem to be decomposed
into local pieces, whose complexities can be estimated and controlled
separately, and is what ultimately enables geometry-adaptive
refinement of the sampling schedule.

By collapsing the path-resolved complexity to the full interval
$[0,1]$, we obtain the \emph{aggregate \ugc mass} given by $\Hinfo(0,
1) = \int_0^1 \rtime (1 - \rtime) \hfun'(\rtime) d \rtime$.
Interestingly, this aggregate quantity has several connections to
classical measures of multivariate dependence.  By a non-trivial
argument (see the proof of~\Cref{PropSimpleRelation}), it turns out to
be equivalent to a multivariate dependence measure, first introduced
by Tononi, Sporns and Edelman~\cite{TonSpoEde94} in 1994, which
spawned a rich line of work
(e.g.,~\cite{BuzZam12a,BuzZam12b,OlbEtAl08,BarBucBul09,RosEtAl19,VarEtAl23}).
A separate argument shows its close relation to the effective total
correlation introduced by Dmitriev et al.~\cite{DmiEtAl26} in their
study of CTMC unmasking algorithms.  Moreover, the value $\Hinfo(0,1)$
can be upper bounded by classical measures of multivariate
dependence~\cite{Wat60, Han75,Han78}, known as total correlation and
dual total correlation.  We elaborate upon these and other connections
in~\Cref{AppComplements}.

\paragraph{Role of the log-reveal-odds density:}
The natural coordinate for progress along the unmasking path is not
reveal time $\rtime$ nor its logarithm, but rather the log-reveal-odds
$\lam = \revodds(\rtime) \defn \log(\rtime/(1-\rtime))$.  Our analysis
gives one-step KL bounds governed by the multiplicative change in
reveal odds; consequently, equal increments in $\lam$ correspond to
equal multiplicative changes in reveal odds, and place the early and
late stages of unmasking on a common scale.  Expressing the \ugc
complexity in this coordinate yields a density $\qdens$ that localizes
the informational difficulty of sampling along the path. As our
results show, regions where $\qdens(\lam)$ is large require finer
resolution.  More precisely, the \emph{log-reveal-odds \ugc density
function} is given by
\begin{align}
\label{EqnDefnQdens}
\qdens(\lam) & \defn r^2 (1 - r)^2 \hfun'(r) \qquad \mbox{where $r =
  \revinv(\lam) \defn \frac{e^{\lam}}{1 + e^\lam}$.}
\end{align}
By construction, the \ugc increment is given by $\Hinfo(p, q) =
\int_{\revodds(p)}^{\revodds(q)} \qdens(\lam) \, d\lam$. \\

The high-level conclusions of this paper take a particularly simple
form in terms of $\qdens$.  We focus on the canonical reveal interval
$\big[ \tfrac{1}{\usedim}, \; 1 - \tfrac{1}{\usedim} \big]$, which
corresponds under the log-reveal-odds transformation to the symmetric
interval $[-\Elld, \Elld]$ with $\Elld \defn \log(\usedim - 1)$.
\begin{subequations}
\begin{itemize}[leftmargin=1em, topsep=2pt, itemsep=2pt,
  parsep=0pt, partopsep=0pt]
\item Single-block unmasking schemes use a single multiplier across
  the reveal path and hence ignore the location of the \ugc mass.
  Their iteration complexity is governed by the \emph{coarse or
  aggregate \ugc complexity}
  \begin{align}
\label{EqnDefnSingBlock}    
    \SB & \defn 2 \Elld \int_{-\Elld}^{\Elld} \qdens(\lam) \, d \lam.
  \end{align}
 The measure $\SB$ depends only on the total \ugc mass over the
 relevant portion of the path, and not on how this mass is
 distributed.  This coarse measure connects directly to past work: it
 enables us to show that a single-block Bernoulli unmasking sampler
 matches the guarantees given for CTMC unmasking~\cite{DmiEtAl26}, while
 sharpening guarantees from past work on fixed-cardinality
 sampling~\cite{CheEtAl25}.  See~\Cref{CorSingle} and~\Cref{PropSimpleRelation}
 for details.
\item Geometry-aware schemes exploit the local \ugc geometry, taking
  finer steps where $\qdens$ is large and coarser steps where it is
  small. Crucially, we show that $\ugc$ increments can be estimated
  from samples, and use these estimates to choose the partition
  adaptively from data (see~\Cref{PropDataSingle}).  Under progressive
  refinement, the resulting complexity converges from above to the
  \emph{fine-partition complexity}
    \begin{align}
      \label{EqnDefnFinePart}
      \FP & \defn \left (\int_{-\Elld}^{\Elld} \sqrt{\qdens(\lam)} d
      \lam\right)^2,
    \end{align}
 We establish non-asymptotic convergence rates and high-probability
 finite-sample guarantees for the resulting samplers;
 see~\Cref{ThmCertifiedMulti,ThmFineEuler} for details. A related
 square-root structure for an information profile appears in the
 asymptotic analysis of fixed-cardinality schedules by Lavenant and
 Zanella~\cite{LavZan25}; we discuss this connection in more detail
 following~\Cref{ThmFineEuler}.
\end{itemize}
\end{subequations}

In summary, the potential gain from data-dependent optimization of the
sampling schedule is governed by the ratio
\begin{align}
\label{EqnDefnRatio}
  \Ratio(\ProbZ) & \defn \frac{\SB}{\FP} \; = \; \frac{2 \Elld
    \int_{-\Elld}^{\Elld} \qdens(\lam) \, d \lam}{ \left
    (\int_{-\Elld}^{\Elld} \sqrt{\qdens(\lam)} d \lam\right)^2} \;
  \stackrel{(i)}{\geq} \; 1,
\end{align}
where the lower bound (i) follows from the Cauchy--Schwarz inequality.
Equality holds if and only if $\qdens$ is constant, whereas larger
values of $\Ratio(\ProbZ)$ arise when the \ugc mass is unevenly
distributed or sharply concentrated along the path.


\subsubsection{Geometric behavior of \ugc density}
\label{SecGeometry}

Since the \ugc density $\qdens$ captures the essential structure of
the problem, it is useful to compute and plot it for three different
ensembles, each chosen to illustrate a qualitative aspect of our
theory.

\paragraph{Noisy repeated bit:}  We begin with a very simple
example.  For any $p \in [0,1]$, we use $V \sim \Ber(p)$ to denote a
Bernoulli random variable with $\Prob(V = 1) = p$ and $\Prob(V = 0) =
1 - p$.  In this ensemble, we construct a binary random vector $\Zvar
= (\Zvar_1, \ldots, \Zvar_\usedim) \in \{0, 1 \}^\usedim$ by first
drawing $\Uvar \sim \Ber(1/2)$, and then setting
\begin{align*}
\Zvar_\jind & = \Uvar \oplus W_\jind \qquad \mbox{for each $\jind = 1,
  \ldots, \usedim$,}
\end{align*}
where each $W_\jind$ is an independent $\Ber(\eta)$-variable for some
$\eta \in [0, 1/2]$, and $\oplus$ denotes addition modulo two.  By
construction, we have $\Prob(\Zvar_\jind = \Uvar) = 1
- \eta$, hence our use of the term ``noisy repeated bit''.
\begin{figure}[h]
\begin{center}
\begin{tabular}{ccc}
\widgraph{0.31\textwidth}{\figdir/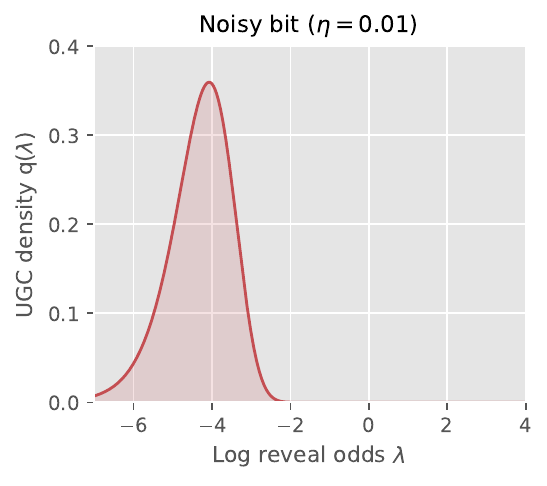}
&
\widgraph{0.31\textwidth}{\figdir/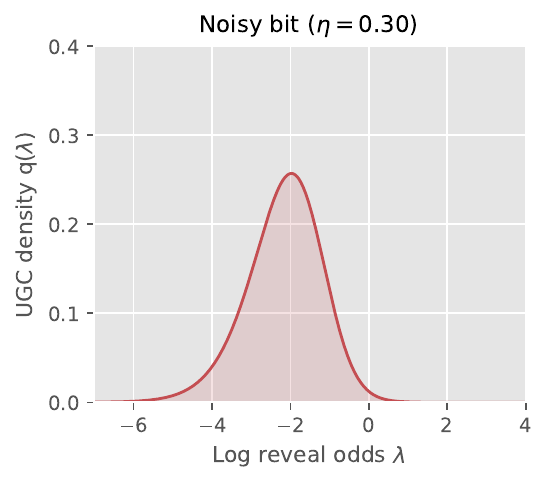}
&
\widgraph{0.31\textwidth}{\figdir/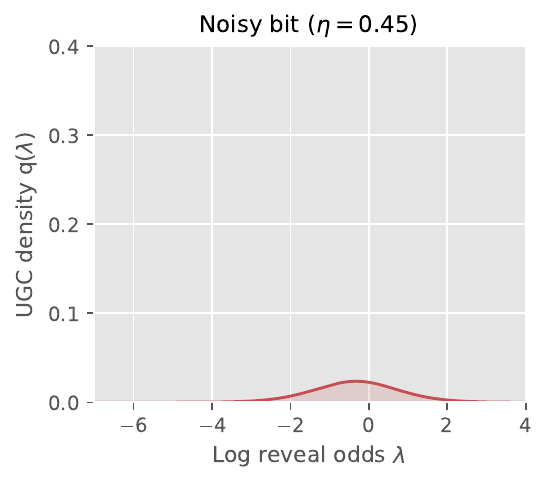}
\\
(a) & (b) & (c) 
\end{tabular}
\caption{Plots of the log-reveal-odds \ugc density $\qdens$ for the noisy
  repeated-bit ensemble for $\eta \in \{0.01, 0.30, 0.45 \}$.  The
  density has a sharp peak for $\eta= 0.01$, and then flattens out and
  shifts to the right as $\eta \rightarrow 0.5$.  Up
  to two digits of accuracy, we have $\Ratio(\Prob_Z) \in \{4.51,
  2.16, 1.65 \}$ in panels (a), (b), and (c), respectively.  For this
  example, we have $\Ratio(\Prob_Z) \asymp \log(\usedim)$ as the
  dimension grows.}
\label{FigNoisyRepeated}
\end{center}
\end{figure}

\Cref{FigNoisyRepeated} gives plots of the \ugc
density~\eqref{EqnDefnQdens} in dimension $\usedim = 128$, and flip
probabilities $\eta \in \{0.01, 0.30, 0.45 \}$.  For $\eta = 0.01$,
the density has a sharp peak at a low value of the log-reveal-odds
parameter $\lam$.  As $\eta$ increases, the dependence among the
coordinates of the multivariate random vector $Z \in \{0,1 \}^d$
decreases; it takes a higher reveal level before information about the
hidden bit $U$ is revealed, as reflected by the rightward shift in the
mode of the \ugc density.  In parallel with this rightward shift, the
entire density collapses towards zero, so that $\Hinfo(0,1)
\rightarrow 0$ as $\eta \rightarrow 1/2$.

\paragraph{Discrete mixture models:}  Our second ensemble corresponds
to the discrete analog of a mixture model.  More precisely, given a
fixed set of $M$ cluster centers, say $C^1, \ldots, C^M\in \{0,
1\}^d$, we generate a binary random vector $\Zvar = (\Zvar_1, \ldots,
\Zvar_d) \in \{0, 1 \}^d$ with components $\Zvar_i = C^J_i \oplus
\xi_i$ for $i = 1, \ldots, \usedim$, where the random cluster index
$J$ is drawn uniformly at random from $\{1, \ldots, M \}$ and the
$\xi_i \sim \Ber(\eta)$ are drawn i.i.d., and independently of the
cluster index.  This procedure generates an $M$-component mixture
distribution with uniform weight $1/M$ on each component.

\begin{figure}[h]
  \begin{center}
    \begin{tabular}{ccc}
      \widgraph{0.31\textwidth}{\figdir/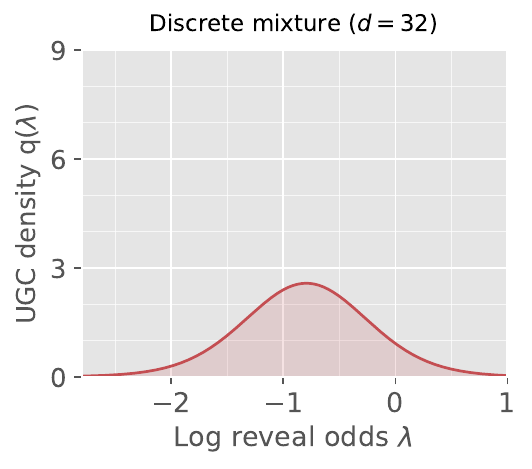} &
      \widgraph{0.31\textwidth}{\figdir/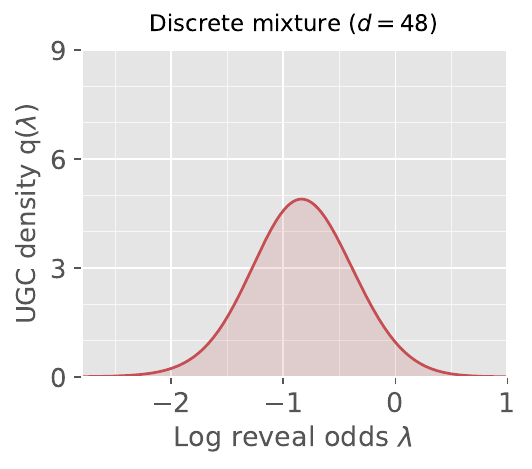} &
      \widgraph{0.31\textwidth}{\figdir/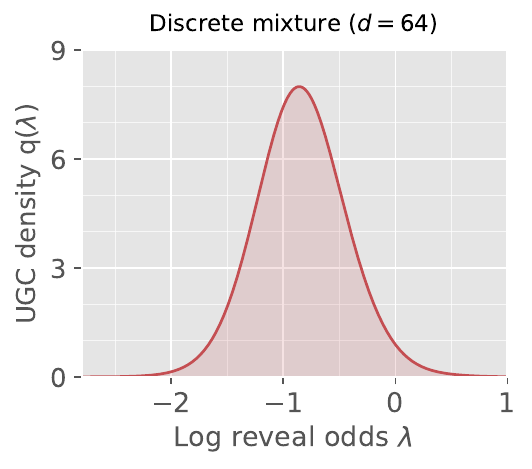} \\
      (a) & (b) & (c)
\end{tabular}
\caption{Plots of the log-reveal-odds \ugc density $\qdens$ for the discrete
  mixture model.  For each dimension $d$, the latent vector $\Zvar \in
  \{0,1\}^d$ is generated by selecting uniformly from $M = 2^{d / 4}$
  cluster centers, and then flipping each coordinate independently
  with probability $\eta = 0.02$.  Panels (a)--(c) correspond to $d
  \in \{32, 48, 64\}$, and hence $M \in \{256, 4096, 65536\}$,
  respectively.  Increasing the dimension produces a progressively
  sharper and taller spike.  Up to two digits of accuracy, we have
  $\Ratio(\Prob_Z) \in \{2.19, 3.13, 3.85 \}$ in panels (a), (b), and
  (c), respectively.  For this example, it can be shown that
  $\Ratio(\Prob_Z) = \ThetaTil(\sqrt{\usedim})$ as the dimension
  grows.}
\label{FigDiscreteMixture}
\end{center}
\end{figure}

We studied this family for varying dimensions $d$ divisible by four,
number of mixture components $M = 2^{d/4}$, and cluster centers sampled
independently and uniformly from the Boolean hypercube $\{0,
1\}^\usedim$.  \Cref{FigDiscreteMixture} shows a sequence of \ugc
densities for this family, with the dimension $d$ and the number $M =
2^{d/4}$ of mixture components increasing as we move from left to right.
The density is unimodal in all cases; as the dimension increases, the
peak location remains fixed but its height increases.  It can be shown
that the ratio~\eqref{EqnDefnRatio} scales as $\Ratio(\Prob_Z) =
\widetilde{\Theta}(\sqrt{\usedim})$, so that there are significant
gains from using geometry-aware stepsizes.  Moreover, since the \ugc
density is concentrated around a single increasingly narrow transition
window, a partition with only $\Btot = 3$ blocks---one covering this
transition and two covering its tails---achieves the fine-partition
complexity up to logarithmic factors.  See~\Cref{SecAlgorithmic}
and~\Cref{FigDiscreteMixture} for further discussion of these factors.

\paragraph{Hierarchical mixture model:}  The previous examples, being
quite simple in nature, all led to unimodal $\ugc$ densities.  We now
consider a hierarchical model in which multiple peaks emerge, and the
peaks have interesting meanings in terms of the underlying data
geometry.
\begin{figure}[htb!]
\begin{center}
\begin{tabular}{@{}cc@{}}
\widgraph{0.60\textwidth}{\figdir/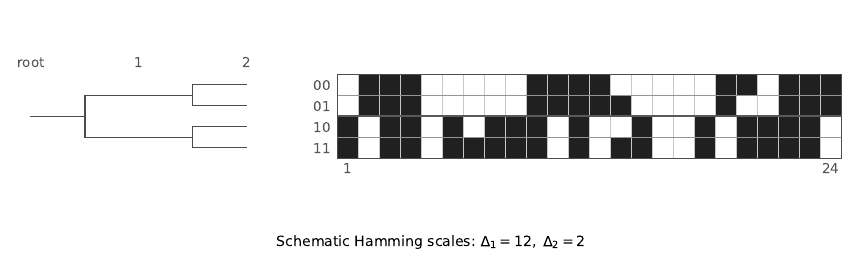}
&
\widgraph{0.28\textwidth}{\figdir/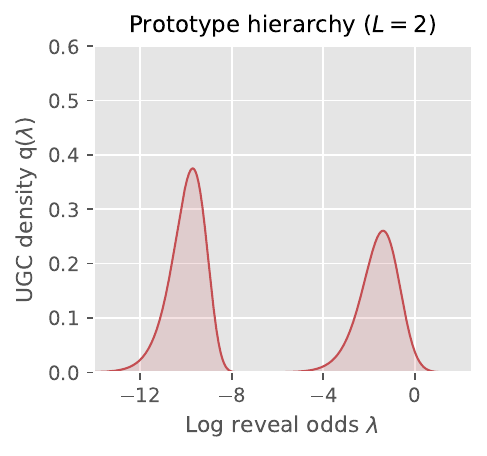}
\\
(a) & (b) \\
\widgraph{0.60\textwidth}{\figdir/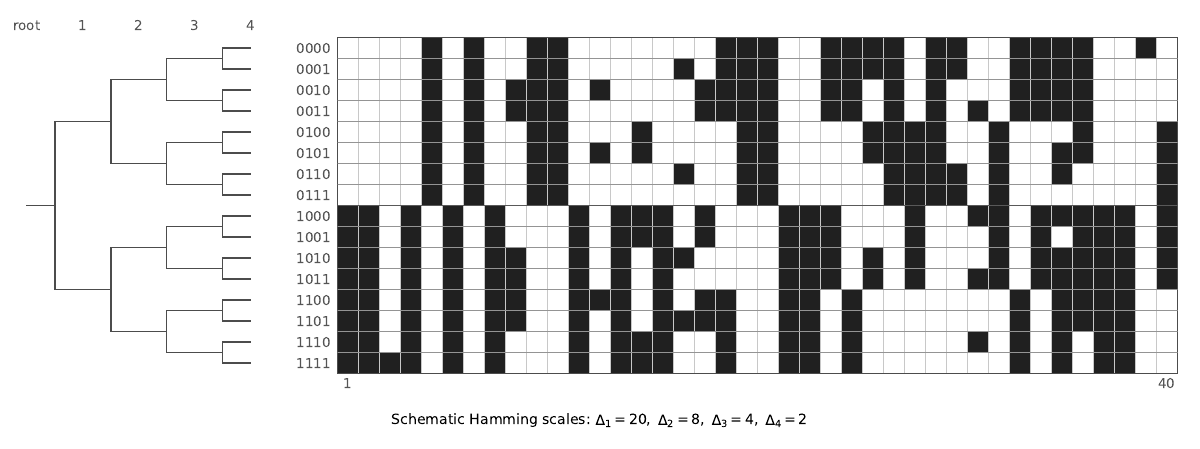} &
\widgraph{0.28\textwidth}{\figdir/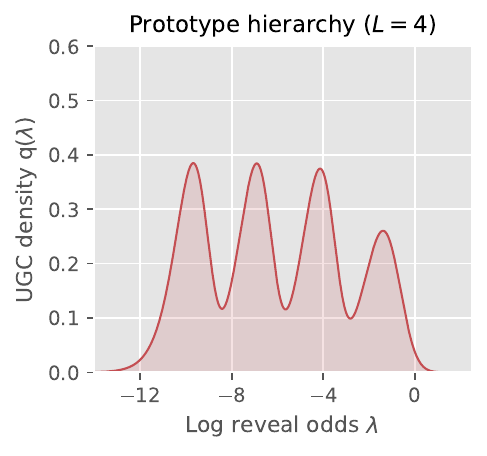} \\
(c) & (d)
\end{tabular}
\caption{Hierarchical binary-prototype mixtures and their log-reveal-odds
  \ugc densities.  The rows correspond to hierarchy depths $L = 2$ and
  $L = 4$.  Left column panels (a) and (c) show the leaf prototypes,
  while right column panels (b) and (d) show the corresponding
  \ugc densities $\qdens$.}
\label{FigHierarchicalMixture}
\end{center}
\end{figure}

For an integer $L \geq 2$, we construct a model that places mass on
exactly $2^L$ prototypes in the Boolean hypercube $\{0,1\}^d$, say
$\{V^1, \ldots, V^M \}$, where $M = 2^L$.  Each prototype is defined by
a path down a binary tree of depth $L$; see the left column
of~\Cref{FigHierarchicalMixture} for some examples.  The geometry of
the prototypes is related to the tree structure in the following way:
for any tree level $\ell \in \{1, \ldots, L\}$, a pair of prototypes
$V^a$ and $V^b$ whose corresponding tree paths first diverge at level
$\ell$ of the tree are separated by Hamming distance $\Delta_\ell$,
for some pre-specified set of distances $\{\Delta_\ell \}_{\ell
  = 1}^{L}$, and total dimension $d = \sum_{\ell = 1}^{L}
\Delta_\ell$. Panel (a) of~\Cref{FigHierarchicalMixture} shows a toy
example with $L = 2$, $(\Delta_1, \Delta_2) = (12, 2)$, and dimension
$d = 12 + 2 = 14$.  It has $M = 2^L = 4$ prototype vectors in
$\{0,1\}^{14}$, assigned the index labels $\{00, 01, 10, 11 \}$,
plotted as rows in the corresponding black-white matrix.  Consider the
pair of prototypes indexed by $00$ and $11$ respectively; they diverge
at level $\ell = 1$ of the tree, and their Hamming distance is equal to
$\Delta_1 = 12$.  Similarly, the prototypes indexed by $00$ and $01$
diverge at level $\ell = 2$ of the tree, and have Hamming distance
$\Delta_2 = 2$.  Panel (c) illustrates a larger construction with $L =
4$, still in a toy dimension $d = 34$ for illustrative purposes.

Panels (b) and (d) show the \ugc densities for the $L = 2$ and $L = 4$
models, using larger separations and dimensions: the $L = 2$ model has
dimension $\usedim = 32776$ and $(\Delta_1, \Delta_2) = (32768, 8)$,
whereas the $L = 4$ model has dimension $\usedim = 34952$ and
$(\Delta_1, \Delta_2, \Delta_3, \Delta_4) = (32768, 2048, 128, 8)$.
Note how the \ugc density for the $L = 2$ model has two distinct peaks
in log-reveal odds.  The level $\ell = 1$ ambiguity is resolved at an
earlier reveal time, since the corresponding prototypes differ by
Hamming distance $\Delta_1 = 32768$; the second peak corresponds to
the later reveal time at which the level $\ell = 2$ ambiguity,
corresponding to Hamming separation $\Delta_2 = 8$, is resolved.
Similar comments apply to the $L = 4$ peaks for the \ugc density in
panel (d).


\subsubsection{From coarse to fine geometry}
\label{SecAlgorithmic}

Let us describe how we exploit finer-grained \ugc geometry for
algorithmic purposes.  A single-block scheme uses one geometric
multiplier across the full log-reveal-odds interval, and its
performance is governed by the coarse complexity $\SB$.  In order to
exploit local \ugc geometry, we instead partition the interval
$[-\Elld, \Elld]$ into $\Btot$ blocks and allow a different geometric
multiplier on each block.  For a partition $\Partition = \{
[\lam_\bind, \lam_{\bind + 1}] \}_{\bind = 0}^{\Btot - 1}$, define the
block lengths $\Sinfo_\bind = \lam_{\bind + 1} - \lam_\bind$ and the
associated \ugc increments $\Hinfo_\bind =
\int_{\lam_\bind}^{\lam_{\bind + 1}} \qdens(\lam) \, d\lam$.
Optimizing the geometric multipliers over these blocks leads to the
partition complexity
\begin{subequations}
\begin{align}
  \label{EqnInterPartComp}
  \PlainPartComp(\Partition) & \defn \left( \sum_{\bind = 0}^{\Btot -
    1} \sqrt{\Sinfo_\bind \Hinfo_\bind} \right)^2.
\end{align}
\Cref{PropTrackMulti} shows that this quantity governs the iteration
complexity of the resulting $\Btot$-block sampler.  Although the
optimal multipliers depend on the unknown increments $\Hinfo_\bind$,
these increments can be estimated from clean samples.  We use these
estimates to choose the partition and its multipliers from data,
leading to the certified guarantees; see~\Cref{PropDataSingle}
and~\Cref{ThmCertifiedMulti}.  Progressive refinement interpolates
between the coarse single-block complexity and the fine-partition
limit~\eqref{EqnDefnFinePart}.  In particular, in~\Cref{ThmFineEuler},
we prove that
\begin{align}
  \label{EqnFPEquiv}
 \inf_{\Partition} \PlainPartComp(\Partition) \; = \; \Big(
 \int_{-\Elld}^{\Elld} \sqrt{\qdens(\lam)} d \lam \Big)^2 \qquad
 \mbox{where the infimum is over partitions of $[-\Elld, \Elld]$.}
\end{align}
\end{subequations}
Thus, increasing the number of blocks allows the sampler to exploit
increasingly fine features of the local \ugc geometry, and ultimately
achieve the fine-grained complexity based on $\sqrt{\qdens}$.

\paragraph{Illustrative example:}
\Cref{FigDiscreteMixBlocks} illustrates this coarse-to-fine
interpolation for the discrete-mixture model
underlying~\Cref{FigDiscreteMixture}.  For $\usedim = 64$, the
single-block scheme has complexity $\SB \approx 65.0$.  A three-block
partition that allocates finer resolution around the narrow
high-density region reduces the partition complexity to approximately
$23.5$, compared with the fine-partition value $\FP \approx 16.9$.

\begin{figure}[h!]
\begin{center}
\begin{tabular}{@{}ccc@{}}
\widgraph{0.31\textwidth}{\figdir/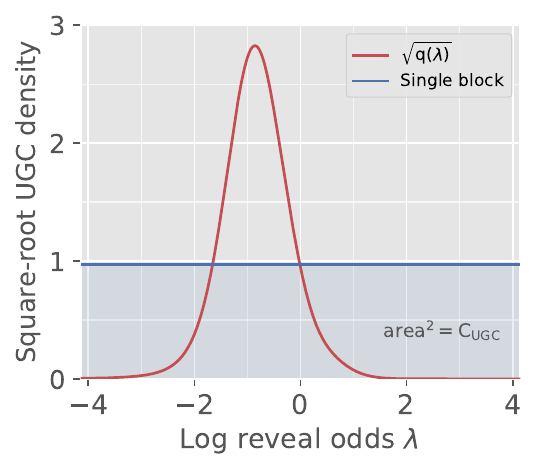}
& \hspace*{0.2in} &
\widgraph{0.31\textwidth}{\figdir/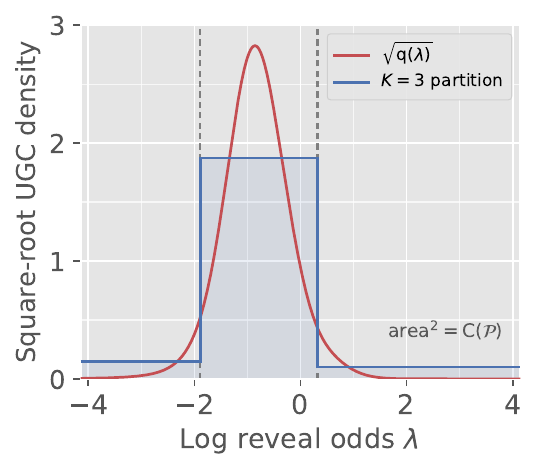}
\\
(a) && (b)
\end{tabular}
\caption{Block geometry for the $d = 64$ random discrete mixture with
  $M = 2^{d / 4}$ and $\eta = 0.02$.  Both panels show the
  square-root \ugc density $\sqrt{\qdens}$ versus the log-reveal-odds
  $\lam$.  Panel (a) uses a single geometric multiplier, yielding
  $\SB \approx 65.0$.  Panel (b) uses the illustrated
  $\Btot = 3$ block partition, yielding
  $\PlainPartComp(\Partition) \approx 23.5$.  The fine-partition
  complexity for this example is $\FP \approx 16.9$.}
\label{FigDiscreteMixBlocks}
\end{center}
\end{figure}

The three-block scheme devotes more iterations to the narrow
high-\ugc region and fewer to the flatter tails.  Even this coarse
adaptation captures most of the available improvement: the complexity
drops from approximately $65.0$ to $23.5$, substantially closing the
gap to the fine-partition value $16.9$.  This example illustrates the
basic role of adaptive partitioning: computation is allocated according
to where the \ugc mass is concentrated, rather than uniformly across
the reveal path.


\section{Unmasking samplers and KL discretization error}
\label{SecMain}

We now turn to the description and analysis of some simple unmasking
samplers.  At a high level, they are based on discretized
approximations to the unmasking reveal process $\{ \Xvar_t, t \in [0,
  1] \}$ from equation~\eqref{EqnDefnRevealProcess}.  Recalling that
$\Masked(x)$ denotes the set of masked coordinates in the vector $x$,
their updates involve the \emph{single-site posterior distributions}
\begin{align}
\label{EqnDefnDenoise}  
\Denoise_i(a, x) & \defn \Prob\big( Z_i = a \mid Z_j = x_j \quad
\mbox{for all $j \notin \Masked(x)$} \big), \quad \mbox{defined for
  each $i \in \Masked(x)$,}
\end{align}
where $a \in \Alphabet$ and $x \in \big( \Alphabet \cup \{\miss \}
\big)^\usedim$.  We use $\Denoise$ to denote the full collection of
these distributions, and often refer to it as the \emph{denoiser}.  It
summarizes residual uncertainty in $\Zvar$ given the observation $x$.

In practice, the denoiser $\Denoise$ is estimated based on samples
from $\Prob_Z$ as follows.  Given a clean sample $\Zvar \sim \ProbZ$,
we can generate training examples by independently masking random
subsets of coordinates.  We then predict the original value of each
masked coordinate from the resulting partially observed sample using
cross-entropy loss.  With this set-up, the population-optimal
predictors are given by the posterior distributions
$\Denoise_i(\mathord\cdot, x)$, and the training procedure generates
estimates $\DenoiseHat_i$ for each $i = 1, \ldots, \usedim$.  For
simplicity, we describe the algorithms using the exact denoiser
$\Denoise$.  Using a learned denoiser $\DenoiseHat$ contributes an
additional term to the KL error bound in a standard way; see
equation~\eqref{EqnAddNoise} following the statement
of~\Cref{ThmMaster} for details.

\subsection{Bernoulli and fixed-cardinality unmasking}
\label{SecSamplers}

We now describe the two standard unmasking samplers that we analyze in
this paper. At a high level, they involve two basic operations:
choosing a subset of indices and then, conditional on this subset,
making a random choice of values for the revealed variables using the
denoiser $\Denoise$.  The samplers that we analyze differ only in how
the random subset is chosen: the \emph{Bernoulli sampler} uses coin
flips, thereby generating a set with a random cardinality, whereas the
\emph{fixed-cardinality sampler} chooses a random subset with fixed
cardinality.  Both samplers are used in practice, and our theory gives
a unified analysis of both in terms of \ugc functions.

Given a subset $A \subseteq [\usedim]$ of coordinates, both samplers
make use of the conditional distribution
\begin{subequations}
\begin{align}
\label{EqnDefnRandUmask}
\RandUmask(X' \mid X, A) & \defn \Big( \bigotimes_{i \notin
  \Masked(X)} \delta_{X_i} \Big) \otimes \left( \bigotimes_{i \in A}
\Denoise_i(\mathord\cdot, X) \right) \otimes \left( \bigotimes_{i \in
  \Masked(X) \setminus A} \delta_\star \right), \qquad \mbox{for $A
  \subseteq \Masked(X)$,}
\end{align}
which fixes all observed variables (i.e., indices $i \notin
\Masked(X)$); fixes all missing variables not in $A$; and makes a random
update to variables in $A$.  Moreover, for an iteration count $N \geq
1$, both samplers are based on a \emph{reveal-time grid}
\begin{align}
\label{EqnDefnRevealGrid}  
0 < \rtime_0 < \rtime_1 < \cdots < \rtime_N < 1.
\end{align}
\end{subequations}
The samplers are initialized with a random choice $\XhatIt{0} \sim
\Qprob_{\rtime_0}$, and then generate a sequence $\{\XhatIt{\iter}
\}_{\iter = 0}^{N}$ moving forward in time across the
grid~\eqref{EqnDefnRevealGrid}.

\subsubsection{Bernoulli unmasking sampler}

We begin by describing the Bernoulli unmasking sampler. It chooses
subsets according to the binomial distribution
\begin{align}
\label{EqnDefnRandSub}
\RandSub(A \mid X, \beta) & \defn \beta^{\card(A)} (1 -
\beta)^{\card(\Masked(X) \setminus A)} \qquad \mbox{with support $A
  \subseteq \Masked(X)$,}
\end{align}
and then updates entries according to the conditional
distribution~\eqref{EqnDefnRandUmask}.  More precisely, given the
initialization $\XhatIt{0}$ and reveal-time
grid~\eqref{EqnDefnRevealGrid}, it generates the sequence
\begin{subequations}
\label{EqnBerUnmask}  
\begin{alignat}{2}
\label{EqnRandSub}
\mbox{\bfseries Draw random binomial subset:} & \quad \Arand{\iter}
\sim \RandSub\big( \mathord\cdot \mid \XhatIt{\iter}, \beta_\iter
\big), \qquad \mbox{where } \beta_\iter \defn \frac{\rtime_{\iter + 1}
  - \rtime_\iter}{1 - \rtime_\iter}, \\
\label{EqnRandUnmask}
\mbox{\bfseries Random unmasking:} & \quad \XhatIt{\iter + 1} \sim
\RandUmask\big( \mathord\cdot \mid \XhatIt{\iter}, \Arand{\iter}
\big),
\end{alignat}
\end{subequations}
for iterations $j = 0, 1, \ldots, N-1$.


\subsubsection{Fixed-cardinality unmasking sampler}

The fixed-cardinality sampler is very similar, except that it replaces
the binomial subset choice~\eqref{EqnDefnRandSub} with a black-box
that draws a uniformly random subset of a specified size.  More
precisely, the fixed-cardinality sampler only admits iteration counts
$N < \usedim - 1$, and it requires that each reveal time be $\usedim$-aligned,
meaning that $\rtime_\iter = \esize_\iter / \usedim$ for integers $1
\leq \esize_0 < \cdots < \esize_N \leq \usedim - 1$.  It initializes
$\XhatIt{0}$ with exactly $\usedim - \esize_0$ masked coordinates.  Then,
at iteration $\iter = 0, \ldots, N - 1$, the current state
$\XhatIt{\iter}$ has exactly $\usedim - \esize_\iter$ masked coordinates.
\begin{subequations}
\begin{alignat}{2}
\label{EqnRandSubUni}
\mbox{\bfseries Draw fixed-cardinality subset:} & \qquad \Arand{\iter}  \sim
\operatorname{Unif} \Big\{ A \subseteq \Masked(\XhatIt{\iter})
\,\Big|\, \card(A) = \esize_{\iter +1} - \esize_\iter \Big\}. \\
\mbox{\bfseries Random unmasking:} & \quad \XhatIt{\iter + 1} \sim
\RandUmask\big( \mathord\cdot \mid \XhatIt{\iter}, \Arand{\iter}
\big),
\end{alignat}
\end{subequations}
For future reference, we note that underlying the fixed-cardinality
sampler is a de-Poissonized version of the unmasking process
$\{\Xvar_t, t \in [0, 1] \}$.  It is defined only at $\usedim$-aligned
times $t = \esize/d$, and the associated variable
\begin{align}
  \label{EqnDefnDePoisson}
  X_t \sim \Prob_t \qquad \mbox{has exactly $\esize$ revealed
    coordinates,}
\end{align}
for a subset of cardinality $\esize$ chosen uniformly at random.  Note
that we are overloading our notation here, but the specific version of
$\Prob_t$ being used will be clear from context.


\subsubsection{Initialization and completion steps}
\label{SecCanonical}
Both samplers involve initialization and completion steps.  In our
analysis, the initialization cost is measured by the KL divergence
\begin{subequations}
  \begin{align}
\label{EqnDefnInitCost}    
\KL(\Prob_{\newrinit} \| \Qprob_{\newrinit}) \qquad \mbox{where
  $\Xvar_\newrinit \sim \Prob_\newrinit$ and $\XhatIt{0} \sim
  \Qprob_{\newrinit}$.}
  \end{align}
Here $\Xvar_\newrinit$ is generated by the standard reveal process for
the Bernoulli sampler~\eqref{EqnDefnRevealProcess}, and from its
de-Poissonized analogue~\eqref{EqnDefnDePoisson} for the
fixed-cardinality sampler.  Moreover, both samplers involve a final
completion step, based on a \emph{completion kernel}
$\CompKernel_\newrfinal(\mathord\cdot \mid x)$ that maps each partially
revealed terminal state $x \in \big( \Alphabet \cup \{\miss\}
\big)^{\usedim}$ to a probability law on fully specified vectors in
$\Alphabet^{\usedim}$.  We denote the resulting completed output by
$\Zhat$ and its law by $\Prob_{\Zhat}$.  The defect of any completion
kernel is measured by
\begin{align}
\label{EqnDefnCompletionDefect}
\CompletionDefect_{\CompKernel}(\newrfinal) & \defn \Exs_{X_T}\left[
  \KL\left( \Law(\Zvar \mid X_T) \,\middle\|\,
  \CompKernel_\newrfinal(\mathord\cdot \mid X_\newrfinal) \right)
  \right].
\end{align}
\end{subequations}
While these terms appear in our bounds, they are not dominant terms.
For example, the following completion kernel $\CompKernel$ has defect
equal to zero, and uses only $\card(\Masked(X_\newrfinal))$ serial
updates.  Fix an ordering of the indices in $\Masked(X_\newrfinal)$,
and then sample from the conditional distributions $\Denoise_i$ in
sequence to fill in the missing entries.  With the endpoint
$\newrfinal = \frac{d-1}{d}$, we have $\card(\Masked(X_\newrfinal)) =
1$ for the exact cardinality sampler and
$\Exs[\card(\Masked(X_\newrfinal))] = 1$ for the Bernoulli sampler, so
that the computational cost of implementing this completion kernel is
minimal.  Similarly, for initialization, it is easy to initialize the
fixed-cardinality sampler with exactly one unmasked entry, so that the
initialization cost vanishes for $\newrinit = 1/d$.  The bounds below
are stated for a general interval $[\newrinit, \newrfinal]$ and
arbitrary initialization-completion choices, but we later specialize
to the interval $\big[ \tfrac{1}{\usedim}, \: 1 - \tfrac{1}{\usedim}
  \big]$, which we refer to as the \emph{canonical reveal time
interval}.


\subsection{Unified analysis}
\label{SecUnified}

We are now set up to state a single theorem that provides a unified
analysis of both the Bernoulli and fixed-cardinality unmasking
samplers.  Our original definition~\eqref{EqnDefnHfun} of the
Bernoulli denoising gain $\hfun$ emphasizes the connection to the
reveal process.  Here, in order to make transparent the connections
with fixed-cardinality sampling, we make use of the following equivalent
representation:
\begin{subequations}
\begin{align}
\label{EqnHfunRandSub}  
  \hfun(\rtime) & = \sum_{\ind = 1}^{\usedim} \Exs_{A_{\ind,\rtime}}
  \big[ \Info(\Zvar_\ind; \Zvar_{A_{\ind,\rtime}}) \big],
\end{align}
where $A_{\ind,\rtime} \subseteq [\usedim] \setminus \{\ind \}$ is a
random subset obtained by including each coordinate of $[\usedim]
\setminus \{ \ind \}$ independently with probability $\rtime$.
See~\Cref{SecHalt} for the proof of this alternative representation.
The analogue for the fixed-cardinality sampler is given by the
\emph{fixed-cardinality unmasking gain} coefficients
\begin{align}
  \label{EqnDefnHcardGain}
\hcard_j & \defn \sum_{\ind = 1}^{\usedim} \Exs_{B_{i,j}}
\big[\Info(\Zvar_\ind; \Zvar_{B_{i,j}}) \big] \qquad \mbox{where
  $B_{i,j}$ is uniform over $j$-cardinality subsets of $[\usedim]
  \setminus \{\ind\}$.}
  \end{align}
\end{subequations}
These coefficients are defined for $j = 1, \ldots, \usedim-1$, with
$\hcard_0 \defn 0$.

These two notions lead to the two versions of \emph{\ugc growth
complexity}, given by
\begin{subequations}
\begin{align}
\label{EqnDefnUGCNew}
\mbox{{\underline{Bernoulli:}}} \qquad \Hinfo(p,q) & \defn \int_p^q
\rtime (1 - \rtime) \hfun'(\rtime) \, d \rtime \qquad \mbox{for pairs
  $0 \leq p < q \leq 1$, and} \\
\label{EqnDefnHcardUGC}
\mbox{{\underline{Exact cardinality:}}} \qquad
\Hcard(a, b) & \defn \sum_{j = a + 1}^{b - 1} \left(
\frac{j}{\usedim} \right) \left( 1 -
\frac{j}{\usedim} \right) \big\{ \hcard_j - \hcard_{j - 1} \big\},
\qquad 0 \leq a < b \leq \usedim.
\end{align}
\end{subequations}

Our main theorem bounds the KL error of both the Bernoulli and
fixed-cardinality unmasking samplers in terms of increments of these
respective \ugc functions.  It applies to any reveal-time
grid~\eqref{EqnDefnRevealGrid}, and the analysis shows why the
\emph{reveal odds function} $\odds(\rtime) \defn \frac{\rtime}{1 -
  \rtime}$ is natural.

\mygraybox{
\begin{theorem}
\label{ThmMaster}
For any grid~\eqref{EqnDefnRevealGrid}, the \Ber-unmasking sampler
produces output $\Zhat$ such that
\begin{subequations}
\begin{align}
\label{EqnMaster}
\KL\left( \Prob_Z \,\middle\|\, \Prob_{\Zhat} \right) & \leq
\sum_{\iter = 0}^{N - 1} \left\{ \frac{\odds(\rtime_{\iter + 1})}
    {\odds(\rtime_\iter)} - 1 \right\} \Hinfo(\rtime_\iter,
    \rtime_{\iter + 1}) + \BOUNDARY.
\end{align}
Moreover, for any $\usedim$-aligned grid $\rtime_\iter =
\frac{a_\iter}{\usedim}$, the \Card-unmasking sampler produces output
$\ZhatCard$ such that
\begin{align}
\label{EqnMasterCard}
\KL\left( \Prob_Z \,\middle\|\,
  \Prob_{\ZhatCard} \right) & \leq \sum_{\iter = 0}^{N - 1}
  \left\{ \frac{\odds(\rtime_{\iter + 1})}
    {\odds\big( \rtime_\iter + \tfrac{1}{\usedim} \big)} - 1
    \right\} \Hcard(a_\iter, a_{\iter + 1}) + \KL(\Prob_{\rtime_0} \|
    \Qprob_{\rtime_0}) + \CompletionDefect_{\CompKernel}(T).
\end{align}
\end{subequations}
\end{theorem}
}

\noindent We prove these two claims
in~\Cref{SecProofThmMasterBer,SecProofThmMasterExact}, respectively.
The proofs are relatively short, and follow the same template of
starting from an exact representation of the KL discretization error
over the interval, and then bounding it in terms of $\Hinfo$ or
$\Hcard$ respectively.  For the Bernoulli sampler, we first derive the
exact KL representation, from which the upper bound~\eqref{EqnMaster}
follows from an elementary argument.  On the other hand, the
fixed-cardinality guarantee~\eqref{EqnMasterCard} makes use of the
exact representation of the KL discretization error, derived by Chen
et al.~\cite{CheEtAl25}; in particular, the coefficients $\hcard_j$
are rescaled versions of their information coefficients.  In contrast
to our upper bounds, neither form of the exact KL error has simple
additive structure in terms of the log-reveal odds.  This structure in
our upper bounds turns out to be very useful: as we discuss in the
sequel (see~\Cref{CorSingle} and~\Cref{AppComplements}), a simple
stepsize choice in the bound~\eqref{EqnMasterCard} gives an immediate
sharpening of the fixed-cardinality sampling guarantees given in the
paper~\cite{CheEtAl25}.  More generally, these upper bounds lend
themselves naturally to developing and analyzing $\Btot$-block
sampling schemes (see~\Cref{SecTracking}).


\paragraph{Estimated denoisers:}
If the sampler replaces each exact single-site posterior
$\Denoise_\ind(\mathord\cdot, x)$ by an estimate
$\DenoiseHat_\ind(\mathord\cdot, x)$, while leaving the random-subset
law and reveal probabilities unchanged, then the right-hand side of
equation~\eqref{EqnMaster} acquires the additional additive term
\begin{align}
  \label{EqnAddNoise}
\Escore(\DenoiseHat) & \defn \sum_{\iter = 0}^{N - 1} \beta_\iter
\Exs\Big[ \sum_{\ind \in \Masked(X_{\rtime_\iter})} \KL\left(
  \Denoise_\ind(\mathord\cdot, X_{\rtime_\iter}) \,\middle \|\,
  \DenoiseHat_\ind(\mathord\cdot, X_{\rtime_\iter}) \right) \Big]
\qquad \mbox{where $\beta_\iter \defn \frac{\rtime_{\iter + 1} -
    \rtime_\iter}{1 - \rtime_\iter}$.}
\end{align}
Here the expectation is under the true reveal-process law of
$X_{\rtime_\iter}$.  A similar comment applies to the \Card-unmasking
sampler.  We establish this fact as part of the proof
in~\Cref{SecProofThmMasterBer}.

\paragraph{Relation between \ugc functions:}

By inspection, it is clear that the $\Ber$-\ugc and $\Card$-\ugc
functions, as defined in equation~\eqref{EqnDefnUGC} and
equation~\eqref{EqnDefnHcardUGC} respectively, are closely related.
This intuition can be formalized as follows.  For any $0 \leq p < q
\leq 1$, define the multinomial triple $(U_0, U_1, U_2) \sim
\operatorname{Multinomial}\left( \usedim + 1;\, p,\, q - p,\, 1 - q
\right)$.  We then have the equivalence
\begin{align}
\label{EqnBerCard}
\Hinfo(p, q) & = \frac{\usedim}{\usedim + 1} \Exs \left[ \Hcard(A_U,
  B_U) \right],
\end{align}
where $A_U \defn \max\{ U_0 - 1, 0 \}$, and $B_U \defn \min\{ U_0 +
U_1, \usedim \}$, and we set $\Hcard(a, b) = 0$ whenever $a \geq b$.
Thus, the $\Ber$-\ugc over $[p, q]$ is the average of $\Card$-\ugc
quantities over randomized cardinality endpoints.
See~\Cref{SecProofCard2Ber} for the proof of the
identity~\eqref{EqnBerCard}.


\subsection{A single-block guarantee and its consequences}

By inspection, the bounds~\eqref{EqnMaster} and~\eqref{EqnMasterCard}
are very well-suited to unmasking-time grid points
$\{\rtime_\iter\}_{\iter = 0}^{N-1}$ for which the reveal odds
$\odds(\rtime_\iter) = \rtime_\iter/(1 - \rtime_\iter)$ evolve in a
geometric way.  In this section, we use~\Cref{ThmMaster} to derive
guarantees for this particularly simple choice of stepsizes.

\subsubsection{Guarantees for a single geometric block}
\label{SecSingle}

For a parameter $\rho > 0$, consider the sequence
\begin{align}
\label{EqnDefnGeoRho}
\makebox[7em][l]{\(\boldsymbol{\Geo(\rho):}\)} \odds(\rtime_{\iter
  + 1}) & = \min \Big \{(1 + \rho) \odds(\rtime_\iter),
\odds(\newrfinal) \Big \}.
\end{align}
When the unmasking algorithm is run with this geometric sequence, we
refer to it as the \emph{$\Geo(\rho)$}-unmasking algorithm.

\mygraybox{
\begin{corollary}[Single-block guarantee for a reveal-odds geometric scheme]
\label{CorSingle}
Given an iteration budget $N \geq \ANNOY$, the $\Geo(\rho)$-unmasking
sampler with multiplier $\rho \defn \exp \big \{ \frac{1}{N}
\log(\odds(\newrfinal)/\odds(\newrinit))\big \} - 1$ produces
completed output $\Zhat$ such that
\begin{align}
\label{EqnSingle}
\KL( \Prob_Z \| \Prob_{\Zhat}) & \leq \frac{2 \Hinfo(\newrinit,
  \newrfinal)}{N} \ANNOY + \KL(\Prob_{\rtime_0} \| \Qprob_{\rtime_0})
+ \CompletionDefect_{\CompKernel}(T).
\end{align}
\end{corollary}
}
\noindent This result is an immediate consequence of~\Cref{ThmMaster},
and we include the proof here.  It makes essential use of the
additivity property~\eqref{EqnHinfoAdditive} of $\Hinfo$.
\begin{proof}
By construction, the schedule traverses the interval $[\newrinit,
\newrfinal]$ in exactly $N$ steps, and moreover, we have
$\frac{\odds(\rtime_{\iter+1})}{\odds(\rtime_\iter)} - 1 \leq \rho$
at each round.  Applying the bound~\eqref{EqnMaster}
from~\Cref{ThmMaster} yields
\begin{align*}
\KL( \Prob_Z \| \Prob_{\Zhat}) & \leq \underbrace{\rho \sum_{\iter =
    0}^{N - 1} \Hinfo(\rtime_\iter, \rtime_{\iter +
    1})}_{\stackrel{(i)}{=} \rho \Hinfo(\newrinit, \newrfinal)} +
\KL(\Prob_{\rtime_0} \| \Qprob_{\rtime_0}) +
\CompletionDefect_{\CompKernel}(T) \\
& \stackrel{(ii)}{\leq} \frac{2 \Hinfo(\newrinit, \newrfinal)}{N} \; \log
\Big(\frac{\odds(\newrfinal)}{\odds(\newrinit)} \Big) +
\KL(\Prob_{\rtime_0} \| \Qprob_{\rtime_0}) +
\CompletionDefect_{\CompKernel}(T),
\end{align*}
where step (i) uses the additivity property~\eqref{EqnHinfoAdditive},
and step (ii) follows since $\exp(u) - 1 \leq 2 u$ for $u \in [0,1]$,
applied with $u = \frac{1}{N} \log \Big(
\frac{\odds(\newrfinal)}{\odds(\newrinit)} \Big)$.
\end{proof}

\paragraph{Exact cardinality sampler:}
An analogous result holds for the fixed-cardinality uniform sampler,
using a slightly different geometric schedule. Given a pair of integers
$a_0 < A$ between $1$ and $\usedim - 1$, we define $t_0 =
\tfrac{a_0}{\usedim}$ and $T = \tfrac{A}{\usedim}$.  Given an
iteration budget $N \geq \log \big(
\frac{\odds(\newrfinal)}{\odds(\newrinit + \tfrac{1}{\usedim})} \big)$,
starting from $a_0$, we recursively choose $a_{j + 1}$ as the largest
integer in $\{a_j + 1, \ldots, A\}$ such that
\begin{subequations}
\begin{align*}
\odds\big( \frac{a_{j + 1}}{\usedim} \big) \leq (1 + \rho) \odds\left(
\frac{a_j + 1}{\usedim} \right) \qquad \mbox{where $\rho \defn \exp
  \Big\{ \frac{1}{N} \log \big(
  \frac{\odds(\newrfinal)}{\odds(\newrinit + \tfrac{1}{\usedim})}
  \big) \Big\} - 1$.}
\end{align*}
With these choices, running the fixed-cardinality sampler yields
completed output $\ZhatCard$ such that
\begin{align*}
\KL\left( \Prob_Z \,\middle\|\, \Prob_{\ZhatCard} \right) & \leq
\frac{2 \Hcard(a_0, A)}{N} \log \Big(
\frac{\odds(\newrfinal)}{\odds(\newrinit + \tfrac{1}{\usedim})} \Big)
+ \KL(\Prob_{\rtime_0} \| \Qprob_{\rtime_0}) +
\CompletionDefect_{\CompKernel}(T).
\end{align*}
\end{subequations}
The proof uses the bound~\eqref{EqnMasterCard} from~\Cref{ThmMaster},
and then applies the same reasoning as the proof of~\Cref{CorSingle}.

\subsubsection{Some consequences}
At a high level, \Cref{CorSingle} and its analogue for
fixed-cardinality sampling show that the iteration complexity when
using a single geometric multiplier protocol is governed by the
\emph{aggregate \ugc mass} functionals $\Hinfo(0,1)$ and $\Hcard(0,
\usedim)$, respectively, for the Bernoulli and fixed-cardinality
versions.  These relations allow us to recover and sharpen several
existing results on unmasking samplers.  As we show
in~\Cref{AppComplements}, the aggregate \ugc complexities associated
with Bernoulli and fixed-cardinality unmasking are equivalent up to
universal constants, and they are closely related to the complexity
governing the CTMC guarantees of Dmitriev et al.~\cite{DmiEtAl26}.
Consequently, the single-block guarantee in~\Cref{CorSingle} yields
comparable guarantees for these different unmasking mechanisms, while
sharpening bounds based only on total correlation and dual total
correlation.  We defer these comparisons, together with further
information-theoretic aspects of aggregate \ugc,
to~\Cref{AppComplements}.



\subsection{Why the \ugc path matters}
\label{SecObscure}

The aggregate \ugc complexity is a coarse measure, relative to the
full log-reveal-odds density that captures the full \ugc path. It is
useful to consider a stylized pair of distributions, based on some
primitives analyzed in the paper~\cite{DmiEtAl26}, that makes this
discrepancy very explicit: while their aggregate complexities are
identical, the \ugc path structure is extremely different, and as we
will see in the next section, the resulting optimal algorithms are
very different.  Recall that the log-reveal-odds \ugc density is given by
\begin{align}
  \label{EqnEquivQdens}
\qdens(\lam) & = r^2 (1 - r)^2 \hfun'(r) \qquad \mbox{where $r =
  \revinv(\lam) = \frac{e^\lam}{1 + e^\lam}$.}
\end{align}

  \begin{figure}[h!]
\begin{center}
\begin{tabular}{ccc}
\widgraph{0.38\textwidth}{\figdir/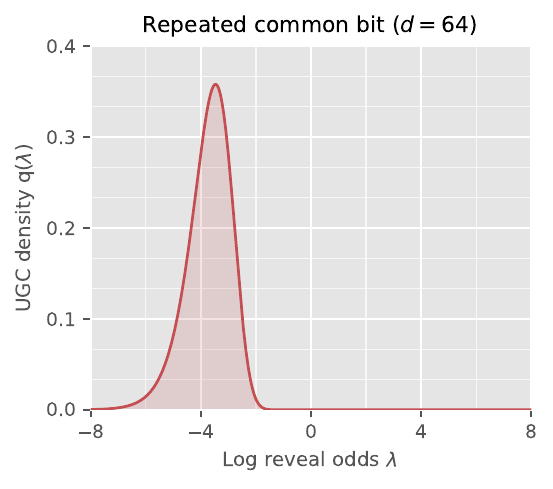}
& \hspace*{.1in} &
\widgraph{0.38\textwidth}{\figdir/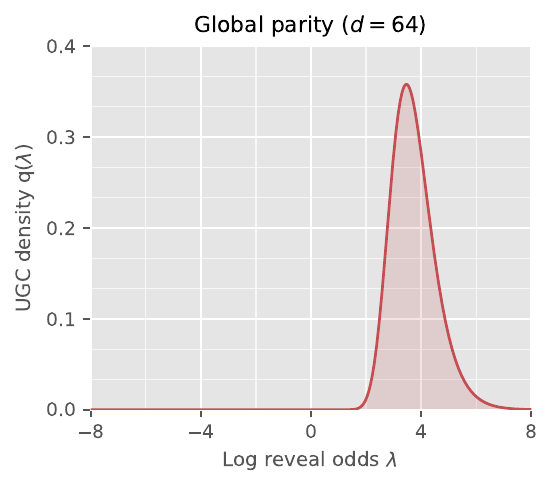}
\\
(a) && (b)
\end{tabular}
\caption{Log-reveal-odds \ugc densities for the repeated-common-bit
  model in panel (a) and the global-parity model in panel (b), both
  with $\usedim = 64$.  The two densities are reflections of one
  another about $\lam = 0$ in log-reveal-odds coordinates: their
  information is concentrated near opposite ends of the reveal path
  even though their full-path \ugc values agree.}
\label{FigRepeatedParityQdens}
\end{center}
\end{figure}

\paragraph{Repeated bit:}
First, suppose that we draw a hidden $\Bit \sim \Ber(1/2)$, and then
generate the random vector $\Zvar \in \{0,1 \}^\usedim$ by setting
$\Zvar_1 = \cdots = \Zvar_\usedim = \Bit$.  This construction is
simply the noiseless ($\eta =0$) instance of the noisy-repeated-bit
example discussed previously in~\Cref{SecGeometry}, and moreover, we
can compute its \ugc density function explicitly.  From
equation~\eqref{EqnEquivQdens}, it suffices to compute $\hfun$ and
then its derivative. Conditional on coordinate $i$ remaining masked,
its mutual-information contribution equals $\Ent(\Bit) = \log 2$ once
at least one of the other $\usedim - 1$ coordinates is revealed; it is
equal to zero otherwise.  It follows that $\hfun_{\mathrm{rep}}(r) =
\usedim \log 2 \big\{ 1 - (1 - r)^{\usedim - 1} \big\}$. Taking
derivatives and combining with the formula~\eqref{EqnEquivQdens}
yields
\begin{subequations}
  \begin{align}
\label{EqnQdensRep}    
  \qdens_{\mathrm{rep}}(\lam) & = \big \{ \usedim (\usedim - 1) \log 2
  \big \} \; r^2 (1 - r)^\usedim.
\end{align}

\paragraph{Global-parity model:}
In the global-parity model, draw $\Zvar_1, \ldots, \Zvar_{\usedim -
  1}$ independently from $\Ber(1 / 2)$ and set $\Zvar_\usedim =
\Zvar_1 \oplus \cdots \oplus \Zvar_{\usedim - 1}$.  Calculations
similar to those in the previous example show that
\begin{align}
\label{EqnQdensPar}
\qdens_{\mathrm{par}}(\lam) & = \big \{ \usedim (\usedim - 1) \log 2
\big \} \; (1 - r)^2 \; r^\usedim.
\end{align}
\end{subequations}
We plot these two \ugc densities in~\Cref{FigRepeatedParityQdens}.
Comparing equation~\eqref{EqnQdensRep} and
equation~\eqref{EqnQdensPar} reveals the relation\footnote{This
follows from the identity $\revinv(-\lam) = 1 - \revinv(\lam)$, which
interchanges $r$ and $1 - r$ under reflection.}
$\qdens_{\mathrm{par}}(\lam) = \qdens_{\mathrm{rep}}(-\lam)$, which
corresponds to reflection across $\lam = 0$ in log-reveal-odds.  This
reflection can be seen by comparing the two panels.

As can be verified via a straightforward calculation, both models have
$\Hinfo(0, 1) = \frac{(\usedim - 1)}{(\usedim + 1)} \log 2$, so that
they have matched aggregate \ugc mass.  However, the full \ugc path
reveals the differences in their distributional structure: the
repeated-bit dependence is resolved near the beginning of the path,
whereas parity dependence persists until the end.  Moreover, as we
develop in the next section, this path difference has important
consequences for designing optimal algorithms.


\section{Certified-optimal bounds for the \ugc path}
\label{SecTracking}
Thus far, we have considered the single-block consequence
of~\Cref{ThmMaster}, based on one geometric multiplier over the full
reveal interval. The additivity~\eqref{EqnHinfoAdditive} of the \ugc
function allows a more refined approach. We first partition the reveal
interval into $\Btot$ blocks, and then optimize the stepsizes within
each block. Doing so leads to a \emph{\ugc partition complexity} that
refines the coarse \ugc bound.  We then show how the required \ugc
increments can be estimated from samples, yielding certified
data-dependent samplers and allowing the block boundaries themselves
to be optimized. Finally, by refining the partition, we recover the
fine-partition complexity $\FP$ in equation~\eqref{EqnDefnFinePart},
which characterizes the optimal discretization error up to a factor of
two.

\subsection{An optimal $\Btot$-block guarantee}

For some $0 < \newrinit < \newrfinal < 1$, we begin with a prescribed
$\Btot$-block partition of the interval $[\newrinit, \newrfinal]$, say
of the form \mbox{ $\Partition \defn \big\{ [\block_{\bind},
    \block_{\bind + 1}] \mid \bind = 0, \ldots, \Btot - 1 \big\}$,}
where $\block_0 = \newrinit$ and $\block_{\Btot} = \newrfinal$. For
each block, define its log-reveal-odds length and \ugc mass by
\begin{align}
  \label{EqnLogRevealBlock}
  \Sinfo_\bind & \defn \logit(\block_{\bind + 1}) -
  \logit(\block_\bind), \quad \mbox{and} \quad \Hinfo_\bind \defn
  \Hinfo(\block_\bind, \block_{\bind + 1}),
\end{align}
where we recall that $\revodds(r) \defn \log(r/(1-r))$.

Given a total iteration budget $N$, suppose that we allocate positive
integers $N_\bind$ with $\sum_{\bind = 0}^{\Btot - 1} N_\bind = N$.
We then run the $\Geo(\rho)$-sampler to traverse the
$\bind^{th}$-block in $N_\bind$ iterations.  Applying~\Cref{ThmMaster}
blockwise yields
\begin{subequations}
\begin{align}
  \label{EqnMultiUpper}
  \KL( \Prob_\Zvar \| \Prob_{\Zhat}) & \leq \sum_{\bind = 0}^{\Btot -
    1} \Big[ \exp\big\{ \Sinfo_\bind / N_\bind \big\} - 1 \Big]
  \Hinfo_\bind + \HACKT.
\end{align}
Here we have again used the additivity
property~\eqref{EqnHinfoAdditive} to bound the error of each block by
the scalar multiple of its \ugc mass.  Hence, for the prescribed
partition, the best guarantee furnished by~\Cref{ThmMaster} is
obtained from the integer program
\begin{align}
  \label{EqnKblockDP}
  \min_{N_0, \ldots, N_{\Btot - 1} > 0} \quad & \sum_{\bind =
    0}^{\Btot - 1} \Big[ \exp\big\{ \Sinfo_\bind / N_\bind \big\} - 1
    \Big] \Hinfo_\bind \qquad \mbox{subject to} \qquad \sum_{\bind =
    0}^{\Btot - 1} N_\bind = N.
\end{align}
\end{subequations}
It is easy to see that this integer program can be solved efficiently
by dynamic programming.

\subsection{Guarantees for an explicit procedure}
\label{SecTrackMulti}

Rather than solving the dynamic program directly, we seek an explicit
allocation that exposes the structure of its solution.  Doing so leads
to the \emph{\ugc-based partition complexity}
\begin{align}
  \label{EqnUGCPartComplex}
  \PartH & \defn \left( \sum_{\bind = 0}^{\Btot - 1}
  \sqrt{\Sinfo_\bind \Hblk_\bind} \right)^2,
\end{align}
which interpolates between the single-block
complexity~\eqref{EqnDefnSingBlock} and the fine-partition
limit~\eqref{EqnDefnFinePart}.

\mygraybox{
\begin{proposition}[Near-optimal geometric schedules for a $\Btot$-block partition]
\label{PropTrackMulti}
Given any $\Btot$-block partition $\Partition$ of $[\newrinit,
  \newrfinal]$ and an iteration budget $\Nscore \geq 2 \left\{ \Btot +
2 \log\left( \frac{\odds(\newrfinal)}{\odds(\newrinit)} \right)
\right\}$, define the geometric multipliers
\begin{subequations}
\begin{align}
  \label{EqnTrackMultiRho}
  \rho_\bind & \defn \min\left\{ 1,\, 4 \frac{\sqrt{\PartH}}{\Nscore}
  \sqrt{\frac{\Sinfo_\bind}{\Hblk_\bind}} \right\}, \qquad \bind = 0,
  1, \ldots, \Btot - 1.
\end{align}
Then the $\Btot$-block \Ber-unmasking procedure based on
$\Partition$ uses at most $\Nscore$ score evaluations and yields an
output $\Zhat$ satisfying
\begin{align}
  \label{EqnTrackMultiKL}
  \KL(\Prob_\Zvar \| \Prob_{\Zhat})
  & \leq
  \frac{4 \PartH}{\Nscore} + \HACKT.
\end{align}
Moreover, if $(\Nstar_0, \ldots, \Nstar_{\Btot - 1})$ denotes the
optimal solution of the dynamic program~\eqref{EqnKblockDP} with
iteration budget $N$, then
\begin{align}
  \label{EqnTrackMultiLower}
  \sum_{\bind = 0}^{\Btot - 1} \rho(\Nstar_\bind) \Hinfo_\bind & \geq
  \frac{\PartH}{N} \qquad \mbox{where $\rho(\Nstar_\bind) \defn
    \exp(\Sinfo_\bind/\Nstar_\bind) - 1$.}
\end{align}
\end{subequations}
\end{proposition}
}

\noindent
The proof is an unmasking analogue of the argument used for Gaussian
diffusion in our companion paper~\cite{Wai26}.  The principal
changes are the log-reveal-odds clock, the corresponding \ugc block
masses, and the boundary terms; the remainder follows \emph{mutatis
mutandis}.  This parallelism reflects a common geometric structure
underlying Gaussian diffusion and discrete unmasking, despite their
different noise mechanisms and natural path coordinates.

The upper and lower bounds in~\Cref{PropTrackMulti} show that, apart
from the initial and terminal cost, the explicit allocation is within a
factor of $4$ of the optimal value of the dynamic
program~\eqref{EqnKblockDP}.  It also always improves upon the
single-block complexity.  Indeed, by Cauchy--Schwarz and the
additivity~\eqref{EqnHinfoAdditive} of the \ugc complexity, we have
\begin{align}
  \label{EqnCauchy}
  \PartH & = \left( \sum_{\bind = 0}^{\Btot - 1} \sqrt{\Sinfo_\bind
    \Hblk_\bind} \right)^2 \leq \left( \sum_{\bind = 0}^{\Btot - 1}
  \Sinfo_\bind \right) \left( \sum_{\bind = 0}^{\Btot - 1} \Hblk_\bind
  \right) = \big\{ \logit(\newrfinal) - \logit(\newrinit) \big\}
  \Hinfo(\newrinit, \newrfinal) = \PartComp([\newrinit, \newrfinal]).
\end{align}
Equality holds if and only if there is a scalar $c > 0$ such that
$\Sinfo_\bind = c \Hblk_\bind$ for every block.  Thus, partitioning
provides a strict improvement whenever the \ugc mass is distributed
non-uniformly relative to log-reveal-odds length.


\subsection{Sandwich-\ugc estimators from KL increments}
\label{SecDataSandwich}

While the bounds in~\Cref{PropTrackMulti} are attractive, the
geometric multipliers $\rho_k$ depend on knowledge of the \ugc
increments $\Hinfo_\bind$ for each block $\bind$.  In this section, we
describe how it is possible to estimate these increments, assuming
that we have available a collection of i.i.d. samples
$\{\Zclean{\ell} \}_{\ell=1}^m$ from the original distribution
$\Prob_Z$.  This estimator forms the basis for the certified-optimal
sampling schemes described in~\Cref{SecDataMultiple}.


\subsubsection{Kullback--Leibler unmasking increments}
\label{SecMaskDenoisingIncrements}

From the definition~\eqref{EqnDefnHfun}, the Bernoulli unmasking gain
$\hfun$ is defined in terms of the original reveal process $\Xvar_p$
with additional conditioning on the event $\{ i \in \Masked(\Xvar_p) \}$.
Consequently, it is useful to introduce the \emph{forced-mask
reveal-time process} that explicitly encodes this conditioning.
Recalling the uniform random variables that underlie the masking
process~\eqref{EqnMaskingProcess}, we define
\begin{subequations}
\begin{align}
\label{EqnForcedMaskCoupling}  
\NewMask{i}{\rtime} \defn \{ i \} \cup \big\{ \ell \in [\usedim]
\setminus \{ i \} \mid \Uvar_\ell > \rtime \big\}, \quad \mbox{and}
\quad \NewXvar{i}{\rtime} \defn \left( \NewMask{i}{\rtime},
\Zmask{(\NewMask{i}{\rtime})^c} \right).
\end{align}
Due to the shared set of uniform random variables, the process
satisfies a natural coupling: for $q > p$, the variable
$\NewXvar{i}{q}$ can be obtained from $\NewXvar{i}{p}$ by revealing each
variable in $\Masked(\NewXvar{i}{p}) \setminus \{ i \}$ independently with probability
$\frac{q - p}{1 - p}$.  Moreover, by construction, the random variable
$\NewXvar{i}{\rtime}$ has the law of the original reveal process
$X_\rtime$ conditional on $i \in \Masked(X_\rtime)$.

For $0 \leq p < q \leq 1$, we define the \emph{KL unmasking increment}
\begin{align}
\label{EqnDefnDinfoMask}
\Dinfo(p, q) & \defn (q - p) \sum_{i = 1}^{\usedim} \Exs\left[
  \KL\left( \pi_{i, q} \,\middle\|\, \pi_{i, p} \right) \right] \qquad
\mbox{where $\pi_{i, \rtime}(\cdot) \equiv \Denoise_i\left(
  \mathord\cdot, \NewXvar{i}{\rtime} \right)$ for each $\rtime \in
               [0, 1]$,}
\end{align}
\end{subequations}
and the expectation is taken with respect to the joint forced-mask
coupling $(\NewXvar{i}{p}, \NewXvar{i}{q})$ for each $i = 1, \ldots,
\usedim$.  With this set-up, we have the following guarantee:

\mygraybox{
\begin{lemma}
\label{LemDinfoSandwich}
For every $0 \leq p < q < 1$, we have the exact relation
\begin{subequations}
  \begin{align}
    \label{EqnDinfoExact}
\frac{\Dinfo(p, q)}{q - p} & = \hfun(q) - \hfun(p).
\end{align}
For every $0 < p < q < 1$, we also have the sandwich relation
\begin{align}
\label{EqnDinfoSandwich}
\frac{1}{\myrho} \Dinfo(p, q) \; \leq \; \Hinfo(p, q) \; \leq \;
\frac{1 + \myrho}{\myrho} \Dinfo(p, q) \, \qquad \mbox{where $\myrho
  \defn \dfrac{\odds(q)}{\odds(p)} - 1 > 0$.}
\end{align}
\end{subequations}
\end{lemma}
}
\noindent See~\Cref{SecProofLemDinfoSandwich} for the proof.

Observe that the sandwich relation~\eqref{EqnDinfoSandwich} is
particularly well-suited to intervals $[p,q]$ for which the ratio of
the reveal odds $\odds(r) = r/(1-r)$ is well-controlled.  Concretely,
for a dyadic interval with $\odds(q)/\odds(p) = 2$, the relation
provides a factor two sandwich.  We now exploit this fact to design an
estimator for a partition of arbitrary length.


\subsubsection{A tail-robust estimate of the \ugc increment}

We now specify and analyze a tail-robust estimate of the \ugc
increment.  For an arbitrary pair $0 < p < q < 1$, consider the
\emph{reveal-odds-dyadic partition} $\{v_j\}_{j=0}^{J}$ given by
\begin{subequations}
\begin{align}
\label{EqnRevDyadic}    
  \odds(v_j) & \defn \min\{ 2^j \odds(p), \odds(q) \} \qquad
  \mbox{where $\odds(r) \defn r/(1-r)$ and $J \defn \left\lceil
    \log_2\left( \frac{\odds(q)}{\odds(p)} \right) \right\rceil$.}
\end{align}
For each interval $[v_j, v_{j+1}]$,~\Cref{LemDinfoSandwich} guarantees
that $\Dinfo(v_j, v_{j+1}) \; \leq \; \Hinfo(v_j, v_{j+1}) \; \leq \;
2 \Dinfo(v_j, v_{j+1})$, and hence that $\sum_{j = 0}^{J-1}
\Dinfo(v_j, v_{j+1}) \; \leq \; \Hinfo(p,q) \; \leq \; 2 \sum_{j =
  0}^{J-1} \Dinfo(v_j, v_{j+1})$, where we have used the additivity
property~\eqref{EqnHinfoAdditive} of $\Hinfo$.

Given a clean sample $\Zclean{\ell}$, we can simulate the forced-mask
reveal process $\XvarForce{t}$ over the interval $t \in [0,1]$, for
each coordinate $i = 1, \ldots, \usedim$, and using these
trajectories, we can construct the \emph{trajectory statistic}
\begin{align}
  \label{EqnDefnQhatSam}
  \QhatSam & \defn \sum_{i = 1}^\usedim \left \{ \sum_{j = 0}^{J - 1}
  (v_{j + 1} - v_j) \KL\Big( \Denoise_i(\cdot,
  \XvarForce{v_{j + 1}}) \| \Denoise_i(\cdot,
  \XvarForce{v_{j}}) \Big) \right \}.
\end{align}
By construction, we have $\Exs[\QhatSam] = \sum_{j=0}^{J-1}
\Dinfo(v_j, v_{j+1})$, where the expectation is taken over the sample
$\Zclean{\ell}$, along with the forced-mask trajectory randomness.
Combined with the factor two sandwich property, we have constructed a
statistic such that
\begin{align}
\label{EqnQhatSamSand}  
\Exs[ \QhatSam] \; \leq \; \Hinfo(p, q) \; \leq \; 2 \Exs[\QhatSam].
\end{align}
\end{subequations}

\paragraph{A tail-robust estimator:}
Given the statistic $\QhatSam$ from equation~\eqref{EqnDefnQhatSam},
the standard Monte Carlo estimate is given by $\frac{1}{\numsam}
\sum_{\samind=1}^\numsam \QhatSam$.  It is unbiased and can be
analyzed, but can be overly sensitive to the tail behavior of the
underlying KL divergences in the estimate.  To introduce tail
robustness, we instead analyze a truncated version of this
estimator.  It requires a moment bound, but provides exponential
Bernstein-type tail control.

More precisely, suppose that, for some moment order $\alpha \geq 4$,
the changes in the one-coordinate denoisers satisfy the aggregate
moment bound
\begin{subequations}
\begin{align}
\label{EqnMomentCondition}
\max_{j = 0, \ldots, J - 1} \left\{ \Exs\left[ \left\{ \sum_{i =
    1}^{\usedim} \KL\left( \Denoise_i\left(
  \mathord\cdot, \XvarForce{v_{j + 1}} \right) \,\middle\|\,
  \Denoise_i\left( \mathord\cdot, \XvarForce{v_j} \right)
  \right) \right\}^{\alpha / 2}
  \right] \right\}^{2 / \alpha} & \leq B_\alpha \qquad \mbox{for some
  known $B_\alpha > 0$.}
\end{align}
Under this condition, we define the
$\tau$-truncated estimator
\begin{align}
\label{EqnDefnHhat}
\Hhat(p, q) & \defn \frac{2}{\numobs} \sum_{\samind = 1}^{\numobs}
\min\{ \Qfunup{\samind}, \tau \} \qquad \mbox{where $\tau \defn (q -
  p) B_\alpha \left\{ \frac{3 (\numobs - 1)} {7 \log(4 / \eta)}
  \right\}^{2 / \alpha}$.}
\end{align}
Here $\eta \in (0,1)$ is a user-chosen failure probability and
$\{\Qfunup{\samind} \}_{\samind = 1}^\numobs$ are i.i.d. copies of the
trajectory statistic $\QhatSam$ from equation~\eqref{EqnDefnQhatSam}.
We also define the (unbiased) empirical variance of our truncated
estimator as
\begin{align}
\label{EqnDefnEmpBernstein}
\Vhat & \defn \frac{1}{\numobs - 1} \sum_{\samind = 1}^{\numobs}
\left( 2 \min\{ \Qfunup{\samind}, \tau \} - \Hhat(p, q) \right)^2.
\end{align}
\end{subequations}

\mygraybox{
\begin{proposition}[Factor-two data-dependent sandwich on \ugc]
\label{PropDataSingle}
Given the estimate $\Hhat(p,q)$, the \ugc increment $\Hinfo(p,q)$ can
be sandwiched as
\begin{subequations}
\begin{align}
\label{EqnDataSingle}
\Hinfo(p, q) & \stackrel{(A)}{\leq} \Hhat(p, q) + \HackErr
\stackrel{(B)}{\leq} 2 \big\{ \Hinfo(p, q) + \HackErr \big\} \qquad
\mbox{with probability at least $1 - \eta$,}
\end{align}
where the upper confidence correction takes the form
\begin{align}
\label{EqnDefnHackErr}
\HackErr & \defn \sqrt{ \frac{2 \Vhat \log(4 / \eta)} {\numobs} } + 4
(q - p) B_\alpha \left\{ \frac{7 \log(4 / \eta)} {3 (\numobs - 1)}
\right\}^{1 - 2 / \alpha}.
\end{align}
\end{subequations}
\end{proposition}
}
\noindent See~\Cref{SecProofPropDataSingle} for the proof.

\paragraph{Interpretation of the tail condition:}
Condition~\eqref{EqnMomentCondition} controls the tails of the
aggregate change in the one-coordinate denoisers over each
reveal-odds-dyadic interval.  This choice is important: it measures
the same posterior changes that enter the trajectory
statistic~\eqref{EqnDefnQhatSam}, rather than the absolute uncertainty
of the individual denoisers.  A condition on the true-symbol log loss,
for instance, could be large even when revealing additional
coordinates produces no change in any posterior.

As a basic example, suppose that the coordinates of $\Zvar$ are
independent.  For the exact Bayes denoisers considered here, the
forced-mask observation $\NewXvar{i}{\rtime}$ reveals only masking
variables and a subset of the coordinates $\Zvar_{-i}$.  Consequently,
$\Zvar_i$ is independent of $\NewXvar{i}{\rtime}$, and hence
\begin{align*}
\Denoise_i\left( \mathord\cdot,
  \NewXvar{i}{\rtime} \right) & = \Law(\Zvar_i) \qquad
  \mbox{for every $\rtime \in [0, 1]$.}
\end{align*}
All of the KL increments in equation~\eqref{EqnMomentCondition} therefore
vanish, so that the condition holds with any positive $B_\alpha$, in agreement
with the fact that the corresponding \ugc increment is zero.


\subsection{Sampling schemes that are certified-optimal}
\label{SecDataMultiple}

In this section, we turn to the key problem of designing samplers that
are \emph{certified-optimal}, meaning that their iteration complexity
meets the optimum specified by the partition complexity functional, up
to data-dependent residual terms, and that for any target accuracy
$\varepsilon > 0$ and failure probability $\eta \in (0, 1)$, there
are data-dependent parameter choices for the sampler
guaranteeing that its output $\Zhat$ satisfies the accuracy bound
\begin{align}
  \label{EqnCertOpt}
  \KL \big(\Prob_{\Zvar} \|
    \Prob_{\ZhatSmall} \big) & \leq \varepsilon \qquad
    \mbox{with probability at least $1 - \eta$.}
\end{align}
In~\Cref{SecDataKblock}, we specify and analyze a certified procedure
for a given $\Btot$-block partition, whereas in~\Cref{SecDataDP},
we describe certified procedures based on dynamic programming for
choosing the block boundaries.


\subsubsection{Certified-optimal sampling using $\Btot$-blocks}
\label{SecDataKblock}

Suppose that we are given a $\Btot$-block partition of the interval
$[\newrinit, \newrfinal]$, based on blocks of the form $[\block_\bind,
  \block_{\bind+1}]$.  Our goal is to design a sampler that achieves
the optimal $\Btot$-bound from~\Cref{PropTrackMulti}, but using a
choice of geometric multipliers that can be estimated based on samples
$\{\Zclean{\samind}\}_{\samind=1}^\numobs$.  Here we describe a simple
way of doing so, one which exploits the guarantees
from~\Cref{PropTrackMulti} as well as the \ugc increment estimator
analyzed in~\Cref{PropDataSingle}.

Suppose that, for each block $[\block_{\bind}, \block_{\bind+1}]$, we
apply the tail-robust estimator from~\Cref{PropDataSingle} with
failure probability $\eta/\Btot$.
We then obtain a tail bound with the upper confidence limit
\begin{subequations}
\begin{align}
\label{EqnNewHackErr}
\underbrace{\NewHackErr}_{\equiv \rhat_\bind} & \defn \sqrt{ \frac{ 2
    \Vhat_\bind \log(4 \Btot / \eta) }{\numobs} } + 4 (\block_{\bind +
  1} - \block_\bind) B_\alpha \left( \frac{ 7 \log(4 \Btot / \eta) }{
  3(\numobs - 1) } \right)^{1 - 2/\alpha},
\end{align}
where $\Vhat_\bind$ is the empirical variance estimate for the block.
Finally, we define the surrogate partition complexity
\begin{align}
\label{EqnDefnPartCompHat}
\PartCompHat(\Partition) & \defn \Big( \sum_{\bind =0}^{\Btot-1}
\sqrt{\Sinfo_\bind} \sqrt{\HhatBind{\bind} + \rhat_\bind} \Big)^2.
\end{align}
\end{subequations}
With this set-up, we have the following \emph{certified analogue} for
the \Ber-unmasking sampler with geometric multipliers, or
the $\Geo(\rho)$-scheme for short.

\mygraybox{
\begin{theorem}[Data-certified guarantees for multi-block scheme]
\label{ThmCertifiedMulti}    
Given a failure probability $\eta \in (0, 1)$, and a $\Btot$-block
partition of the interval $[\newrinit, \newrfinal]$, consider any
iteration budget \mbox{$\Nscore \geq 2 \left\{ \Btot + 2 \big\{
  \logit(\newrfinal) - \logit(\newrinit) \big\} \right\}$.}
Then running the $\Btot$-block \mbox{$\Geo(\rho)$-unmasking} scheme
with the geometric multipliers
\begin{subequations}
\begin{align}
\label{EqnCertifiedMultiRho}
\rhohat_\bind & \defn \min\left\{ 1, \frac{4
  \sqrt{\PartCompHat(\Partition)}}{\Nscore}
\sqrt{\frac{\Sinfo_\bind}{\HhatBind{\bind} + \rhat_\bind}} \right\}
\qquad \mbox{for $\bind = 0, 1, \ldots, \Btot - 1$.}
\end{align}
yields a completed output $\Zhat$ such that
\begin{align}
\label{EqnCertifiedMulti}
\KL(\Prob_{\Zvar} \| \Prob_{\Zhat}) & \leq \frac{4
  \PartCompHat(\Partition)}{\Nscore} + \BOUNDARY \qquad \mbox{with
  probability at least $1 - \eta$,}
\end{align}
and does so with at most $N$ unmasking rounds.
\end{subequations}
\end{theorem}
}
\noindent See~\Cref{SecProofThmCertifiedMulti} for the proof.

\paragraph{Certified-optimal:} Let us now clarify why the guarantee of~\Cref{ThmCertifiedMulti} is both
certified and optimal (up to constant factors).  Suppose that our goal
is to design a sampler that achieves a prescribed KL error
$\varepsilon > 0$.  As discussed in~\Cref{SecCanonical}, it is
straightforward to ensure that the initialization and terminal
contribution satisfies $\HACKT \leq \varepsilon / 2$.
From~\Cref{ThmCertifiedMulti}, it can be seen that choosing an
iteration number \mbox{$\Nscore \geq 8 \PartCompHat(\Partition) /
  \varepsilon$} guarantees that the resulting sampler has output
$\Zhat$ such that
\begin{align}
\label{EqnCertifiedOptimal}  
\KL(\Prob_\Zvar \| \Prob_{\Zhat}) & \leq \varepsilon \qquad
\mbox{with probability at least $1 - \eta$.}
\end{align}
Thus, the estimated partition complexity provides both a
data-dependent schedule and an end-to-end certificate of its sampling
accuracy.  For this reason, we refer to it as a \emph{certified
guarantee}.

It is also \emph{near-optimal} in a precise sense.  In particular,
recall that the procedure in~\Cref{PropTrackMulti} achieves the DP
optimum up to a constant factor of $4$.  Under the same initialization
conditions, doing so requires an iteration number $N \geq 8
\PartH/\varepsilon$.  As shown in the proof
of~\Cref{ThmCertifiedMulti}, on the simultaneous high-probability
event from~\Cref{PropDataSingle}, we have
\begin{align*}
  \sqrt{\PartCompHat(\Partition)} & \leq \sqrt{2} \left\{
  \sqrt{\PartH} + \sum_{\bind = 0}^{\Btot - 1} \sqrt{\Sinfo_\bind
    \rhat_\bind} \right\}.
\end{align*}
Thus, the increase in computational complexity incurred by replacing the
oracle \ugc increments by their data-dependent estimates is captured
explicitly by the confidence radii $\rhat_\bind$, together with the
factor of $\sqrt{2}$ in the sandwich bound.


\subsubsection{Certified-optimal boundary selection}
\label{SecDataDP}

Our analysis thus far has focused on optimal schemes for a fixed set
of $\Btot$ blocks. In this section, we formulate the problem of
optimizing the split points that define a (near)-optimal set of
$\Btot$ blocks via a simple dynamic program.

More precisely, for an integer $\Jtot > \Btot$ that defines the
resolution, suppose that we fix a fine deterministic grid of
\emph{candidate reveal times}
$\tau_0 = \newrinit < \tau_1 < \cdots < \tau_\Jtot = \newrfinal$.
Our goal is to select indices $\{ i_\bind \}_{\bind=1}^{\Btot-1}$
from the set $\{ 1, \ldots, \Jtot - 1 \}$, and use the corresponding
reveal times $t_\bind = \tau_{i_\bind}$, together with the endpoints
$t_0 = \newrinit$ and $t_\Btot = \newrfinal$, to define a $\Btot$-block
partition. We seek to minimize the square root of the partition
complexity
\begin{align*}
  \sqrt{\PartComp(\Partition)} & = \sum_{\bind=0}^{\Btot-1}
  \sqrt{\Sinfo_\bind \Hblk_\bind}, \qquad \mbox{where $\Sinfo_\bind
    \defn \logit(t_{\bind+1}) - \logit(t_\bind)$, and
    $\Hblk_\bind \defn \Hinfo(t_\bind,
    t_{\bind+1})$.}
\end{align*}

Given the naturally sequential structure, this optimization problem
can be solved by dynamic programming, as we now formalize. For
candidate intervals $0 \leq i < j \leq \Jtot$, define the edge cost
\begin{subequations}
\begin{align}
\label{EqnKBlockEdgeCost}
e(i, j) & \defn \sqrt{\Sinfo(\tau_i, \tau_j) \Hinfo(\tau_i, \tau_j)},
\qquad \mbox{where $\Sinfo(a, b) \defn \logit(b) - \logit(a)$.}
\end{align}
With this notation, our ultimate goal is to compute
\begin{align}
  \label{EqnFinalDP}
  \ValFun_\Btot(\Jtot) & \defn \min_{0 = j_0 < j_1 < \cdots < j_\Btot
    = \Jtot} \sum_{\bind=0}^{\Btot-1} e(j_\bind, j_{\bind+1}).
\end{align}
\end{subequations}
This quantity is the minimum cost of a $\Btot$-block partition of the
interval $[\tau_0, \tau_\Jtot] = [\newrinit, \newrfinal]$. It is easy
to see that it can be computed via dynamic programming, where we
introduce intermediate quantities $\ValFun_\bind(j)$, defined for
each $\bind \in \{ 0, \ldots, \Btot \}$ and $j \in \{ 0, \ldots, \Jtot
\}$ as the minimum cost of a $\bind$-block partition of the
subinterval $[\tau_0, \tau_j]$, with value $+\infty$ when no such
partition exists.

Finally, we observe that the edge costs in
equation~\eqref{EqnKBlockEdgeCost} depend on the \ugc increments
$\Hinfo(\tau_i, \tau_j)$.  As in the previous section, we can obtain
confidence bounds for increments of this type using the estimator
from~\Cref{PropDataSingle}, thereby obtaining a data-dependent version
of the dynamic program.

\section{Fine partition limit and optimal Euler discretization}
\label{SecFine}

In the preceding section, we developed and analyzed sampling methods
based on finite $\Btot$-block partitions, and showed how their
performance depends on the local \ugc mass of the chosen partition.
We now characterize the limit under arbitrarily fine refinement and
connect it to the optimal KL error of Euler discretization.  This
analysis highlights the central role of the log-reveal-odds \ugc density
$\qdens$ defined in equation~\eqref{EqnDefnQdens}.


\subsection{Fine partition limit and its Euler optimality}
\label{SecUGCDensityRole}

We specialize to the canonical reveal interval $\IntStar \defn
[1/\usedim, 1 - 1/\usedim]$.  Its image under the log-reveal-odds map
is $[-\Elld, \Elld]$, where $\Elld \defn \log(\usedim - 1)$.  For a
partition $\Partition = \{ [\block_\bind, \block_{\bind + 1}]
\}_{\bind = 0}^{\Btot - 1}$ of $\IntStar$, define the corresponding
log-reveal-odds blocks $\Lint_\bind \defn [\logit(\block_\bind),
  \logit(\block_{\bind + 1})]$ and their lengths $\Sinfo_\bind \defn
\MyLen(\Lint_\bind)$.  Our previously defined \emph{partition
complexity} can then be written directly in terms of $\qdens$ as
\begin{subequations}
  \begin{align}
    \label{EqnDefnNewPartComp}
  \PartComp(\Partition) & \defn \left( \sum_{\bind = 0}^{\Btot - 1}
  \sqrt{ \Sinfo_\bind \int_{\Lint_\bind} \qdens(\lam) \, d\lam }
  \right)^2.
\end{align}
Observe that refinement can only decrease this complexity: whenever
$\Partition'$ is a refinement of $\Partition$, we have
$\PartComp(\Partition') \leq \PartComp(\Partition)$.  This fact
motivates the \emph{fine-partition limit}
\begin{align}
  \label{EqnDefnPartHfine}
  \PartHfine(\IntStar) & \defn \inf_{\Partition}
  \PartComp(\Partition),
\end{align}
\end{subequations}
where the infimum is over all finite partitions of $\IntStar$.  We now
characterize this infimum explicitly in terms of $\sqrt{\qdens}$, and
show that the same quantity governs the sharp leading-order KL error
of optimal Euler discretization.

\mygraybox{
\begin{theorem}[Fine-partition limit and Euler optimality]
\label{ThmFineEuler}
Suppose that $\qdens$ is continuous and strictly positive on $[-\Elld,
  \Elld]$.  Then the fine-partition limit~\eqref{EqnDefnPartHfine} is
given by
\begin{subequations}
\begin{align}
  \label{EqnFinePartSandwich}
  \PartHfine(\IntStar) & = \left( \int_{-\Elld}^{\Elld}
  \sqrt{\qdens(\lam)} \, d\lam \right)^2.
\end{align}
Moreover, the optimal $N$-step KL discretization error satisfies
\begin{align}
  \label{EqnFineEulerOptimality}
  \inf_{\frac{1}{\usedim} = \rtime_0 < \cdots < \rtime_N =
    \frac{\usedim - 1}{\usedim}} \sum_{\iter = 0}^{N - 1}
  \underbrace{\UnmaskKL(\rtime_\iter, \rtime_{\iter + 1})}_{\mbox{KL
      discretization error}} & = \frac{\PartHfine(\IntStar)}{2 N} +
  o(N^{-1}),
\end{align}
where the infimum is over all $N$-step discretizations of the
canonical interval $[\tfrac{1}{d}, \tfrac{d-1}{d}]$.
\end{subequations}
\end{theorem}
}
\noindent See~\Cref{SecProofThmFineEuler} for the proof.

\noindent Taken together, the two claims in~\Cref{ThmFineEuler} show
that the fine-partition complexity is precisely twice the
leading-order constant governing the optimal Euler KL error, and that
it can be approached arbitrarily closely by the $\Btot$-block schemes
analyzed in~\Cref{PropTrackMulti} and~\Cref{ThmCertifiedMulti}.

\paragraph{Related work:}
A related square-root principle appears in the analysis of Lavenant
and Zanella~\cite{LavZan25}, who study random-order fixed-cardinality
unmasking and express its factorization error in terms of a
cardinality-indexed information profile.  They then consider a
high-dimensional continuum limit in which, using our notation, both
the ambient dimension $\usedim$ and the number of sampling rounds $N$
diverge. In the regime $\usedim / N \rightarrow +\infty$, optimization
of the resulting limiting functional yields a square-root
information-profile rule.

Our results differ in both formulation and scope.  The \ugc density
$\qdens$ is defined directly from the finite-dimensional reveal path,
with log-reveal odds providing the natural path coordinate, and our
finite-step bounds apply at fixed dimension and for arbitrary finite
partitions.  Moreover, rather than optimizing only a limiting
variational problem, we optimize blockwise schedules at finite
dimension, and use estimated \ugc increments to construct
data-dependent partitions with certified guarantees.  Finally,
equation~\eqref{EqnFineEulerOptimality} gives the sharp leading-order
optimal Euler error at fixed dimension $\usedim$ as the number of
discretization steps $N$ grows.

\subsection{Gains and convergence to the fine-partition limit}
\label{SecXORSATSeparation}

We first ask how much can be gained by exploiting the local \ugc
geometry, and then how many blocks are needed to realize this gain.
The maximal improvement over a single-block scheme is measured by
\begin{align}
  \label{EqnDefnRatioNew}
  \Ratio(\Prob_Z)
  & \defn
  \frac{\PartComp([\IntStar])}{\PartHfine(\IntStar)}
  \geq 1.
\end{align}
For the repeated-bit and parity models
(see~\Cref{FigRepeatedParityQdens}), we have $\Ratio(\ProbZ) \asymp
\log(\usedim)$, yielding a growing but relatively modest separation.
Much larger gains are possible when the \ugc density becomes
increasingly concentrated with dimension.  For example, the discrete
mixture model in~\Cref{FigDiscreteMixture} satisfies
\mbox{$\Ratio(\Prob_\Zvar) = \OmegaTil(\sqrt{\usedim})$} with high
probability over the cluster centers.  We next exhibit the same
phenomenon in a simpler random XORSAT model, for which the scaling can
be analyzed via an easier argument.

\subsubsection{An example with $\OmegaTil(\sqrt{\usedim})$ gains}

To obtain a tractable example with the same
$\OmegaTil(\sqrt{\usedim})$ separation, consider a dense random XORSAT
ensemble.  For ambient dimension $\usedim$ and latent dimension
$\kdim$, draw $\Amat \in \{0,1\}^{\usedim \times \kdim}$ with i.i.d.
$\Ber(1/2)$ entries, draw $\Uvar \in \{0,1\}^{\kdim}$ with i.i.d.
$\Ber(1/2)$ entries, and set $\Zvar = \Amat \Uvar$ modulo two.  Thus,
$\Zvar$ is supported on a random Boolean subspace.  We parameterize
the ensemble by the ratio $\alpha \defn \kdim / \usedim$.

\Cref{FigXORSAT} plots the \ugc density for $\usedim = 64$ and $\alpha
\in \{0.2, 0.5, 0.8\}$.  The density is unimodal, with its peak moving
from right to left as $\alpha$ increases.  To quantify the resulting
gain, we focus on the symmetric case $\alpha = 0.5$.
\begin{figure}[h]
\begin{center}
\begin{tabular}{ccc}
\widgraph{0.31\textwidth}{\figdir/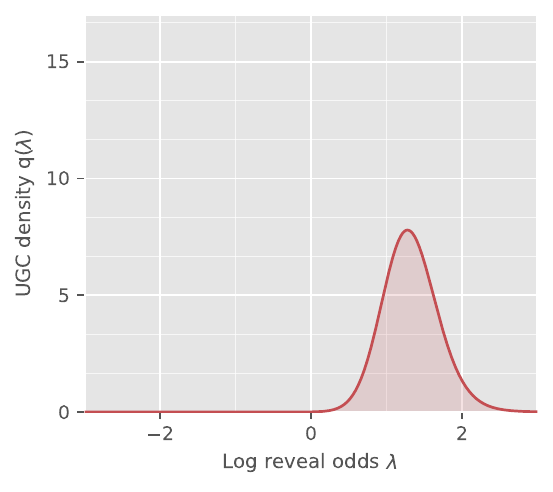}
&
\widgraph{0.31\textwidth}{\figdir/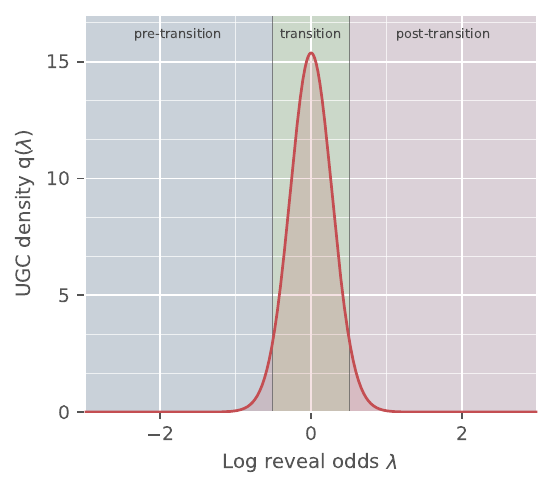}
&
\widgraph{0.31\textwidth}{\figdir/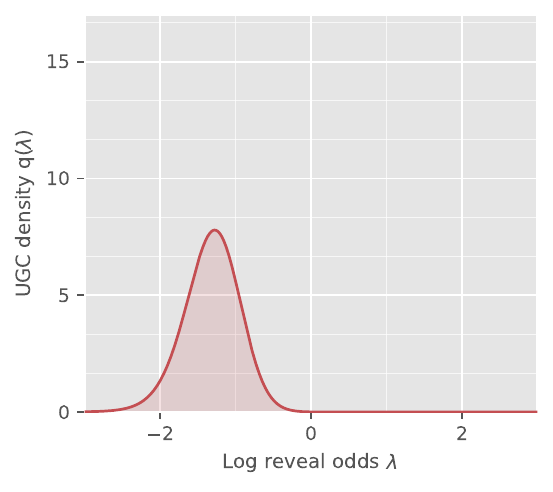}
\\
(a) & (b) & (c)
\end{tabular}
\caption{Plots of the log-odds \ugc density $\qdens$ for the random
  XORSAT ensemble with $d = 64$ and parameter $\alpha \in \{0.2, 0.5,
  0.8\}$ in panels (a)--(c), respectively.  Panel (b) marks the
  pre-transition, transition, and post-transition regions used in the
  three-block construction.}
\label{FigXORSAT}
\end{center}
\end{figure}

\mygraybox{
\begin{lemma}
\label{LemXORSAT}
For the random XORSAT model with $\alpha = 0.5$, for each even
dimension $\usedim$ above a universal constant, there is a
distribution $\ProbZ$ such that
\begin{align}
  \label{EqnXORSAT}
  \Ratio(\ProbZ) & \geq c \, \sqrt{\usedim \log \usedim} \qquad
  \mbox{for a universal constant $c > 0$,}
\end{align}
and this gain is achieved, up to constants, by a $\Btot = 3$ block
scheme.
\end{lemma}
}
\noindent See~\Cref{SecXORSAT} for the proof.

Our analysis uses the probabilistic method~\cite{AloSpe16} to show that, for each
sufficiently large dimension, there exists a realization of $\Amat$,
inducing a ``good'' distribution $\ProbZ$, that satisfies the
following properties.  First, its coarse complexity satisfies
$\PartComp([\IntStar]) = \Theta(\usedim \log \usedim)$.  As a
consequence, it suffices to construct a three-block partition with
complexity $\Order(\sqrt{\usedim \log \usedim})$.  Set $\deld = c_1
\sqrt{\log \usedim / \usedim}$ for a sufficiently large constant $c_1
> 0$, and define the reveal time intervals
\begin{align*}
  \IntLeft & \defn \big[ 1/\usedim, 1/2 - \deld \big], & \IntMid &
  \defn \big[ 1/2 - \deld, 1/2 + \deld \big], & \IntRight & \defn
  \big[ 1/2 + \deld, 1 - 1/\usedim \big].
\end{align*}
Our proof shows that the ``good'' $\ProbZ$ has \ugc masses
\begin{align}
  \label{EqnXORWidth}
  \Hinfo(\IntLeft) = \Order(\usedim^{-10}), \quad \Hinfo(\IntMid) =
  \Theta(\usedim), \quad \mbox{and} \quad \Hinfo(\IntRight) & =
  \Order(\usedim^{-10}).
\end{align}
Thus, consistent with the behavior shown in~\Cref{FigXORSAT},
essentially all of the \ugc mass lies in the transition interval
$\IntMid$, which has log-reveal-odds width $\Order(\sqrt{\log \usedim
  / \usedim})$.  This combination leads to the claimed scaling
$\Order(\sqrt{\usedim \log \usedim})$ for the $3$-block partition.


\subsubsection{Uniform partitions: a generic convergence bound}

The XORSAT example shows that a small number of well-placed blocks can
already be near-optimal.  For comparison, let $\UniK$ denote the
uniform $\Btot$-block partition of $[-\Elld, \Elld]$, which ignores
local geometry when choosing the block boundaries.  Here we quantify
its convergence to the fine-partition limit under a simple global
regularity condition.

In particular, suppose that $\fdens \defn \sqrt{\qdens}$ is continuous
on $[-\Elld, \Elld]$, and define its modulus of continuity by
\begin{align}
\label{EqnBlockAccuracyModulus}
\omega_\fdens(h) & \defn \sup_{\lam, \lam'} |\fdens(\lam) -
\fdens(\lam')| \qquad \mbox{where the sup ranges over pairs $\lam,
  \lam' \in [-\Elld, \Elld]$ with $|\lam - \lam'| \leq h$.}
\end{align}
\mygraybox{
\begin{lemma}[Accuracy of uniform $\Btot$-block approximations]
\label{LemBlockAccuracy}
For the uniform partition $\UniK$, we have
\begin{align}
  \label{EqnBlockAccuracyEqual}
  \sqrt{\PartComp(\UniK)} & \leq \sqrt{\PartHfine(\IntStar)} + \Elld
  \, \omega_\fdens\left( \frac{2 \Elld}{\Btot} \right).
\end{align}
\end{lemma}
}
\noindent
The proof is the same modulus-of-continuity argument used for Gaussian
diffusion in our companion paper (Lemma 4,~\cite{Wai26}).  As a
particular consequence, if $\fdens$ is Lipschitz with constant
$\Lip(\fdens)$, then
\begin{align}
  \label{EqnBlockAccuracyLipschitz}
  \sqrt{\PartComp(\UniK)} & \leq \sqrt{\PartHfine(\IntStar)} + \frac{2
    \Lip(\fdens) \Elld^2}{\Btot}.
\end{align}
This generic bound can be conservative because it depends on a global
regularity constant and fixes the block boundaries in advance.  For
the XORSAT example, the Lipschitz constant $\Lip(\fdens)$ grows with
dimension even though $\Btot = 3$ adaptively placed blocks are
near-optimal.  When samples are available, the data-dependent DP
procedure in~\Cref{SecDataDP} instead optimizes the block boundaries
directly.


\section{Proofs}
\label{SecProofs}

In this section, we collect the proofs of a subset of our results,
including the proof of~\Cref{ThmMaster}, with the Bernoulli and
fixed-cardinality versions given in~\Cref{SecProofThmMasterBer}
and~\Cref{SecProofThmMasterExact}, respectively; the proof of our
data-certification procedures, including~\Cref{LemDinfoSandwich}
and~\Cref{PropDataSingle} in~\Cref{SecProofData}; the proof
of~\Cref{ThmCertifiedMulti} in~\Cref{SecProofThmCertifiedMulti}; and
finally, the proof of~\Cref{ThmFineEuler}
in~\Cref{SecProofThmFineEuler}.  We defer the proofs of other results
in the paper to the appendices.

\subsection{Proof of~\Cref{ThmMaster}: Bernoulli version}
\label{SecProofThmMasterBer}

We begin by proving the bound~\eqref{EqnMaster} for the Bernoulli
unmasking sampler.

\subsubsection{Main argument}
\label{SecProofThmMasterMain}

Recall from equation~\eqref{EqnDefnRevealProcess} the reveal process
$\{\Xvar_t, t \in [0,1] \}$.  For any pair $0 < p < q < 1$, we let
$\Kexact{p}{q}$ denote the transition kernel of the reveal process in
moving from $X_p$ to $X_q$.  On the other hand, the unmasking sampler
is based on the updates in equation~\eqref{EqnRandSub} and
equation~\eqref{EqnRandUnmask}, and we
let $\Khat{p}{q}$ denote the associated transition kernel that moves
from $\Xhat_p$ to $\Xhat_q$.  We introduce the shorthand notation
\begin{align}
\label{EqnDefnUnmaskKL}  
\UnmaskKL(p, q) & \defn \Exs_{X_p}\left[ \KL\left( \Kexact{p}{q}(
  \mathord\cdot \mid X_p ) \,\middle\|\, \Khat{p}{q}( \mathord\cdot
  \mid X_p ) \right) \right],
\end{align}
corresponding to the averaged Kullback--Leibler (KL) discrepancy
between the two transition kernels over the interval $[p,q]$.
Here the expectation is taken over the marginal distribution
of the reveal process variable $X_p$.

\noindent The core technical result at the heart of~\Cref{ThmMaster} is
the following bound:
\mygraybox{
\begin{lemma}[Multiplicative control of the one-step unmasking defect]
\label{LemGeoOneStep}
For every $0 < p < q < 1$, we have
\begin{align}
\label{EqnGeoOneStep}
\UnmaskKL(p, q) & \leq \left\{ \frac{\odds(q)}{\odds(p)} - 1 \right\}
\Hinfo(p, q) \qquad \mbox{where $\odds(r) \defn \dfrac{r}{1 - r}$.}
\end{align}
\end{lemma}
}
\noindent See~\Cref{SecProofLemGeoOneStep} for the proof.

Given this lemma, the proof of~\Cref{ThmMaster} is very simple.
Recalling the definition~\eqref{EqnDefnCompletionDefect} of the
completion defect $\CompletionDefect_{\CompKernel}(T)$, we have
\begin{align*}
\KL\left( \Prob_Z \,\middle\|\, \Prob_{\Zhat} \right) &
\stackrel{(i)}{\leq} \KL(\Prob_{\rtime_0} \| \Qprob_{\rtime_0}) +
\sum_{\iter = 0}^{N - 1} \UnmaskKL(\rtime_\iter, \rtime_{\iter+1}) +
\CompletionDefect_{\CompKernel}(T) \\
& \stackrel{(ii)}{\leq} \KL(\Prob_{\rtime_0} \| \Qprob_{\rtime_0}) +
\sum_{\iter = 0}^{N - 1} \left\{
\frac{\odds(\rtime_{\iter+1})}{\odds(\rtime_\iter)} - 1 \right\}
\Hinfo(\rtime_\iter, \rtime_{\iter+1}) +
\CompletionDefect_{\CompKernel}(T),
\end{align*}
where step (i) follows from the KL chain rule, combined with the data
processing inequality, whereas step (ii) follows by applying the
bound~\eqref{EqnGeoOneStep} repeatedly to each of the KL increments.

\paragraph{Extension to estimated denoisers:}
Fix a state $x$ at step $\iter = 0, \ldots, N - 1$.  In the existing
kernel notation, $\Kexact{\rtime_\iter}{\rtime_{\iter + 1}}(
\mathord\cdot \mid x)$ is the exact block transition,
$\Khat{\rtime_\iter}{\rtime_{\iter + 1}}(\mathord\cdot \mid x)$ is the
frozen product transition formed from the exact one-site posteriors,
and $\Khat{\rtime_\iter}{\rtime_{\iter + 1}}[
  \DenoiseHat](\mathord\cdot \mid x)$ is the frozen product transition
formed from the estimated posteriors.  The two frozen kernels have the
same reveal probability $\beta_\iter \defn \frac{\rtime_{\iter + 1} -
  \rtime_\iter}{1 - \rtime_\iter}$.  Since the exact block transition
and $\Khat{\rtime_\iter}{\rtime_{\iter + 1}}(\mathord\cdot \mid x)$
have the same one-coordinate marginals, splitting the log-likelihood
ratio through the product of these marginals gives
\begin{align*}
\KL\left( \Kexact{\rtime_\iter}{\rtime_{\iter + 1}}( \mathord\cdot
\mid x) \,\middle\|\, \Khat{\rtime_\iter}{\rtime_{\iter + 1}}[
  \DenoiseHat](\mathord\cdot \mid x) \right) & = \KL\left(
\Kexact{\rtime_\iter}{\rtime_{\iter + 1}}( \mathord\cdot \mid x)
\,\middle\|\, \Khat{\rtime_\iter}{\rtime_{\iter + 1}}(\mathord\cdot
\mid x) \right) + \beta_\iter \sum_{\ind \in \Masked(x)} \KL\left(
\Denoise_\ind(\mathord\cdot, x) \,\middle\|\,
\DenoiseHat_\ind(\mathord\cdot, x) \right).
\end{align*}
Indeed, for each $\ind \in \Masked(x)$, both one-coordinate kernels
leave the coordinate masked with probability $1 - \beta_\iter$.
Conditional on revealing it, they use $\Denoise_\ind(\mathord\cdot,
x)$ and $\DenoiseHat_\ind(\mathord\cdot, x)$, respectively.  Their
one-coordinate KL divergence is given by $\beta_\iter \KL\left(
\Denoise_\ind(\mathord\cdot, x) \,\middle\|\,
\DenoiseHat_\ind(\mathord\cdot, x) \right)$.  Averaging the first KL
term over $x = X_{\rtime_\iter}$ under the exact reveal-process law
gives exactly $\UnmaskKL(\rtime_\iter, \rtime_{\iter + 1})$.
Averaging the decomposition and summing over $\iter = 0, \ldots, N -
1$ in the same KL chain-rule argument adds the term displayed after
equation~\eqref{EqnMaster}, which completes the proof.


\subsubsection{Proof of~\Cref{LemGeoOneStep}}
\label{SecProofLemGeoOneStep}

Our proof is based on the exact one-step representation
\begin{subequations}
\begin{align}
\label{EqnOneStepExact}
\UnmaskKL(p, q) & = \int_p^q (q - u) \hfun'(u) \, du,
\end{align}
which we return to prove below.  Taking this claim as given, let us
prove the bound~\eqref{EqnGeoOneStep}.  Introduce the shorthand $f(u)
= u \, (1 - u)$, and observe that $f(u) \geq 0$ for $u \in [0,1]$.
Note that for every $u \in [p, q]$, we have $(q - u) \leq (q - p)$,
and $f(u) \geq p (1 - q)$, so that
\begin{align}
\label{EqnItchy}
q - u \; \leq \; \frac{q - p}{p (1 - q)} f(u) & \stackrel{(i)}{=}
\left\{ \frac{q (1 - p)}{p (1 - q)} - 1 \right\} f(u) \;
\stackrel{(ii)}{=} \; \left\{ \frac{\odds(q)}{\odds(p)} - 1
\right\} f(u)
\end{align}
\end{subequations}
where step (i) follows from simple algebra; and step (ii) uses the
definition $\odds(r) = r/(1-r)$.
Thus, we can write
\begin{align*}
  \UnmaskKL(p, q) & \stackrel{(iii)}{=} \int_p^q (q - u) \hfun'(u) \,
  du \stackrel{(iv)}{\leq} \left\{ \frac{\odds(q)}{\odds(p)} - 1
  \right\} \int_p^q f(u) \; \hfun'(u) du \; \stackrel{(v)}{=} \;
  \left\{ \frac{\odds(q)}{\odds(p)} - 1 \right\} \Hinfo(p,q),
\end{align*}
where step (iii) follows from the exact
representation~\eqref{EqnOneStepExact}; step (iv) follows from the
inequality~\eqref{EqnItchy} along with the fact that $\hfun'(t) \geq
0$ for all $t \in [0,1]$; and step (v) follows from the
definition~\eqref{EqnDefnUGC} of the \ugc function. Thus, we have
established the claim~\eqref{EqnGeoOneStep}.

\paragraph{Proof of the exact representation~\eqref{EqnOneStepExact}:}

We first claim that $\hfun$ satisfies the 
relation
\begin{subequations}
\begin{align}
  \label{EqnMaskInfoRate}
\hfun(u) & = S(Z) - \frac{d}{du} \Info\left( \Zvar; X_u \right),
\end{align}
where we have introduced the shorthand $S(Z) \defn \sum_{i =
  1}^{\usedim} \Ent(\Zvar_i)$ for the entropy sum.  To verify the
claim~\eqref{EqnMaskInfoRate}, fix some $\delta > 0$ sufficiently
small.  Conditional on $X_u$, each coordinate $i \in \Masked(X_u)$ is
revealed by time $u + \delta$ with probability $\delta / (1 - u)$.
Revealing two or more coordinates has probability $O(\delta^2)$, while
revealing only coordinate $i$ contributes its current posterior
entropy.  Thus, dividing the first-order expansion of the information
increment by $\delta$ and letting $\delta \downarrow 0$ gives
\begin{align*}
  \frac{d}{du} \Info\left( \Zvar; X_u \right) & = \frac{1}{1 - u}
  \sum_{i = 1}^{\usedim} \Exs\left[ \mathbf{1}\big\{ i \in
    \Masked(X_u) \big\} \Ent\left( \Zvar_i \mid X_u \right) \right] \;
  = \; \sum_{i = 1}^{\usedim} \Ent\left( \Zvar_i \mid X_u,\, i \in
  \Masked(X_u) \right).
\end{align*}
Since the event $i \in \Masked(X_u)$ is independent of $\Zvar_i$, we
have the equivalence
\begin{align*}
  \Ent\left( \Zvar_i \mid X_u,\, i \in \Masked(X_u) \right) =
  \Ent(\Zvar_i) - \Info\left( \Zvar_i; X_u \mid i \in \Masked(X_u)
  \right).
\end{align*}
Summing over $i$ and using the definitions of $S(Z)$ and $\hfun(u)$
proves the claim~\eqref{EqnMaskInfoRate}.

We now prove the exact representation.  Conditional on $X_p = x$, let
$A \subseteq \Masked(x)$ be the set of coordinates revealed between
$p$ and $q$.  The exact and frozen-product kernels induce the same law
for the random subset $A$. Conditional on $A$, they unmask coordinates
using, respectively, the joint posterior of $\Zmask{A}$ and the
product of its one-site posterior marginals.  Consequently, we can
decompose the KL divergence as
\begin{align*}
  \KL\left( \Kexact{p}{q}( \mathord\cdot \mid x ) \,\middle\|\,
  \Khat{p}{q}( \mathord\cdot \mid x ) \right) & = \underbrace{
    \Exs\left[ \sum_{i \in A} \Ent\left( \Zvar_i \mid X_p = x \right)
      \,\middle|\, X_p = x \right] }_{\Term_1} - \underbrace{
    \Exs\left[ \Ent\left( \Zmask{A} \mid X_p = x, A \right)
      \,\middle|\, X_p = x \right] }_{\Term_2}.
\end{align*}
Note that each coordinate still masked at time $p$ is included in $A$
with probability $\beta_{p, q} \defn \frac{q - p}{1 - p}$. Averaging
over $X_p$ and using the definition~\eqref{EqnDefnHfun} of $\hfun$, we
find that
\begin{align*}
  \Exs[\Term_1] & = \beta_{p, q} (1 - p) \left\{ S(\Zvar) - \hfun(p)
  \right\} \; = \; (q - p) \left\{ S(\Zvar) - \hfun(p) \right\}.
\end{align*}
Moreover, conditional on $X_p$, the new part of $X_q$ is $(A,
\Zmask{A})$.  Since $A$ is conditionally independent of $\Zvar$ and
the random sub-vector $\Zmask{A}$ is a function of $(\Zvar, A)$, we
have
\begin{align*}
  \Exs[\Term_2] & = \Info\left( \Zvar; A, \Zmask{A} \mid X_p \right)
  \; = \; \Info\left( \Zvar; X_q \mid X_p \right).
\end{align*}
Combining these two identities yields
  \begin{align}
    \label{EqnUnmaskKLInformation}
  \UnmaskKL(p, q) & = (q - p) \left\{ S(Z) - \hfun(p) \right\} -
  \Info\left( \Zvar; X_q \mid X_p \right).
\end{align}

It remains to analyze the trailing mutual information term.
We have
\begin{align}
\label{EqnInter}  
  \Info\left( \Zvar; X_q \mid X_p \right) & \stackrel{(i)}{=}
  \Info\left( \Zvar; X_q \right) - \Info\left( \Zvar; X_p \right) \;
  \stackrel{(ii)}{=} \; \int_p^q \frac{d}{du} \Info\left( \Zvar; X_u
  \right) du \; \stackrel{(iii)}{=}\; \int_p^q \left\{ S(Z) - \hfun(u)
  \right\} \, du,
\end{align}
\end{subequations}
where step (i) follows because $\Zvar \longrightarrow X_q
\longrightarrow X_p$ is a Markov chain; step (ii) follows from the
fundamental theorem of calculus; and step (iii) follows from the
identity~\eqref{EqnMaskInfoRate}.

Finally, we have
\begin{align*}
  \UnmaskKL(p, q) & \stackrel{(iv)}{=} \int_p^q \left\{ \hfun(u) -
  \hfun(p) \right\} \, du \; \stackrel{(v)}{=} \; \int_p^q \int_p^u
  \hfun'(v) \, dv \, du \; \stackrel{(vi)}{=} \; \int_p^q (q - v)
  \hfun'(v) \, dv,
\end{align*}
where equality (iv) follows by substituting the
representation~\eqref{EqnInter} into
equation~\eqref{EqnUnmaskKLInformation} and simplifying terms; step
(v) follows from the fundamental theorem of calculus; and step (vi)
follows from Fubini's theorem.  This proves the exact
representation~\eqref{EqnOneStepExact}.

\subsection{Proof of~\Cref{ThmMaster}: Exact cardinality sampling}
\label{SecProofThmMasterExact}

For integers $0 \leq a < b \leq \usedim - 1$, let $A$ be a uniformly
random $a$-subset of $[\usedim]$, independent of $\Zvar$, and define
\begin{align*}
X^{[a]} & \defn \Big( \Zmask{A}, \MyMask{A^c} \Big) \qquad \mbox{where
  $A \sim \operatorname{Unif}\Big\{ S \subseteq [\usedim] \,\Big|\,
  \card(S) = a \Big\}$},
\end{align*}
where $A$ is a subset of cardinality $a$, chosen uniformly at random,
and independently of $\Zvar$.  Conditional on $A$, we choose $(B \mid
A) \sim \operatorname{Unif}\Big\{ S \subseteq A^c \,\Big|\, \card(S)
= b - a \Big\}$, corresponding to the uniform-subset
rule~\eqref{EqnRandSubUni} that defines the exact cardinality sampler.

Write $\Kexact{a}{b}^{\mathrm{card}}$ for the exact transition from
$X^{[a]}$ to cardinality $b$, and $\Khat{a}{b}^{\mathrm{card}}$ for
the transition generated by the \Card-unmasking sampler.  In analogy
with the Bernoulli defect~\eqref{EqnDefnUnmaskKL}, define the one-step
KL defect
\begin{subequations}
\begin{align}
\label{EqnDefnCardUnmaskKL}
\UnmaskCard(a, b) & \defn \Exs_{X^{[a]}}\left[ \KL\left(
  \Kexact{a}{b}^{\mathrm{card}}(\mathord\cdot \mid X^{[a]})
  \,\middle\|\, \Khat{a}{b}^{\mathrm{card}}(\mathord\cdot \mid
  X^{[a]}) \right) \right].
\end{align}
The kernels use the same conditional law for $B$.  Conditional on
$(X^{[a]}, B)$, the exact kernel samples from $\Law(\Zvar_B \mid
\Zvar_A)$, whereas the sampler uses $\bigotimes_{i \in B} \Law(\Zvar_i
\mid \Zvar_A)$, corresponding to the product
update~\eqref{EqnDefnRandUmask}.

We now apply the one-block specialization of the exact error
representation of Chen et al.~\cite[Theorem~3.3]{CheEtAl25}.  After
translation\footnote{To be clear, their information curve at index $j
+ 1$ is $\hcard_j / \usedim$, and their previous cardinality and block
size are $a$ and $b - a$.  Since $B$ is recoverable from the input and
output mask patterns, retaining the subset partition does not change
this transition KL divergence.} to our normalization and notation, it
guarantees that
\begin{align}
\label{EqnCardOneStepExact}
\UnmaskCard(a, b) & = \frac{1}{\usedim} \sum_{j = a + 1}^{b - 1} (b -
j) \big\{ \hcard_j - \hcard_{j - 1} \big\}.
\end{align}
\end{subequations}

We claim that this exact representation admits the following upper
bound:
\mygraybox{
\begin{lemma}[Fixed-cardinality one-step unmasking defect]
\label{LemCardOneStep}
For all integers $0 \leq a < b \leq \usedim - 1$, we have
\begin{align}
\label{EqnCardOneStep}
\UnmaskCard(a, b) & \leq \left\{ \frac{\odds(b / \usedim)}
           {\odds\big( (a + 1) / \usedim \big)} - 1 \right\}
           \Hcard(a, b).
\end{align}
\end{lemma}
}
\begin{proof}
If $b = a + 1$, both sums are empty.  Otherwise, we introduce the
shorthand
\begin{align*}
\rho_{a, b} \defn \frac{\odds(b / \usedim)} {\odds\big( (a + 1)
  / \usedim \big)} - 1 \; = \; \frac{\usedim (b - a - 1)} {(a +
  1)(\usedim - b)},
\end{align*}
and make note of the upper bound $\frac{\usedim (b - j)} {j (\usedim -
  j)} \leq \rho_{a, b}$, valid for every $a + 1 \leq j \leq b - 1$.
Since $\hcard_j$ is non-decreasing, the increments $\hcard_j -
\hcard_{j - 1}$ are nonnegative.  Putting together the pieces, we
arrive at the upper bound
\begin{align*}
\UnmaskCard(a, b) & = \frac{1}{\usedim} \sum_{j = a + 1}^{b - 1} (b -
j) \big\{ \hcard_j - \hcard_{j - 1} \big\} \; \leq \; \rho_{a, b}
\frac{1}{\usedim^2} \sum_{j = a + 1}^{b - 1} j (\usedim - j) \big\{
\hcard_j - \hcard_{j - 1} \big\} \; = \; \rho_{a, b} \Hcard(a, b),
\end{align*}
where the last equality follows from the
definition~\eqref{EqnDefnHcardUGC} of $\Hcard$.  This completes the
proof of the upper bound~\eqref{EqnCardOneStep}.
\end{proof}

Note that~\Cref{LemCardOneStep} is the fixed-cardinality analogue of
the Bernoulli sampler bound given in~\Cref{LemGeoOneStep}.
Consequently, the KL bound~\eqref{EqnMasterCard} for the
fixed-cardinality sampler follows from the same argument, instead
using~\Cref{LemCardOneStep} to bound the one-step errors.


\subsection{Proof of data-certification results}
\label{SecProofData}

In this section, we provide proofs of our data-based certification
results, including that of~\Cref{LemDinfoSandwich}
in~\Cref{SecProofLemDinfoSandwich}; as well as that
of~\Cref{PropDataSingle} in~\Cref{SecProofPropDataSingle}.


\subsubsection{Proof of~\Cref{LemDinfoSandwich}}
\label{SecProofLemDinfoSandwich}

We break our proof into parts, one for each claim in the statement.

\paragraph{Proof of the claim~\eqref{EqnDinfoExact}:}
We first prove the exact expression~\eqref{EqnDinfoExact}.  For each
coordinate $i$, let $\NewXvar{i}{p}$ and $\NewXvar{i}{q}$ denote the
coupled observations at times $p$ and $q$, with coordinate $i$ kept
masked. Note that the observation $\NewXvar{i}{p}$ is obtained from
$\NewXvar{i}{q}$ by a random remasking operation independent of
$\Zvar_i$ conditional on $\NewXvar{i}{q}$. Consequently, the triple
$\Zvar_i \longrightarrow \NewXvar{i}{q} \longrightarrow
\NewXvar{i}{p}$ forms a Markov chain, so that we can write
\begin{subequations}
\begin{align}
  \label{EqnMarkovIncrement}
  \Info(\Zvar_i; \NewXvar{i}{q} \mid \NewXvar{i}{p}) & = \Info\left(
  \Zvar_i; \NewXvar{i}{q} \right) - \Info\left( \Zvar_i;
  \NewXvar{i}{p} \right).
\end{align}

At the same time, by the conditional KL representation of conditional
mutual information, we have
\begin{align}
  \Info(\Zvar_i; \NewXvar{i}{q} \mid \NewXvar{i}{p})
  & = \Exs_{(\NewXvar{i}{p}, \NewXvar{i}{q})} \left[
    \KL\left(
      \Law(\Zvar_i \mid \NewXvar{i}{q}, \NewXvar{i}{p})
      \,\middle\|\,
      \Law(\Zvar_i \mid \NewXvar{i}{p})
    \right)
  \right] \notag\\
  \label{EqnTatte}
  & = \Exs_{(\NewXvar{i}{p}, \NewXvar{i}{q})} \left[
    \KL\left(
      \Law(\Zvar_i \mid \NewXvar{i}{q})
      \,\middle\|\,
      \Law(\Zvar_i \mid \NewXvar{i}{p})
    \right)
  \right],
\end{align}
where the second equality follows from the Markov relation
$\Zvar_i \longrightarrow \NewXvar{i}{q} \longrightarrow
\NewXvar{i}{p}$, which implies
\begin{align*}
  \Law(\Zvar_i \mid \NewXvar{i}{q}, \NewXvar{i}{p})
  & = \Law(\Zvar_i \mid \NewXvar{i}{q}).
\end{align*}
Combining equations~\eqref{EqnMarkovIncrement} and~\eqref{EqnTatte}
yields the identity
\begin{align*}
  \Exs\left[ \KL\left( \Law(\Zvar_i \mid \NewXvar{i}{q}) \,\middle\|\,
    \Law(\Zvar_i \mid \NewXvar{i}{p}) \right) \right] \; = \;
  \Info\left( \Zvar_i; \NewXvar{i}{q} \right) - \Info\left( \Zvar_i;
  \NewXvar{i}{p} \right).
\end{align*}
Summing both sides of this identity over the indices $i = 1, \ldots,
\usedim$, and using the definitions of $\Dinfo$ and $\hfun$ in
equation~\eqref{EqnDefnDinfoMask} and
equation~\eqref{EqnDefnHfun}, respectively,
we find that
\begin{align}
\label{EqnExactIdentityTwo}  
  \frac{\Dinfo(p, q)}{q-p} & = \hfun(q) - \hfun(p),
\end{align}
as claimed in equation~\eqref{EqnDinfoExact}.

\paragraph{Proof of the sandwich relation~\eqref{EqnDinfoSandwich}:}
Turning to the claim~\eqref{EqnDinfoSandwich}, for any pair $0 < p < q
< 1$, we have
\begin{align*}
p (1 - q) \; \leq \; \rtime (1 - \rtime) \; \leq \; q (1 - p) \qquad
\mbox{for all $\rtime \in [p, q]$.}
\end{align*}
Multiplying both sides by $\hfun'(\rtime) \geq 0$ and integrating
yields
\begin{align*}
  p (1 - q) \big \{ \hfun(q) - \hfun(p) \big \} \; \leq \; \int_p^q
  \rtime (1 - \rtime) \hfun'(\rtime) \, d \rtime \; \leq \; q (1 - p)
  \big \{ \hfun(q) - \hfun(p) \big \}.
\end{align*}
Using the definition~\eqref{EqnDefnUGC} of $\Hinfo(p, q)$ and the
exact identity~\eqref{EqnExactIdentityTwo}, we have proved that
\begin{align}
\label{EqnAlmostSandwich}  
  \frac{p (1 - q)}{q - p} \Dinfo(p, q) \; \leq \; \Hinfo(p, q) \; \leq
  \; \frac{q (1 - p)}{q - p} \Dinfo(p, q).
\end{align}
\end{subequations}
Finally, an elementary calculation gives $\myrho \defn
\frac{\odds(q)}{\odds(p)} - 1 = \frac{q - p}{p (1 - q)}$, so
that we have
\begin{align*}
  \frac{p (1 - q)}{q - p} & = \frac{1}{\myrho}, \quad \mbox{and} \quad
  \frac{q (1 - p)}{q - p} = \frac{1 + \myrho}{\myrho}.
\end{align*}
Combining with equation~\eqref{EqnAlmostSandwich} proves the
claim~\eqref{EqnDinfoSandwich}.


\subsubsection{Proof of~\Cref{PropDataSingle}}
\label{SecProofPropDataSingle}

Introduce the shorthand $\mu(p, q) \defn \Exs[\QhatSam] = \sum_{j =
  0}^{J - 1} \Dinfo(v_j, v_{j + 1})$ for the expected value of
$\QhatSam$.  Note that we have
\begin{subequations}
\begin{align}
\label{EqnProofProxySandwich}
\mu(p, q) & \leq \Hinfo(p, q) \leq 2 \mu(p, q),
\end{align}
as previously stated in equation~\eqref{EqnQhatSamSand}.  We first
derive the required moment bound for the trajectory statistic from the
denoiser condition~\eqref{EqnMomentCondition}. Set $r \defn \alpha /
2$, so that $r \geq 2$.  Introducing the shorthand notation
\begin{align*}
K_j & \defn \sum_{i = 1}^{\usedim} \KL\left(
  \Denoise_i\left( \mathord\cdot,
  \XvarForce{v_{j + 1}} \right) \,\middle\|\, \Denoise_i\left(
  \mathord\cdot, \XvarForce{v_j} \right) \right),
\qquad \mbox{for $j = 0, \ldots, J - 1$,}
\end{align*}
our statistic $\QhatSam$ can be written as $\QhatSam = \sum_{j = 0}^{J
  - 1} (v_{j + 1} - v_j) K_j$.  We then apply Minkowski's inequality
and make use of the moment condition~\eqref{EqnMomentCondition},
thereby obtaining the moment bound
\begin{align*}
\left\{ \Exs\left[ (\QhatSam)^r \right]
  \right\}^{1 / r} & \leq \sum_{j = 0}^{J - 1} (v_{j + 1} - v_j)
  \left\{ \Exs\left[ K_j^r
  \right] \right\}^{1 / r} \; \leq \; \sum_{j = 0}^{J - 1} (v_{j + 1}
- v_j) B_\alpha = \underbrace{(q - p) B_\alpha}_{\equiv M_\alpha}.
\end{align*}

In terms of the shorthand $\Yup{\samind} \defn 2 \min\{
\Qfunup{\samind}, \tau \}$, we can write $\Hhat(p, q) =
\frac{1}{\numobs} \sum_{\samind = 1}^{\numobs} \Yup{\samind}$.  We
first bound the truncation bias. Since $r > 1$, we have
\begin{align}
\label{EqnMaskTruncatedBias}
0 \leq 2 \mu(p, q) - \Exs[\Yup{\samind}] & = 2 \Exs\left[ (\QhatSam -
  \tau)_+ \right] \; \leq \;\frac{ 2 \Exs\left[ (\QhatSam)^r \right]
}{ \tau^{r - 1} } \leq \frac{ 2 M_\alpha^r }{ \tau^{r - 1} }.
\end{align}

We are now set up to apply the empirical Bernstein inequality of
Maurer and Pontil~\cite[Theorem~4]{MauPont09}.  We apply it to both
the i.i.d. variables $\Yup{\samind} / (2 \tau) \in [0, 1]$, and then
separately to the i.i.d. variables $1 - \Yup{\samind} / (2 \tau)$.
Observe that the two collections have the same empirical variance.  We
apply a failure probability of $\eta/2$ to each application, so that a
union bound over both tail bounds has failure probability at most
$\eta$.  Via this argument, we are guaranteed to have the two-sided
tail bound
\begin{align}
\label{EqnMaskEmpiricalBernstein}
\left| \Exs[\Yup{\samind}] - \Hhat(p, q) \right| & \leq \sqrt{ \frac{
    2 \Vhat \log(4 / \eta) }{ \numobs } } + \frac{ 14 \tau \log(4 /
  \eta) }{ 3 (\numobs - 1) }
\end{align}
with probability at least $1 - \eta$.

Combining the truncation bias~\eqref{EqnMaskTruncatedBias} with the
tail bound~\eqref{EqnMaskEmpiricalBernstein}, we find that
\begin{align*}
\left| 2 \mu(p, q) - \Hhat(p, q) \right| & \leq \sqrt{ \frac{ 2 \Vhat
    \log(4 / \eta) }{ \numobs } } + \frac{ 14 \tau \log(4 / \eta) }{ 3
  (\numobs - 1) } + \frac{ 2 M_\alpha^r }{ \tau^{r - 1} }
\end{align*}
with probability at least $1 - \eta$.  We claim that the right-hand side of
this bound is precisely $\HackErr$.  In particular, we can express the
truncation level~\eqref{EqnDefnHhat} as $\tau = M_\alpha c_\eta^{-1 /
  r}$, where $c_\eta \defn \frac{ 7 \log(4 / \eta) }{ 3 (\numobs - 1)
}$.  In terms of this shorthand, we have $\frac{ 14 \tau \log(4 /
  \eta) }{ 3 (\numobs - 1) } = 2 M_\alpha c_\eta^{1 - 1 / r}$ and
$\frac{ 2 M_\alpha^r }{ \tau^{r - 1} } = 2 M_\alpha c_\eta^{1 - 1 /
  r}$.  Since $1 / r = 2 / \alpha$, the claimed equivalence follows,
and we conclude that
\begin{align}
\label{EqnMaskScaledProxyEvent}
\left| 2 \mu(p, q) - \Hhat(p, q) \right| & \leq \HackErr
\end{align}
\end{subequations}
with probability at least $1 - \eta$.

It remains to prove the sandwich~\eqref{EqnDataSingle} claimed in the
statement.  Conditioned on the event~\eqref{EqnMaskScaledProxyEvent},
the upper inequality in equation~\eqref{EqnProofProxySandwich} gives
\begin{align*}
\Hinfo(p, q) & \leq 2 \mu(p, q) \leq \Hhat(p, q) + \HackErr.
\end{align*}
For the other side, the same concentration event and the lower
inequality in equation~\eqref{EqnProofProxySandwich} give
\begin{align*}
\Hhat(p, q) + \HackErr & \leq 2 \mu(p, q) + 2 \HackErr \leq 2 \big\{
\Hinfo(p, q) + \HackErr \big\}.
\end{align*}
Combining the last two displays proves the claim~\eqref{EqnDataSingle}.

%


\subsection{Proof of~\Cref{ThmCertifiedMulti}}
\label{SecProofThmCertifiedMulti}

We need to verify two separate claims: the stated
bound~\eqref{EqnCertifiedMulti} on the KL divergence, and the fact
that the procedure terminates in at most $N$ unmasking rounds.
Throughout the proof, we make use of the shorthand
\begin{align}
  \label{EqnDefnAshort}
  \ashort_\bind \defn \frac{4
    \sqrt{\PartCompHat(\Partition)}}{\Nscore}
  \sqrt{\frac{\Sinfo_\bind}{\HhatBind{\bind} + \rhat_\bind}} \qquad
  \mbox{so that $\rhohat_\bind = \min\{1, \ashort_\bind\}$.}
\end{align}

\subsubsection{Verifying the KL bound~\eqref{EqnCertifiedMulti}}

Our first step is to apply~\Cref{PropDataSingle} with failure
probability \mbox{$\eta /\Btot \in (0,1)$} to each block
$[\block_\bind, \block_{\bind + 1}]$.  The estimator in this
proposition gives us an estimate $\HhatBind{\bind} \geq 0$ along with
the upper confidence correction $\rhat_\bind > 0$.  By the union bound
over all $\Btot$ blocks, with probability at least $1 - \eta$, we are
guaranteed to have
\begin{align}
  \label{EqnGoodEvent}  
  \Hinfo_\bind & \leq \HhatBind{\bind} + \rhat_\bind \qquad \mbox{for
    all $\bind = 0, \ldots, \Btot - 1$,}
\end{align}
where we have used inequality (A) in equation~\eqref{EqnDataSingle}.
We condition on this ``good event'' throughout the remainder of the
proof.

For each block, we define the iteration count $\Nhat_\bind \defn
\left\lceil \frac{\Sinfo_\bind}{\log(1 + \rhohat_\bind)}
\right\rceil$, using the stepsizes $\rhohat_\bind$ specified in the
theorem statement.  Observe that from our choice of $\rhohat_\bind$
and the definition~\eqref{EqnDefnGeoRho} of the geometric schedule, it
takes at most $\Nhat_\bind$ iterations to traverse the block, and
every step inside this block has reveal-odds multiplier at most $1 +
\rhohat_\bind$.  We now apply~\Cref{ThmMaster} to this concatenated
multi-block schedule.  By the additivity~\eqref{EqnHinfoAdditive} of
$\Hinfo$, the contribution from block $\bind$ takes the form
$\rhohat_\bind \Hinfo_\bind$.  Summing up these terms, along with the
boundary correction, yields
\begin{align*}
  \KL(\Prob_\Zvar \| \Prob_{\Zhat}) & \leq \sum_{\bind = 0}^{\Btot -
    1} \rhohat_\bind \Hinfo_\bind + \BOUNDARY \\
  & \stackrel{(i)}{\leq} \sum_{\bind = 0}^{\Btot - 1} \rhohat_\bind
  (\HhatBind{\bind} + \rhat_\bind) + \BOUNDARY \\
  & \stackrel{(ii)}{\leq} \sum_{\bind = 0}^{\Btot - 1} \; \;
  \underbrace{\tfrac{4 \sqrt{\PartCompHat(\Partition)}}{\Nscore}
    \sqrt{\tfrac{\Sinfo_\bind}{\HhatBind{\bind} + \rhat_\bind}}}_{ =
    \ashort_\bind} \; (\HhatBind{\bind} + \rhat_\bind) + \BOUNDARY \\
& = \frac{4 \sqrt{\PartCompHat(\Partition)}}{\Nscore} \sum_{\bind =
    0}^{\Btot - 1} \sqrt{\Sinfo_\bind (\HhatBind{\bind} +
    \rhat_\bind)} + \BOUNDARY \\
  & \stackrel{(iii)}{=} \frac{4 \PartCompHat(\Partition)}{\Nscore} +
  \BOUNDARY,
\end{align*}
where step (i) uses the good event bound~\eqref{EqnGoodEvent}; step
(ii) uses the fact that $\rhohat_\bind \leq \ashort_\bind$, along with
the definition of $\ashort_\bind$; and equality (iii) follows from the
definition of $\PartCompHat(\Partition)$.


\subsubsection{Verifying the score-evaluation budget}

It remains to verify that the procedure terminates in at most $N$
unmasking rounds, so that it satisfies the prescribed score evaluation
budget.  Since $\rhohat_\bind \in [0, 1]$ and $\log(1 + u) \geq u / 2$
for $u \in [0, 1]$, we have
\begin{align*}
  \Nhat_\bind \; \leq \; 1 + \frac{\Sinfo_\bind} {\log(1 +
    \rhohat_\bind)} \; \leq \; 1 + \frac{2
    \Sinfo_\bind}{\rhohat_\bind} & = 1 + 2 \Sinfo_\bind \max\{1,
  \tfrac{1}{\ashort_\bind} \} \\
  & \leq 1 + 2 \Sinfo_\bind + \frac{\Nscore}{2
    \sqrt{\PartCompHat(\Partition)}} \sqrt{\Sinfo_\bind
    (\HhatBind{\bind} + \rhat_\bind)}.
\end{align*}
We now sum over blocks, using the definition~\eqref{EqnLogRevealBlock}
of $\Sinfo_\bind$, and the definition of $\PartCompHat(\Partition)$.
Doing so yields
\begin{align*}
  \sum_{\bind = 0}^{\Btot - 1} \Nhat_\bind & \leq \Btot + 2
  \sum_{\bind = 0}^{\Btot - 1} \Sinfo_\bind + \frac{\Nscore}{2
    \sqrt{\PartCompHat(\Partition)}} \sum_{\bind = 0}^{\Btot - 1}
  \sqrt{\Sinfo_\bind (\HhatBind{\bind} + \rhat_\bind)} \\
  & = \Btot + 2 \big\{ \logit(\newrfinal) -
  \logit(\newrinit) \big\} +
  \frac{\Nscore}{2} \\
  & \leq \Nscore,
\end{align*}
where the final inequality follows from the assumed lower bound
$\Nscore \geq 2 \Btot + 4 \big\{ \logit(\newrfinal) -
\logit(\newrinit) \big\}$.

\subsection{Proof of~\Cref{ThmFineEuler}}
\label{SecProofThmFineEuler}

We split our proof into two parts, corresponding to the two claims
in the theorem.

\subsubsection{Proof of the claim~\eqref{EqnFinePartSandwich}}

Introducing the shorthand $\fdens(\lam) \defn \sqrt{\qdens(\lam)}$ and
$\Fstar \defn \int_{-\Elld}^{\Elld} \fdens(\lam) \, d\lam$, we need to
prove that $\inf_{\Partition} \PartComp(\Partition) = \Fstar^2$.  By
the definition~\eqref{EqnDefnPartHfine}, for a partition of $[-\Elld,
  \Elld]$ into intervals $L_\bind$ with length $|L_\bind|$, applying
the Cauchy--Schwarz inequality on each interval gives
\begin{align}
\label{EqnLower}  
  \sqrt{\PartComp(\Partition)} & = \sum_\bind \sqrt{ |L_\bind|
    \int_{L_\bind} \qdens(\lam) \, d\lam } \geq \sum_\bind
  \int_{L_\bind} \fdens(\lam) \, d\lam = \Fstar.
\end{align}
Conversely, let $\UniK$ be a uniform $\Btot$-block partition with
blocks $\{L_k\}$. By the integral form of the mean-value theorem, we
can write $\int_{L_k} \qdens(\lam) d \lam \; = \; |L_k|
\qdens(\xi_\bind)$ for some $\xi_\bind \in L_\bind$, and consequently,
we obtain
\begin{align*}
\sqrt{\inf_\Partition \PartComp(\Partition)} \leq \sqrt{
  \PartComp(\UniK)} & = \sum_\bind \sqrt{ |L_\bind| \int_{L_\bind}
  \qdens(\lam) \, d\lam } \; = \sum_\bind \fdens(\xi_\bind) |L_\bind|.
\end{align*}
The quantity on the right-hand side is a Riemann sum for $\fdens$;
taking the limit as $\Btot \rightarrow +\infty$ shrinks the intervals
to zero, and so it converges to $\Fstar$.  We have thus shown that
$\inf_\Partition \PartComp(\Partition) \leq \Fstar^2$.  Combined with
the lower bound~\eqref{EqnLower}, the proof is complete.


\subsubsection{Proof of the claim~\eqref{EqnFineEulerOptimality}}

Recall our shorthand $\fdens = \sqrt{\qdens}$ and $\Fstar =
\int_{-\Elld}^{\Elld} \fdens(\lam) \, d\lam$.  For a pair $-\Elld \leq
x < y \leq \Elld$, introduce the shorthand $\gamma(x, y) \defn
\UnmaskKL\big( \newrevinv(x), \newrevinv(y) \big)$ where $R(\lam) =
\frac{e^\lam}{1 + e^\lam}$.  By changing variables in
equation~\eqref{EqnOneStepExact}, and using the derivative
$\newrevinv'(t) = \newrevinv(t) \{ 1 - \newrevinv(t) \}$, we obtain
\begin{subequations}
  \begin{align}
    \label{EqnCalifornia}
  \gamma(x, y) = \int_x^y \qdens(t) \frac{\newrevinv(y) -
    \newrevinv(t)}{\newrevinv'(t)} \, dt.
\end{align}

Define the constants $m \defn \min_{[-\Elld,\Elld]} \newrevinv' > 0$,
$M \defn \max_{\lam \in [-\Elld,\Elld]} \newrevinv'(\lam) < \infty$,
and $q_{\min} \defn \min_{\lam \in [-\Elld,\Elld]} \qdens(\lam) > 0$.
We then have $\newrevinv(y) - \newrevinv(t) \geq m (y - t)$ and
$\newrevinv'(t) \leq M$, from which equation~\eqref{EqnCalifornia}
implies that
\begin{align}
\label{EqnQuadLower}
  \gamma(x, y) \geq c_0 (y - x)^2 \qquad \mbox{where $c_0 \defn
    \frac{q_{\min} m}{2 M}$.}
\end{align}
In addition, we claim that
\begin{align}
  \label{EqnFineEulerLocalExpansion}
  \gamma(x, x + \delta) = \frac{1}{2} \qdens(x) \delta^2 + o(\delta^2)
  \qquad \mbox{uniformly in $x$.}
\end{align}
Writing $t = x + s$, uniform continuity of $\newrevinv'$ on $[-\Elld,
  \Elld]$, together with the fact that it is bounded away from zero
over this interval, imply that
\begin{align*}
  \frac{\newrevinv(x + \delta) - \newrevinv(x + s)}{ \newrevinv'(x +
    s)} & = (\delta - s) \{ 1 + o(1) \}
\end{align*}
uniformly in $x$ and $0 \leq s \leq \delta$.  Similarly, the
continuity of $\qdens$ gives $\qdens(x + s) = \qdens(x) + o(1)$
uniformly.  Substituting these two estimates into
equation~\eqref{EqnCalifornia} and integrating over $s \in [0,
  \delta]$ yields the local expansion
bound~\eqref{EqnFineEulerLocalExpansion}.

\paragraph{Upper bound:}
For the upper bound, we choose the grid $\{ \lam_{k, N} \}_{k = 0}^N$
to have equal $\fdens$-mass, so that $\int_{\lam_{k, N}}^{\lam_{k + 1,
    N}} \fdens(\lam) \, d\lam = \frac{\Fstar}{N}$.  Since $\fdens$ is
continuous and bounded away from zero, this mesh has width
$O(N^{-1})$, and hence we have $\qdens(\lam_{k, N}) (\lam_{k + 1, N} -
\lam_{k, N})^2 = \frac{\Fstar^2}{N^2} + o(N^{-2})$ uniformly in $k$.
Equation~\eqref{EqnFineEulerLocalExpansion} therefore gives
\begin{align}
  \label{EqnFineEulerUpperProof}
  \sum_{k = 0}^{N - 1} \gamma(\lam_{k, N}, \lam_{k + 1, N}) =
  \frac{\Fstar^2}{2 N} + o(N^{-1}).
\end{align}
\end{subequations}

\paragraph{Lower bound:}
For the lower bound, we take a grid $\{ \mu_{k, N} \}_{k = 0}^N$ whose
cost is within $N^{-2}$ of the infimum, and write $\delta_{k, N} =
\mu_{k + 1, N} - \mu_{k, N}$.  The preceding upper bound shows that
the infimum is $O(N^{-1})$, so the near-optimal grid $\{ \mu_{k, N}
\}$ also has total cost $O(N^{-1})$.  Consequently, the quadratic
lower bound~\eqref{EqnQuadLower} gives
\begin{align*}
  c_0 \sum_{k = 0}^{N - 1} \delta_{k, N}^2 & \leq \sum_{k = 0}^{N - 1}
  \gamma(\mu_{k, N}, \mu_{k + 1, N}) = O(N^{-1}),
\end{align*}
and hence $\sum_k \delta_{k, N}^2 = O(N^{-1})$.  Since $\max_k
\delta_{k, N}^2 \leq \sum_k \delta_{k, N}^2 = O(N^{-1})$, the mesh
tends to zero, so the uniform local expansion may be summed over the
grid with total remainder $o\big( \sum_k \delta_{k, N}^2 \big) =
o(N^{-1})$.  Doing so yields
\begin{align*}
  \sum_{k = 0}^{N - 1} \gamma(\mu_{k, N}, \mu_{k + 1, N}) =
  \frac{1}{2} \sum_{k = 0}^{N - 1} \qdens(\mu_{k, N}) \delta_{k, N}^2
  + o(N^{-1}) & \stackrel{(i)}{\geq} \frac{1}{2 N} \left\{ \sum_{k =
    0}^{N - 1} \fdens(\mu_{k, N}) \delta_{k, N} \right\}^2 + o(N^{-1})
  \\
  & \stackrel{(ii)}{=} \frac{\Fstar^2}{2 N} + o(N^{-1}),
\end{align*}
where step (i) follows from the Cauchy--Schwarz inequality and $\qdens
= \fdens^2$; and the last step follows from the Riemann-sum limit
$\sum_k \fdens(\mu_{k, N}) \delta_{k, N} = \Fstar + o(1)$. \\

\noindent Finally, combining this lower bound with the upper
bound~\eqref{EqnFineEulerUpperProof} yields the claim.


\section{Discussion}
\label{SecDiscussion}

In this paper, we have shown that the unmasking growth complexity
(\ugc) controls the performance of random subset unmasking schemes,
and reveals the structure required to optimize their performance.  It
is a path-based measure with additive local increments that bound the
corresponding KL discretization errors.  The reveal odds determine the
natural multiplicative scale for a finite unmasking step, making
log-reveal-odds the appropriate path coordinate. The log-reveal-odds
\ugc density gives a local description of sampling difficulty: regions
carrying little \ugc mass can be traversed rapidly with large
stepsizes, whereas regions of concentrated mass require smaller
stepsizes.  We formalized this resource-allocation principle via a
notion of \emph{partition complexity}.  At one extreme, it reduces to a
(coarse) single-block complexity; under progressive refinement, it
converges to a functional of the square root of the \ugc density,
which also determines the sharp leading-order complexity of optimally
scheduled Euler discretization.

Notably, this paper moves beyond treating the underlying geometry as
an oracle quantity.  We show how the \ugc increments can be reliably
estimated using KL increments along a coupled reveal trajectory.  By
doing so, we obtain sampling methods that are
\emph{certified-optimal}, meaning that they are certified (with high
probability) to achieve a pre-specified KL error, and their iteration
complexity is within a constant factor of the optimal iteration
complexity.  Thus, our analysis connects three core tasks in diffusion
sampling: characterizing the geometry of the target distribution,
designing a sampling schedule, and certifying the accuracy of the
resulting sampler.

Let us make note of some important qualifications to this
conclusion. Our strongest data-dependent guarantees rely on access to
Bayes denoisers and on moment control for their KL increments. With
learned denoisers, approximation error contributes an additional term
to the KL guarantee, and it needs to be estimated for a fully
certified guarantee.  Moreover, our notion of optimality is relative
to frozen-poster Euler discretizations.  It is possible that
higher-order discretizations and other more sophisticated procedures
could admit different local geometries, and hence potentially
different optimal clocks for stepsize allocation.

This paper has focused exclusively on the utility of \ugc complexity
for designing and optimizing unmasking samplers.  However, thinking
beyond the context of sampling, the view of data geometry afforded by
the \ugc density is of independent interest for statistical analysis.
As a notable example, \Cref{FigHierarchicalMixture} shows the \ugc
density $\qdens$ for a discrete mixture model with a hierarchical
structure, where the $\qdens$-modes correspond to reveal times at
which successive levels of the hierarchy are resolved.  Our methods
allow us to reliably estimate increments of the \ugc density, and it
would be interesting to do so for real-world discrete datasets.

In our companion paper~\cite{Wai26} on Gaussian diffusion sampling, we
introduced the denoising growth complexity (\dgc), and provided
analogous results on sampling performance and stepsize optimization,
again culminating in a characterization via the \dgc density.  Despite
the substantial differences between Gaussian diffusion and discrete
masking, the two theories exhibit a striking parallel. The natural
path coordinates differ (log-heat-time versus log-reveal-odds), but
both the \ugc and \dgc densities capture a form of information
curvature along the noising path, and their shape dictates how
computational effort should be allocated.

\subsubsection*{Acknowledgements}
This work was partially supported by a Guggenheim Fellowship, an NSF
grant (DMS-2311072), and the Ford Professorship at MIT.  We thank
Yuting Wei for her inspiring talk during the MIT Statistics and Data
Science conference in spring 2026.

\bibliographystyle{alpha_initials}
{\small{
\bibliography{final_masked_refs}

@book{AloSpe16,
  author = {Noga Alon and Joel H. Spencer},
  edition = {4},
  isbn = {978-1-119-06195-3},
  publisher = {John Wiley \& Sons},
  title = {The Probabilistic Method},
  year = {2016},
}

@article{CheEtAl25,
  archiveprefix = {arXiv},
  arxivid = {2511.04647},
  author = {S. Chen and K. Cong and J. Li},
  eprint = {2511.04647},
  journal = {arXiv preprint arXiv:2511.04647},
  title = {Optimal Inference Schedules for Masked Diffusion Models},
  url = {https://arxiv.org/abs/2511.04647},
  year = {2025},
}

@article{DmiEtAl26,
  archiveprefix = {arXiv},
  arxivid = {2602.15008},
  author = {D. Dmitriev and Z. Huang and Y. Wei},
  eprint = {2602.15008},
  journal = {arXiv preprint arXiv:2602.15008},
  title = {Efficient Sampling with Discrete Diffusion Models: Sharp and Adaptive Guarantees},
  url = {https://arxiv.org/abs/2602.15008},
  year = {2026},
}

@inproceedings{LiaEtAl25a,
  archiveprefix = {arXiv},
  arxivid = {2506.02318},
  author = {Yuchen Liang and Renxiang Huang and Lifeng Lai and Ness Shroff and Yingbin Liang},
  booktitle = {Advances in Neural Information Processing Systems},
  eprint = {2506.02318},
  title = {Absorb and Converge: Provable Convergence Guarantee for Absorbing Discrete Diffusion Models},
  url = {https://proceedings.neurips.cc/paper_files/paper/2025/hash/1d571ea833394630c1dec71664f16cd9-Abstract-Conference.html},
  volume = {38},
  year = {2025},
}

@inproceedings{LiaEtAl25b,
  archiveprefix = {arXiv},
  arxivid = {2509.16756},
  author = {Yuchen Liang and Yingbin Liang and Lifeng Lai and Ness Shroff},
  booktitle = {Advances in Neural Information Processing Systems},
  eprint = {2509.16756},
  title = {Discrete Diffusion Models: Novel Analysis and New Sampler Guarantees},
  url = {https://proceedings.neurips.cc/paper_files/paper/2025/hash/f1e7c90552850afcc2558d78950c519d-Abstract-Conference.html},
  volume = {38},
  year = {2025},
}

@inproceedings{RenEtAl25a,
  archiveprefix = {arXiv},
  arxivid = {2502.00234},
  author = {Yinuo Ren and Haoxuan Chen and Yuchen Zhu and Wei Guo and Yongxin Chen and Grant M. Rotskoff and Molei Tao and Lexing Ying},
  booktitle = {Advances in Neural Information Processing Systems},
  eprint = {2502.00234},
  title = {Fast Solvers for Discrete Diffusion Models: Theory and Applications of High-Order Algorithms},
  url = {https://proceedings.neurips.cc/paper_files/paper/2025/hash/f46ddea413df86832418c5e04e59644f-Abstract-Conference.html},
  volume = {38},
  year = {2025},
}

@inproceedings{RenEtAl25b,
  archiveprefix = {arXiv},
  arxivid = {2410.03601},
  author = {Yinuo Ren and Haoxuan Chen and Grant M. Rotskoff and Lexing Ying},
  booktitle = {International Conference on Learning Representations},
  eprint = {2410.03601},
  title = {How Discrete and Continuous Diffusion Meet: Comprehensive Analysis of Discrete Diffusion Models via a Stochastic Integral Framework},
  url = {https://openreview.net/forum?id=6awxwQEI82},
  year = {2025},
}

@article{Gil01,
  author = {Daniel T. Gillespie},
  doi = {10.1063/1.1378322},
  journal = {The Journal of Chemical Physics},
  number = {4},
  pages = {1716--1733},
  title = {Approximate Accelerated Stochastic Simulation of Chemically Reacting Systems},
  volume = {115},
  year = {2001},
}

@article{RatEtAl05,
  author = {Muruhan Rathinam and Linda R. Petzold and Yang Cao and Daniel T. Gillespie},
  doi = {10.1137/040603206},
  journal = {Multiscale Modeling \& Simulation},
  number = {3},
  pages = {867--895},
  title = {Consistency and Stability of Tau-Leaping Schemes for Chemical Reaction Systems},
  volume = {4},
  year = {2005},
}

@article{Li07,
  author = {Tiejun Li},
  doi = {10.1137/06066792X},
  journal = {Multiscale Modeling \& Simulation},
  number = {2},
  pages = {417--436},
  title = {Analysis of Explicit Tau-Leaping Schemes for Simulating Chemically Reacting Systems},
  volume = {6},
  year = {2007},
}

@article{AndEtAl11,
  author = {David F. Anderson and Arnab Ganguly and Thomas G. Kurtz},
  doi = {10.1214/10-AAP756},
  journal = {The Annals of Applied Probability},
  number = {6},
  pages = {2226--2262},
  title = {Error Analysis of Tau-Leap Simulation Methods},
  volume = {21},
  year = {2011},
}

@article{Han75,
  author = {Te Sun Han},
  doi = {10.1016/S0019-9958(75)80004-0},
  journal = {Information and Control},
  number = {4},
  pages = {337--368},
  title = {Linear Dependence Structure of the Entropy Space},
  volume = {29},
  year = {1975},
}

@article{Han78,
  author = {Te Sun Han},
  doi = {10.1016/S0019-9958(78)90275-9},
  journal = {Information and Control},
  number = {2},
  pages = {133--156},
  title = {Nonnegative Entropy Measures of Multivariate Symmetric Correlations},
  volume = {36},
  year = {1978},
}

@article{Wat60,
  author = {Satosi Watanabe},
  doi = {10.1147/rd.41.0066},
  journal = {IBM Journal of Research and Development},
  number = {1},
  pages = {66--82},
  title = {Information Theoretical Analysis of Multivariate Correlation},
  volume = {4},
  year = {1960},
}

@inproceedings{MauPont09,
  archiveprefix = {arXiv},
  author = {Andreas Maurer and Massimiliano Pontil},
  booktitle = {Proceedings of the 22nd Conference on Learning Theory},
  eprint = {0907.3740},
  primaryclass = {stat.ML},
  title = {Empirical {B}ernstein Bounds and Sample Variance Penalization},
  year = {2009},
}

@techreport{Wai26,
  archiveprefix = {arXiv},
  arxivid = {2607.26285},
  author = {Martin J. Wainwright},
  eprint = {2607.26285},
  journal = {arXiv preprint arXiv:2607.26285},
  number = {arXiv:2607.26285},
  title = {Denoising Growth Complexity: {D}ata Geometry and Certified Schedules for Diffusion Sampling},
  url = {https://arxiv.org/abs/2607.26285},
  year = {2026},
  institution = { {M}assachusetts {I}nstitute of {T}echnology}
}

@article{BuzZam12a,
  author = {J{\'e}r{\^o}me Buzzi and Lorenzo Zambotti},
  doi = {10.1007/s00440-011-0350-y},
  journal = {Probability Theory and Related Fields},
  number = {3--4},
  pages = {421--440},
  title = {Approximate Maximizers of Intricacy Functionals},
  volume = {153},
  year = {2012},
}

@article{BuzZam12b,
  author = {J{\'e}r{\^o}me Buzzi and Lorenzo Zambotti},
  doi = {10.1214/11-AIHP416},
  journal = {Annales de l'Institut Henri Poincar{\'e}, Probabilit{\'e}s et Statistiques},
  number = {2},
  pages = {343--367},
  title = {Mean Mutual Information and Symmetry Breaking for Finite Random Fields},
  volume = {48},
  year = {2012},
}

@article{TonSpoEde94,
  author = {Giulio Tononi and Olaf Sporns and Gerald M. Edelman},
  doi = {10.1073/pnas.91.11.5033},
  journal = {Proceedings of the National Academy of Sciences of the United States of America},
  number = {11},
  pages = {5033--5037},
  title = {A Measure for Brain Complexity: Relating Functional Segregation and Integration in the Nervous System},
  volume = {91},
  year = {1994},
}

@article{OlbEtAl08,
  author = {Eckehard Olbrich and Nils Bertschinger and Nihat Ay and J{\"u}rgen Jost},
  doi = {10.1140/epjb/e2008-00134-9},
  journal = {The European Physical Journal B},
  number = {3},
  pages = {407--415},
  title = {How Should Complexity Scale with System Size?},
  volume = {63},
  year = {2008},
}

@article{BarBucBul09,
  author = {Lionel Barnett and Christopher L. Buckley and Seth Bullock},
  doi = {10.1103/PhysRevE.79.051914},
  journal = {Physical Review E},
  number = {5},
  pages = {051914},
  title = {Neural Complexity and Structural Connectivity},
  volume = {79},
  year = {2009},
}

@article{RosEtAl19,
  author = {Fernando E. Rosas and Pedro A. M. Mediano and Michael Gastpar and Henrik J. Jensen},
  doi = {10.1103/PhysRevE.100.032305},
  journal = {Physical Review E},
  number = {3},
  pages = {032305},
  title = {Quantifying High-Order Interdependencies via Multivariate Extensions of the Mutual Information},
  volume = {100},
  year = {2019},
}

@article{VarEtAl23,
  author = {Thomas F. Varley and Maria Pope and Joshua Faskowitz and Olaf Sporns},
  doi = {10.1038/s42003-023-04843-w},
  journal = {Communications Biology},
  pages = {451},
  title = {Multivariate Information Theory Uncovers Synergistic Subsystems of the Human Cerebral Cortex},
  volume = {6},
  year = {2023},
}

@book{Lorentz86,
  address = {New York},
  author = {G. G. Lorentz},
  edition = {Second},
  publisher = {Chelsea Publishing Company},
  title = {{B}ernstein {P}olynomials},
  year = {1986},
}

@book{Phillips03,
  address = {New York},
  author = {G. M. Phillips},
  publisher = {Springer-Verlag},
  title = {{I}nterpolation and {A}pproximation by {P}olynomials},
  year = {2003},
}

@inproceedings{GhazvininejadEtAl2019MaskPredict,
  author = {Marjan Ghazvininejad and Omer Levy and Yinhan Liu and Luke Zettlemoyer},
  booktitle = {Proceedings of the 2019 Conference on Empirical Methods in Natural Language Processing and the 9th International Joint Conference on Natural Language Processing ({EMNLP-IJCNLP})},
  doi = {10.18653/v1/D19-1633},
  pages = {6112--6121},
  publisher = {Association for Computational Linguistics},
  title = {{Mask-Predict}: Parallel Decoding of Conditional Masked Language Models},
  year = {2019},
}

@inproceedings{ChangEtAl2022MaskGIT,
  author = {Huiwen Chang and Han Zhang and Lu Jiang and Ce Liu and William T. Freeman},
  booktitle = {Proceedings of the {IEEE/CVF} Conference on Computer Vision and Pattern Recognition ({CVPR})},
  pages = {11315--11325},
  title = {{MaskGIT}: Masked Generative Image Transformer},
  year = {2022},
}

@inproceedings{ChangEtAl2023Muse,
  author = {Huiwen Chang and Han Zhang and Jarred Barber and Aaron Maschinot and Jose Lezama and Lu Jiang and Ming-Hsuan Yang and Kevin Patrick Murphy and William T. Freeman and Michael Rubinstein and Yuanzhen Li and Dilip Krishnan},
  booktitle = {Proceedings of the 40th International Conference on Machine Learning},
  pages = {4055--4075},
  publisher = {PMLR},
  series = {Proceedings of Machine Learning Research},
  title = {{Muse}: Text-To-Image Generation via Masked Generative Transformers},
  volume = {202},
  year = {2023},
}

@inproceedings{YuEtAl2023MAGVIT,
  author = {Lijun Yu and Yong Cheng and Kihyuk Sohn and Jose Lezama and Han Zhang and Huiwen Chang and Alexander G. Hauptmann and Ming-Hsuan Yang and Yuan Hao and Irfan Essa and Lu Jiang},
  booktitle = {Proceedings of the {IEEE/CVF} Conference on Computer Vision and Pattern Recognition ({CVPR})},
  pages = {10459--10469},
  title = {{MAGVIT}: Masked Generative Video Transformer},
  year = {2023},
}

@inproceedings{SahooEtAl2024MDLM,
  author = {Subham Sekhar Sahoo and Marianne Arriola and Yair Schiff and Aaron Gokaslan and Edgar Marroquin and Justin T. Chiu and Alexander Rush and Volodymyr Kuleshov},
  booktitle = {Advances in Neural Information Processing Systems},
  title = {Simple and Effective Masked Diffusion Language Models},
  volume = {37},
  year = {2024},
}

@inproceedings{WangEtAl2024DPLM,
  author = {Xinyou Wang and Zaixiang Zheng and Fei Ye and Dongyu Xue and Shujian Huang and Quanquan Gu},
  booktitle = {Proceedings of the 41st International Conference on Machine Learning},
  pages = {52309--52333},
  publisher = {PMLR},
  series = {Proceedings of Machine Learning Research},
  title = {Diffusion Language Models Are Versatile Protein Learners},
  volume = {235},
  year = {2024},
}

@inproceedings{WangEtAl2025MaskGCT,
  author = {Yuancheng Wang and Haoyue Zhan and Liwei Liu and Ruihong Zeng and Haotian Guo and Jiachen Zheng and Qiang Zhang and Xueyao Zhang and Shunsi Zhang and Zhizheng Wu},
  booktitle = {International Conference on Learning Representations},
  title = {{MaskGCT}: Zero-Shot Text-to-Speech with Masked Generative Codec Transformer},
  year = {2025},
}

@inproceedings{NieEtAl2025LLaDA,
  author = {Shen Nie and Fengqi Zhu and Zebin You and Xiaolu Zhang and Jingyang Ou and Jun Hu and Jun Zhou and Yankai Lin and Ji-Rong Wen and Chongxuan Li},
  booktitle = {Advances in Neural Information Processing Systems},
  title = {Large Language Diffusion Models},
  volume = {38},
  year = {2025},
}

@inproceedings{BenEtAl24,
  author = {Benton, J. and De Bortoli, V. and Doucet, A. and Deligiannidis, G.},
  booktitle = {International Conference on Learning Representations (ICLR) 2024},
  title = {Nearly $d$-Linear Convergence Bounds for Diffusion Models via Stochastic Localization},
  url = {https://openreview.net/forum?id=r5njV3BsuD},
  year = {2024},
}

@book{BroEtAl11,
  author = {S. Brooks and A. Gelman and G. L. Jones and X. L. Meng},
  publisher = {CRC Press},
  title = {Handbook of {M}arkov Chain {M}onte {C}arlo},
  year = {2011},
}

@inproceedings{CheEtAl23a,
  arxivid = {2211.01916},
  author = {H. Chen and H. Lee and J. Lu},
  booktitle = {International Conference on Machine Learning},
  eprint = {2211.01916},
  title = {{Improved Analysis of Score-based Generative Modeling: User-Friendly Bounds under Minimal Smoothness Assumptions}},
  year = {2023},
}

@inproceedings{CheEtAl23c,
  archiveprefix = {arXiv},
  arxivid = {2209.11215},
  author = {S. Chen and S. Chewi and J. Li and Y. Li and A. Salim and A. R. Zhang},
  booktitle = {International Conference on Learning Representations},
  eprint = {2209.11215},
  title = {{Sampling is as easy as learning the score: theory for diffusion models with minimal data assumptions}},
  year = {2023},
}

@article{CheEtAl24,
  author = {Chen, M. and Mei, S. and Fan, J. and Wang, M.},
  journal = {arXiv preprint arXiv:2404.07771},
  title = {An Overview of Diffusion Models: Applications, Guided Generation, Statistical Rates and Optimization},
  url = {https://arxiv.org/abs/2404.07771},
  year = {2024},
}

@article{CroEtAl23,
  author = {F.-A. Croitoru and V. Hondru and R. T. Ionescu and M. Shah},
  doi = {10.1109/TPAMI.2023.3261988},
  journal = {IEEE Transactions on Pattern Analysis and Machine Intelligence},
  number = {9},
  pages = {10850--10869},
  title = {Diffusion Models in Vision: A Survey},
  volume = {45},
  year = {2023},
}

@article{GelEtAl13,
  author = {A. Gelman and J. Carlin and H. S. Stern and D. B. Dunson and A. Vehtari and D. K. Salomon},
  journal = {CRC Press},
  title = {{B}ayesian Data Analysis},
  year = {2013},
}

@article{HoEtAl20,
  author = {J. Ho and A. Jain and P. Abbeel},
  journal = {Advances in Neural Information Processing Systems (NeurIPS)},
  pages = {6840--6851},
  title = {Denoising diffusion probabilistic models},
  volume = {33},
  year = {2020},
}

@book{RobCase04,
  author    = {Robert, Christian P. and Casella, George},
  title     = {Monte Carlo Statistical Methods},
  edition   = {2},
  series    = {Springer Texts in Statistics},
  publisher = {Springer},
  address   = {New York},
  year      = {2004},
  doi       = {10.1007/978-1-4757-4145-2},
  isbn      = {978-0-387-21239-5}
}

@article{RomEtAl22,
  author = {R. Rombach and E. Blattmann and S. L. Dhariwal and A. M. D. M. L. and P. E. S.},
  journal = {Proceedings of the IEEE/CVF Conference on Computer Vision and Pattern Recognition (CVPR)},
  pages = {10684--10694},
  title = {High-Resolution Image Synthesis with Latent Diffusion Models},
  year = {2022},
}

@book{RubKroe08,
  address = {Hoboken, NJ},
  author = {R. Y. Rubinstein and D. P. Kroese},
  edition = {2nd},
  publisher = {John Wiley and Sons},
  title = {Simulation and the Monte Carlo Method},
  year = {2008},
}

@inproceedings{SohEtAl15,
  author = {J. Sohl‑Dickstein and E. Weiss and N. Maheswaranathan and S. Ganguli},
  booktitle = {Proceedings of the 32nd International Conference on Machine Learning},
  pages = {2256--2265},
  publisher = {PMLR},
  series = {Proceedings of Machine Learning Research},
  title = {Deep Unsupervised Learning using Nonequilibrium Thermodynamics},
  url = {https://proceedings.mlr.press/v37/sohl-dickstein15.html},
  volume = {37},
  year = {2015},
}

@inproceedings{SonErmo19,
  author = {Y. Song and S. Ermon},
  booktitle = {Advances in Neural Information Processing Systems (NeurIPS)},
  pages = {11895--11907},
  title = {Generative Modeling by Estimating Gradients of the Data Distribution},
  year = {2019},
}

@inproceedings{SonEtAl21,
  author = {Y. Song and J. Sohl-Dickstein and D. P. Kingma and A. Kumar and S. Ermon and B. Poole},
  booktitle = {International Conference on Learning Representations},
  title = {Score-Based Generative Modeling through Stochastic Differential Equations},
  year = {2021},
}

@article{YanEtAl25,
  author = {L. Yang and Z. Zhang and Y. Song and et al.},
  journal = {ACM Computing Surveys (to appear)},
  note = {arXiv:2209.00796},
  title = {Diffusion Models: A Comprehensive Survey of Methods and Applications},
  year = {2025},
}

@inproceedings{HooEtAl21,
  author = {Emiel Hoogeboom and Didrik Nielsen and Priyank Jaini and Patrick Forr{\'e} and Max Welling},
  booktitle = {Advances in Neural Information Processing Systems},
  title = {Argmax Flows and Multinomial Diffusion: Learning Categorical Distributions},
  volume = {34},
  year = {2021},
}

@inproceedings{AusEtAl21,
  author = {Jacob Austin and Daniel D. Johnson and Jonathan Ho and Daniel Tarlow and Rianne van den Berg},
  booktitle = {Advances in Neural Information Processing Systems},
  title = {Structured Denoising Diffusion Models in Discrete State-Spaces},
  volume = {34},
  year = {2021},
}

@inproceedings{CamEtAl22,
  author = {Andrew Campbell and Joe Benton and Valentin De Bortoli and Thomas Rainforth and George Deligiannidis and Arnaud Doucet},
  booktitle = {Advances in Neural Information Processing Systems},
  doi = {10.52202/068431-2049},
  title = {A Continuous Time Framework for Discrete Denoising Models},
  volume = {35},
  year = {2022},
}

@inproceedings{LouEtAl24,
  author = {Aaron Lou and Chenlin Meng and Stefano Ermon},
  booktitle = {Proceedings of the 41st International Conference on Machine Learning},
  pages = {32819--32848},
  publisher = {PMLR},
  series = {Proceedings of Machine Learning Research},
  title = {Discrete Diffusion Modeling by Estimating the Ratios of the Data Distribution},
  url = {https://proceedings.mlr.press/v235/lou24a.html},
  volume = {235},
  year = {2024},
}

@inproceedings{ShiEtAl24b,
  author = {Jiaxin Shi and Kehang Han and Zhe Wang and Arnaud Doucet and Michalis K. Titsias},
  booktitle = {Advances in Neural Information Processing Systems},
  title = {Simplified and Generalized Masked Diffusion for Discrete Data},
  volume = {37},
  year = {2024},
}

@inproceedings{LiCai25,
  archiveprefix = {arXiv},
  arxivid = {2505.21400},
  author = {Gen Li and Changxiao Cai},
  booktitle = {Advances in Neural Information Processing Systems},
  eprint = {2505.21400},
  title = {Breaking {AR}'s Sampling Bottleneck: Provable Acceleration via Diffusion Language Models},
  url = {https://proceedings.neurips.cc/paper_files/paper/2025/hash/111298628bfb153ee1f84b10fec3a8b9-Abstract-Conference.html},
  volume = {38},
  year = {2025},
}

@article{LavZan25,
  archiveprefix = {arXiv},
  arxivid = {2510.25544},
  author = {Hugo Lavenant and Giacomo Zanella},
  eprint = {2510.25544},
  journal = {arXiv preprint arXiv:2510.25544},
  title = {Error Bounds and Optimal Schedules for Masked Diffusions with Factorized Approximations},
  url = {https://arxiv.org/abs/2510.25544},
  year = {2025},
}
}}

\appendix

\section{Consequences of single-block theory}
\label{AppComplements}

In this appendix, we discuss some consequences of~\Cref{CorSingle},
which gives single-block stepsize schedules with iteration complexity
depending on the aggregate \ugc masses $\Hinfo(0,1)$ and
$\Hcard(0,d)$, for the Bernoulli and fixed-cardinality samplers,
respectively.  For this reason, it is interesting to relate the
aggregate \ugc mass to other complexity measures, and thereby to
relate these complexity measures to the sampling guarantees in the
papers~\cite{CheEtAl25,DmiEtAl26}.

We begin by observing that the aggregate \ugc mass $\Hinfo(0,1)$ turns
out to be \emph{equivalent} to a measure of multivariate dependence
first introduced by Tononi, Sporns, and Edelman~\cite{TonSpoEde94} for
modeling brain function.  This complexity measure, hereafter referred
to as $\TSE$, has since been studied by various
authors~\cite{OlbEtAl08,BarBucBul09,RosEtAl19,VarEtAl23}, and also
extended to a more general notion of
intricacy~\cite{BuzZam12a,BuzZam12b}.  The $\TSE$ complexity is given
by
\begin{align}
\label{EqnDefnTSE}  
  \TSE(\Prob_Z) & \defn \sum_{k = 1}^{\usedim} \left( e_k -
  \frac{k}{\usedim} e_\usedim \right) \qquad \mbox{where $e_k \defn
    \Exs\big[ \Ent(\Zvar_{A_k}) \big]$,}
\end{align}
with $A_k$ being a uniformly random subset of $[\usedim]$ with $k$
elements.  We establish its equivalence to the aggregate \ugc
mass as part of~\Cref{PropSimpleRelation} below.

Turning to sampling connections, the paper~\cite{CheEtAl25} analyzed
the fixed-cardinality sampler, and gave guarantees in terms of the
minimum of two classical measures of multivariate
dependence~\cite{Wat60,Han75,Han78}, known as the \emph{total
correlation} and \emph{dual total correlation}, defined by
\begin{subequations}
\begin{align}
\label{EqnDefnTC}
  \TC(\Prob_Z) \defn \sum_{i = 1}^{\usedim} \Ent(\Zvar_i) -
  \Ent(\Zvar) \quad \mbox{and} \quad \DTC(\Prob_Z) \defn \Ent(\Zvar) -
  \sum_{i = 1}^{\usedim} \Ent(\Zvar_i \mid \Zvar_{-i}).
\end{align}
Here $\Ent$ denotes the Shannon entropy, whereas $\Zvar_{-i}$ denotes
the $(\usedim-1)$-dimensional subvector of all coordinates except the
$i^{th}$.

The single-block sampling guarantee in terms of $\Hcard(0,d)$, or
equivalently $\Hinfo(0,1)$, sharpens these guarantees for the
fixed-cardinality sampler.  As we show in~\Cref{SecBernRep}, in
particular via a Bernstein polynomial representation, the total and
dual total correlation measures have the simple representations
\begin{align}
\label{EqnTCRep}
\TC(\Prob_Z) = \int_0^1 (1 - t) \hfun'(t) dt \quad \mbox{and} \quad
\DTC(\Prob_Z) = \int_0^1 t \hfun'(t) dt.
\end{align}
Since $\Hinfo(0, 1) = \int_0^1 t \, (1 - t) \hfun'(t) dt$, we see
immediately that
\begin{align}
\label{EqnMinBound}  
\Hinfo(0,1) & \leq \min \{ \TC(\ProbZ), \DTC(\ProbZ) \}.
\end{align}
\end{subequations}

In their work on continuous-time Markov chain (CTMC) sampling,
Dmitriev et al.~\cite{DmiEtAl26} gave a CTMC unmasking sampler with
iteration complexity controlled by a functional $\DHW(\ProbZ)$, which
they refer to as effective total correlation, that is also upper
bounded by the minimum of the $\TC$ and $\DTC$ complexities.  We also
show as part of~\Cref{PropSimpleRelation} that this functional is
lower bounded by $\Hinfo(0,1)$, and can be upper bounded by a constant
multiple of it.  Thus, both the Bernoulli and fixed-cardinality
samplers, in their single-block instantiation, inherit the guarantees
of the DHW analysis.  In particular, their paper studies various
interesting examples, including stochastic block models and quantized
versions of low-dimensional structure, for which $\DHW(\ProbZ)$, and
hence $\Hinfo(0,1)$ and $\Hcard(0,d)$, are relatively small.  Notably,
they construct an instance with $\min \{ \TC(\ProbZ), \DTC(\ProbZ) \}
= \Theta(\usedim)$, whereas their complexity has constant scaling.
Our analysis shows that the guarantees on CTMC-based
unmasking~\cite{DmiEtAl26} have analogues for both fixed-cardinality
unmasking~\cite{CheEtAl25,LavZan25} and Bernoulli unmasking. \\

\noindent With this context, we now formally state the complexity
relations:
\mygraybox{
\begin{proposition}[Connections between aggregate complexity measures]
\label{PropSimpleRelation}
For the full interval $[0,1]$, the Bernoulli aggregate-\ugc complexity
is related to \TSE{} complexity via the relation
\begin{subequations}
\begin{align}
\label{EqnUGCtoTSE}
\Hinfo(0, 1) & = \frac{2}{\usedim + 1} \TSE(\Prob_Z).
\end{align}
Moreover, the fixed-cardinality complexity $\Hcard(0,d)$ is sandwiched
by the Bernoulli version
\begin{align}
\label{EqnCard2Ber}
\Hinfo(0, 1) \; \leq \; \Hcard(0, d) \; \leq \; 2 \Hinfo(0, 1).
\end{align}
Finally, the \DHW{} complexity satisfies the sandwich relation
\begin{align}
\label{EqnUGC2DHW}
\Hinfo(0, 1) \; \stackrel{(iii)}{\leq} \; \DHW(\Prob_Z) \;
\stackrel{(iv)}{\leq} \; \frac{e}{e - 1} \Hinfo(0, 1).
\end{align}
\end{subequations}
\end{proposition}
}
\noindent We prove these claims in~\Cref{SecProofPropSimpleRelation};
let us highlight a few interesting features here. The proofs of both
claims~\eqref{EqnUGCtoTSE} and~\eqref{EqnCard2Ber} make use of a
Bernstein polynomial representation of $\hfun$, a result of
independent interest. In particular, we show in~\Cref{SecBernRep} that
\begin{align}
\label{EqnHfunBernsteinMain}
\hfun(\rtime) & = \sum_{j = 0}^{\usedim - 1} \binom{\usedim - 1}{j}
\rtime^j (1 - \rtime)^{\usedim - 1 - j} \hcard_j,
\end{align}
so that $\hfun$ is a polynomial of degree at most $\usedim - 1$.

In proving the sandwich~\eqref{EqnUGC2DHW}, we show that the $\DHW$
complexity functional has the explicit representation $\DHW(\ProbZ) =
\int_0^1 t \min \{1, \log(1/t) \} \hfun'(t) dt$, from which
the estimates~\eqref{EqnUGC2DHW} follow from simple algebra.  The
analysis leading to this representation reveals an interesting
algorithmic connection to Bernoulli unmasking: the DHW unmasking
algorithm, while described and analyzed as a form of
$\tau$-leaping~\cite{Gil01} in a continuous-time Markov chain, is
closely related to the Bernoulli unmasking
sampler. See~\Cref{SecDHWRelation} for details of this relation.


\section{Alternative representations of $\hfun$}
\label{SecHalt}

In this section, we develop and explore some alternative
representations of $\hfun$ and its derivative $\hfun'$ that provide
useful insight for subsequent proofs.

\subsection{Proof of the identity~\eqref{EqnHfunRandSub}}

In this appendix, we prove the random subset
identity~\eqref{EqnHfunRandSub}, namely that we can write
\begin{align*}
  \hfun(\rtime) = \sum_{\ind = 1}^{\usedim} \Exs_{A_{\ind,\rtime}}
  \big[ \Info(\Zvar_\ind; \Zvar_{A_{\ind,\rtime}}) \big],
\end{align*}
where $A_{\ind,\rtime} \subseteq [\usedim] \setminus \{\ind \}$ is a
random subset obtained by including each coordinate of $[\usedim]
\setminus \{ \ind \}$ independently with probability $\rtime$.

To prove this representation, for each $\ind \in [\usedim]$, let
$E_\ind$ denote the event that coordinate $\ind$ remains masked at
reveal time $\rtime$, and let $A \equiv A_{\ind,\rtime} \subseteq
[\usedim] \setminus \{ \ind \}$ denote the set of revealed
coordinates.  Conditional on $E_\ind$, the observation $X_\rtime$ is
equivalent to the pair $(A, \Zvar_{A})$.  Consequently, we have the
mutual information identity $\Info(\Zvar_\ind; X_\rtime \mid E_\ind) =
\Info\big( \Zvar_\ind; A, \Zvar_{A} \big)$.  By the chain rule for
mutual information, we have
\begin{align*}
  \Info\big( \Zvar_\ind; A, \Zvar_{A} \big) & =
  \underbrace{\Info(\Zvar_\ind; A)}_{=0} + \Info\big( \Zvar_\ind;
  \Zvar_{A} \mid A \big),
\end{align*}
using the fact that $\Zvar_\ind$ and $A \equiv A_{\ind, \rtime}$ are
independent.  Conditioning on $A = a$ does not change the joint law of
$(\Zvar_\ind, \Zvar_A)$, so that we have shown that
\begin{align*}
 \Info(\Zvar_\ind; X_\rtime \mid E_\ind) \; = \; \Info(\Zvar_i;
  \Zvar_A \mid A) & = \sum_a \Prob(A = a) \Info(\Zvar_i; \Zvar_a) =
  \Exs_A\big[ \Info(\Zvar_i; \Zvar_A) \big] \qquad \mbox{for each
    $\ind =1, \ldots, \usedim$.}
\end{align*}
Summing this identity over the coordinate index $\ind$ yields the
claim~\eqref{EqnHfunRandSub}.


\subsection{Representation via Bernstein polynomials}
\label{SecBernRep}
In this section, we develop representations of both $\hfun$ and
$\hfun'$ as Bernstein polynomials in terms of the coefficients
$\hcard_j$. These representations have some useful immediate
consequences, and are used in the proofs in~\Cref{SecCompare}.  Recall
the definition of the coefficients
\begin{align}
\label{EqnDefnHcardGainTwo}
\hcard_j & \defn \sum_{\ind = 1}^{\usedim} \Exs_{B_{i,j}}
\big[\Info(\Zvar_\ind;  \Zvar_{B_{i,j}}) \big] \qquad \mbox{where
    $B_{i,j}$ is uniform over $j$-cardinality subsets of $[\usedim]
    \setminus \{\ind\}$.}
\end{align}
Letting $A_{i,t}$ be a Bernoulli random subset of $[\usedim] \setminus
\{\ind\}$ with inclusion probability $t$, we can write
\begin{align*}
\hcard_j & = \sum_{i=1}^d \Exs\Big[ \Info(\Zvar_\ind; \Zvar_{A_{i,t}})
  \mid \card(A_{i,t}) = j \Big].
\end{align*}
Since $\Prob[\card(A_{i,t}) = j] = \binom{\usedim - 1}{j} \rtime^j (1 -
\rtime)^{\usedim - 1 - j}$, it follows by combining the tower property
with the representation~\eqref{EqnHfunRandSub} that
\begin{subequations}
  \begin{align}
    \label{EqnHfunBernstein}
  \hfun(\rtime) & = \sum_{j = 0}^{\usedim - 1} \binom{\usedim - 1}{j}
  \rtime^j (1 - \rtime)^{\usedim - 1 - j} \hcard_j.
  \end{align}
Moreover, by differentiating equation~\eqref{EqnHfunBernstein}, we
obtain
  \begin{align}    
  \label{EqnHfunPrimeBernstein}
  \hfun'(\rtime) & = (\usedim - 1) \sum_{j = 0}^{\usedim - 2}
  \binom{\usedim - 2}{j} \rtime^j (1 - \rtime)^{\usedim - 2 - j}
  \big\{ \hcard_{j + 1} - \hcard_j \big\}.
  \end{align}
\end{subequations}
This calculation uses standard properties of
derivatives of Bernstein polynomials~\cite{Lorentz86,Phillips03}.  (In
particular, differentiation transforms coefficients into their first
differences.)

As one immediate consequence, from equation~\eqref{EqnHfunBernstein},
we see that $\hfun$ is a polynomial of degree at most $\usedim -1$.
It follows immediately that the derivative $\hfun'$ exists.  Let us
develop two additional consequences.


\subsubsection{Proof of $\hfun$-monotonicity}
\label{SecHfunMonotone}

We now prove that $\hfun$ is monotone by
showing that $\hfun'(\rtime) \geq 0$. The Bernstein polynomial
representation~\eqref{EqnHfunPrimeBernstein} of $\hfun'$ shows that it
suffices to prove that $\hcard_{j + 1} \geq \hcard_j$.

For $0 \leq j \leq \usedim - 2$ and a fixed coordinate $i$, let $A_j$
be a random subset, chosen uniformly from all the $j$-element subsets
of $[\usedim] \setminus \{ i \}$, and, conditionally on $A_j$, let $J$
be uniform over its complement inside $[\usedim] \setminus
\{i\}$. This defines a coupling between the subsets. Then the random
set $A_j \cup \{ J \}$ is uniform over the $(j + 1)$-element subsets,
while the chain rule for mutual information gives the decomposition
$\Info(\Zvar_i; \Zvar_{A_j \cup \{ J \}}) - \Info(\Zvar_i;
\Zvar_{A_j}) = \Info(\Zvar_i; \Zvar_J \mid \Zvar_{A_j}) \geq 0$.
Averaging first over the coupling and then over $i$ shows that
$\hcard_{j + 1} - \hcard_j \geq 0$, as claimed.


\subsubsection{Proof of the total-correlation representations~\eqref{EqnTCRep}}
\label{SecProofEqnTCRep}

Now let us prove the representations of the total correlation $\TC$
and dual total correlation $\DTC$ given in equation~\eqref{EqnTCRep}.
From the random subset representation~\eqref{EqnHfunRandSub} of
$\hfun$, we have
\begin{align*}
  \hfun(1) & = \sum_{i = 1}^{\usedim} \Info(\Zvar_i; \Zvar_{-i}) \; =
  \; \sum_{i=1}^\usedim \big \{ \Ent(Z_i) - \Ent(\Zvar_i \mid
  \Zvar_{-i}) \big \},
\end{align*}
where the second equality expands the definition of mutual
information.  By comparing with the definitions~\eqref{EqnDefnTC}, we
see immediately that $\hfun(1) = \TC(\ProbZ) + \DTC(\ProbZ)$.  Since
$\hfun(1) = \int_0^1 \hfun'(t) dt$, using the fact that $\hfun(0) =
0$, if we can prove the identity $\TC(\ProbZ) = \int_0^1 (1 - t)
\hfun'(t) dt$, then it follows that $\DTC(\ProbZ) = \int_0^1 t \hfun'(t)
dt$.  Accordingly, the remainder of our analysis focuses on the $\TC$
representation.

For $j = 1, \ldots, \usedim$, let $e_j$ be the entropy coefficient
defined in equation~\eqref{EqnDefnTSE}, along with $e_0 = 0$.  Note
that we have $e_1 = \frac{1}{\usedim} \sum_{i = 1}^{\usedim}
\Ent(\Zvar_i)$, and $e_\usedim = \Ent(\Zvar)$.  For any $i \in
    [\usedim]$ and $A \subseteq [\usedim] \setminus \{ i \}$, we have
    the mutual information identity $\Info(\Zvar_i; \Zvar_A) =
    \Ent(\Zvar_i) + \Ent(\Zvar_A) - \Ent(\Zvar_{A \cup \{ i \}})$.  It
    can be averaged over $i$ and over uniform $j$-cardinality subsets
    $A$.  Doing so and using the definition~\eqref{EqnDefnHcardGain}
    of the $\hcard_j$ coefficients, we find that
\begin{align}
  \label{EqnHardE}
  \hcard_j = \usedim \big( e_1 + e_j - e_{j + 1} \big).
\end{align}
Substituting this identity into the Bernstein
representation~\eqref{EqnHfunBernstein} and integrating over the
interval $[0,1]$ yields
\begin{align*}
  \int_0^1 \hfun(t) \, dt & = \sum_{j = 0}^{\usedim - 1}
  \binom{\usedim - 1}{j} \hcard_j \big \{\int_0^1 \rtime^j (1 -
  \rtime)^{\usedim - 1 - j} d \rtime \} \; \stackrel{(i)}{=}
  \frac{1}{\usedim} \sum_{j = 0}^{\usedim - 1} \hcard_j
  \stackrel{(ii)}{=} \usedim e_1 - e_\usedim \; \stackrel{(iii)}{=}
  \TC(\Prob_Z),
\end{align*}
where step (i) uses the beta-integral identity $\int_0^1
\binom{\usedim - 1}{j} t^j (1 - t)^{\usedim - 1 - j} \, dt =
\frac{1}{\usedim}$; step (ii) uses the identity~\eqref{EqnHardE} along
with some algebra; and step (iii) follows from the
definition~\eqref{EqnDefnTC} of the total correlation, combined with
the definition~\eqref{EqnDefnTSE} of the entropy coefficients.

Finally, integrating by parts shows that $\int_0^1 \hfun(t) dt =
\int_0^1 (1-t) \hfun'(t) dt$, where the boundary term vanishes since
$\hfun(0) = 0$.  This completes the proof of the $\TC$ representation
in equation~\eqref{EqnTCRep}.

\section{Proofs of connecting results} 
\label{SecCompare}

In this section, we prove various results that connect the \ugc
aggregate complexity to other complexity notions
(\Cref{SecProofPropSimpleRelation}), and compare the
DHW $\tau$-leaping unmasking algorithm with the Bernoulli unmasking
procedure (\Cref{SecDHWRelation}).


\subsection{Proof of~\Cref{PropSimpleRelation}}
\label{SecProofPropSimpleRelation}

The proposition consists of three claims, and we prove each of them in
turn.

\subsubsection{Proof of $\TSE$ identity~\eqref{EqnUGCtoTSE}}

Beginning with the definition~\eqref{EqnDefnUGC}, integration by parts
yields $\Hinfo(0, 1) = \int_0^1 (2 \rtime - 1) \hfun(\rtime) \,
d\rtime$, where the boundary term vanishes since $\rtime (1 - \rtime)
= 0$ at the endpoints.  Substituting the Bernstein
representation~\eqref{EqnHfunBernstein} of $\hfun$ into this equation
yields
\begin{align}
\label{EqnBlackVelvet}
  \Hinfo(0, 1) & = \sum_{j = 0}^{\usedim - 1} \hcard_j \binom{\usedim
    - 1}{j} \left \{ \int_0^1 (2 \rtime - 1) \rtime^j (1 -
  \rtime)^{\usedim - 1 - j} d \rtime \right \} \; = \;
  \frac{1}{\usedim (\usedim +1)} \sum_{j = 0}^{\usedim - 1} (2 j -
  \usedim + 1) \; \hcard_j,
\end{align}
where the second step follows from a standard beta-integral identity.

To complete the proof, we need to relate the entropy coefficients $e_j =
\Exs[\Ent(Z_{A_j})]$, where $A_j$ is a uniformly random subset of
cardinality $j$, to the $\hcard_j$ coefficients.  For each $j = 0,
\ldots, \usedim - 1$, we can write the mutual information as
$\Info(\Zvar_i; \Zvar_A) = \Ent(\Zvar_i) + \Ent(\Zvar_A) -
\Ent(\Zvar_{A \cup \{ i \}})$. Averaging this identity over $i$ and
the $j$-element subsets $A \subseteq [\usedim] \setminus \{ i \}$
yields the relation $\hcard_j = \usedim \big( e_1 + e_j - e_{j + 1}
\big)$.  Substituting this identity into the
expression~\eqref{EqnBlackVelvet} yields
\begin{align*}
  \Hinfo(0, 1) = \frac{1}{\usedim + 1} \Big \{ e_1 \underbrace{\sum_{j
      = 0}^{\usedim - 1} (2 j - \usedim + 1)}_{=0} + \sum_{j =
    0}^{\usedim - 1} (2 j - \usedim + 1) (e_j - e_{j + 1}) & =
  \frac{1}{\usedim + 1} \left\{ 2 \sum_{j = 1}^{\usedim - 1} e_j -
  (\usedim - 1) e_\usedim \right\} \\
  & = \frac{2}{\usedim + 1} \sum_{j = 1}^{\usedim} \left( e_j -
  \frac{j}{\usedim} e_\usedim \right) = \frac{2}{\usedim + 1}
  \TSE(\Prob_Z),
\end{align*}
which establishes the claim~\eqref{EqnUGCtoTSE}.

\subsubsection{Proof of the bound~\eqref{EqnCard2Ber}}
\label{SecProofCard2Ber}

For any $0 \leq p < q \leq 1$, define the multinomial triple $(U_0,
U_1, U_2) \sim \operatorname{Multinomial}\left( \usedim + 1;\, p,\, q
- p,\, 1 - q \right)$.  As claimed previously, we have the identity
\begin{align}
\label{EqnBerCardTwo}
\Hinfo(p, q) & = \frac{\usedim}{\usedim + 1} \Exs \left[ \Hcard(A_U,
  B_U) \right],
\end{align}
where $A_U \defn \max\{ U_0 - 1, 0 \}$, and $B_U \defn \min\{ U_0 +
U_1, \usedim \}$.

The claims~\eqref{EqnCard2Ber} follow as a special case.  For $(p, q)
= (0, 1)$, we have $(A_U, B_U) = (0, \usedim)$ almost surely, and
hence $\Hinfo(0, 1) = \frac{\usedim}{\usedim + 1} \Hcard(0, \usedim)$,
from which it follows that $\Hinfo(0, 1) \leq \Hcard(0, \usedim) \leq
2 \Hinfo(0, 1)$, as claimed.

We now prove the claim~\eqref{EqnBerCardTwo}.  Introduce the shorthand
notation $\Delta_j \defn \hcard_j - \hcard_{j - 1}$ and $c_j \defn
\frac{j}{\usedim} \left( 1 - \frac{j}{\usedim} \right)$, and observe
that
\begin{subequations}
\begin{align}
  \label{EqnHcardSimple}
  \Hcard(a,b) = \sum_{j=a+1}^{b-1} c_j \Delta_j.
\end{align}
Moreover, using the representation~\eqref{EqnHfunPrimeBernstein} of
$\hfun'$ as a Bernstein polynomial, we can write
\begin{align}
\label{EqnBerExp}  
  \rtime (1 - \rtime) \hfun'(\rtime) & = \sum_{j = 1}^{\usedim - 1}
  c_j \Delta_j \, \usedim \binom{\usedim}{j} \rtime^j (1 -
  \rtime)^{\usedim - j}.
\end{align}
Generate $\usedim + 1$ i.i.d. $\operatorname{Unif}[0, 1]$ random
variables $Y_i$, and let $Y_{(j + 1)}$ denote the $(j + 1)$-st order
statistic. Its density is given by $(\usedim + 1) \binom{\usedim}{j}
\rtime^j (1 - \rtime)^{\usedim - j}$. Consequently, by integrating the
Bernstein expansion~\eqref{EqnBerExp} over $[p, q]$, we obtain
\begin{align}
  \label{EqnInterh}
  \Hinfo(p, q) & = \frac{\usedim}{\usedim + 1} \sum_{j = 1}^{\usedim -
    1} c_j \Delta_j \Prob\big( Y_{(j + 1)} \in (p, q) \big).
\end{align}
\end{subequations}
Now recalling our i.i.d. uniform samples $\{Y_i\}_{i=1}^{\usedim +
  1}$, define the multinomial count vector $U = (U_0, U_1, U_2)$ via
\begin{align*}
U_0 \defn \card \{ \ell \mid Y_\ell \in [0,p] \}, \quad U_1 \defn
\card \{ \ell \mid Y_\ell \in (p, q)\} \quad \mbox{and} \quad U_2
\defn \card \{ \ell \mid Y_\ell \in [q, 1] \}.
\end{align*}
By construction, we have $Y_{(j + 1)} \in (p, q)$ if and only if $A_U
< j < B_U$. Thus, we can rewrite the expression~\eqref{EqnInterh} as
\begin{align*}
  \Hinfo(p, q) = \frac{\usedim}{\usedim + 1} \Exs\left[ \sum_{j =
      1}^{\usedim - 1} c_j \Delta_j \mathbf{1}\{ A_U < j < B_U \}
    \right] & = \frac{\usedim}{\usedim + 1} \Exs\left[ \sum_{j = A_U +
      1}^{B_U - 1} c_j \Delta_j \right] \\
  & = \frac{\usedim}{\usedim + 1} \Exs\big[ \Hcard(A_U, B_U) \big],
\end{align*}
where the last step uses the representation~\eqref{EqnHcardSimple}.
This completes the proof of equation~\eqref{EqnBerCardTwo}.


\subsubsection{Proof of the $\DHW$ sandwich~\eqref{EqnUGC2DHW}}

For $u \geq 0$, we define a reveal time $r = e^{-u}$, so that $u =
\log(1/r)$ is the log-inverse-reveal time.  The analysis of Dmitriev
et al.~\cite{DmiEtAl26} involves a forward masked vector $\Yvar_u$; it
is related to our reveal process $\Xvar_r$ via the relation $\Yvar_u
\overset{\mathrm{law}}{=} \Xvar_{e^{-u}}$.  In the process $\Yvar_u$,
each coordinate $\Yvar_u^i$ equals $\Zvar_i$ with probability $r =
e^{-u}$ and equals $\miss$ otherwise, and they define the
conditional-information density
\begin{align*}
  \Difun(u) & \defn \sum_{i \neq j}
    \Info(\Yvar_u^i; \Yvar_u^j \mid \Yvar_u^{-(i, j)}),
\end{align*}
where $\Yvar_u^{-(i, j)}$ denotes the collection of unmasked
coordinates other than $i$ and $j$.  They define the effective total
correlation by $\DHW(\Prob_Z) \defn \int_0^\infty \min\{ 1, u \}
\Difun(u) \, du$.

In order to prove the sandwich relation, it suffices to prove the
identity
\begin{subequations}
\begin{align}
  \label{EqnKeyRelation}  
  \Difun(u) & = e^{-2u} \hfun'(e^{-u}) \qquad \mbox{for all $u \geq
    0$.}
\end{align}
Indeed, when this identity holds, the change of variables $t = e^{-u}$
guarantees that
\begin{align}
  \DHW(\Prob_Z) & = \int_0^1
    \weight(\rtime) \hfun'(\rtime) \, d \rtime, \qquad
    \mbox{where $\weight(\rtime) \defn \rtime \min\{ 1,
    -\log \rtime \}$.}
\end{align}
\end{subequations}
In terms of the shorthand $v(t) \defn t \, (1 - t)$, it is
straightforward to verify that
\begin{align*}
v(t) \; \leq \; w(t) \; \leq \; \frac{e}{e - 1} v(t) \qquad \mbox{for
  all $t \in [0,1]$,}
\end{align*}
from which the sandwich~\eqref{EqnUGC2DHW} follows.

\paragraph{Proof of the identity~\eqref{EqnKeyRelation}:}
We begin by relating their definitions to equivalent objects in our
notation.  Fix an ordered pair of indices $i \neq j$ and let $A_r^{i,
  j}$ be the random subset of $[\usedim] \setminus \{ i, j \}$
obtained by retaining each coordinate independently with probability
$r$.  If either $i$ or $j$ is masked, one of the first two arguments
in the conditional mutual information is deterministic, and its
contribution is zero.  The probability that both coordinates are
retained is $r^2$.  Conditional on this event and on $A_r^{i, j} = A$,
the three arguments reduce to $\Zvar_i$, $\Zvar_j$, and $\Zvar_A$,
respectively.  Since the masking pattern is independent of $\Zvar$,
the law of total expectation gives
\begin{align*}
  \Info(\Yvar_u^i; \Yvar_u^j \mid \Yvar_u^{-(i, j)})
  & = r^2 \Exs\big[
    \Info(\Zvar_i; \Zvar_j \mid \Zvar_{A_r^{i, j}}) \big].
\end{align*}
Moreover,
$\Prob(A_r^{i, j} = A) = r^{\card(A)} (1 - r)^{\usedim - 2 -
  \card(A)}$.  Expanding the expectation and summing over the ordered
pairs therefore yields
\begin{subequations}
\begin{align}
  \label{EqnDHWRep}
  \Difun(u) & \defn r^2 \sum_{i \neq j} \sum_{A \subseteq [\usedim]
    \setminus \{ i, j \}} r^{\card(A)} (1 - r)^{\usedim - 2 -
    \card(A)} \Info(\Zvar_i; \Zvar_j \mid \Zvar_A) \qquad \mbox{with
    $r = e^{-u}$.}
\end{align}
In order to complete the proof, it suffices to show that $\hfun'$
satisfies the identity
\begin{align}
\label{EqnHprimeRep}  
  \hfun'(r) & = \sum_{i \neq j} \sum_{A \subseteq [\usedim] \setminus
    \{ i, j \}} r^{\card(A)} (1 - r)^{\usedim - 2 - \card(A)}
  \Info(\Zvar_i; \Zvar_j \mid \Zvar_A).
\end{align}
\end{subequations}
Our claim~\eqref{EqnKeyRelation} then follows by comparing the two
representations~\eqref{EqnHprimeRep} and~\eqref{EqnDHWRep}.

To prove the identity~\eqref{EqnHprimeRep}, fix an integer $k \in \{
0, \ldots, \usedim - 2 \}$.  Beginning with the
definition~\eqref{EqnDefnHcardGain} of $\hcard_k$ and averaging over
the choice of the additional coordinate, we find that
\begin{align*}
  (\usedim - 1) \binom{\usedim - 2}{k} \big\{ \hcard_{k + 1} -
  \hcard_k \big\} & = \sum_{i \neq j} \sum_{\substack{ A \subseteq
      [\usedim] \setminus \{ i, j \} \\ \card(A) = k}}
  \underbrace{\Big\{ \Info(\Zvar_i; \Zvar_{A \cup \{ j \}}) -
    \Info(\Zvar_i; \Zvar_A) \Big\}}_{\Info(\Zvar_i; \Zvar_j \mid
    \Zvar_A)},
\end{align*}
where the underbrace relation follows from the chain rule for mutual
information.  Substituting this identity into our Bernstein
representation~\eqref{EqnHfunPrimeBernstein} for $\hfun'$, and then
summing over $k$ yields equation~\eqref{EqnHprimeRep}.

\subsection{Relation with the DHW $\tau$-leaping masking sampler}
\label{SecDHWRelation}

In this section, we formalize the exact relationship between
Algorithm~1 of Dmitriev et al.~\cite{DmiEtAl26} and the Bernoulli
unmasking sampler~\eqref{EqnBerUnmask}.  Any DHW grid of the form $0 =
u_0 < \cdots < u_N = U$ can be used to define a sequence of reveal
times $\rtime_j \defn e^{-(U - u_j)}$, and the corresponding inverse
transformation $U - u_j = \log(1/\rtime_j)$.

In terms of the shorthand $\beta_k \defn (\rtime_{k + 1} - \rtime_k)/(1
- \rtime_k)$, for the Bernoulli unmasking update with estimated
denoiser $\DenoiseHat_i(\mathord\cdot, x)$, the
corresponding transition kernel is given by
\begin{subequations}
  \begin{align}
    \label{EqnBerKernel}
  \Ker_{k, i}^{\Ber}(\miss \mid x) = 1 - \beta_k, \quad \mbox{and}
  \quad \Ker_{k, i}^{\Ber}(a \mid x) = \beta_k
  \DenoiseHat_i(a, x)
  \qquad \mbox{for each $a \in \Alphabet$.}
\end{align}
On the other hand, for a current state $x$ and masked coordinate $i
\in \Masked(x)$, write the DHW estimated score as $q_{k, i}^{\DHW}(a,
x) \defn \scorehat_{\log(1/\rtime_k)}(x^{i \to a}, x)$.  Assuming its
total mass is positive, define its normalized shape and relative mass
by
\begin{align}
\label{EqnDHWBasics}
  \pi_{k, i}(a, x) \defn \frac{q_{k, i}^{\DHW}(a, x)}{ \sum_{b \in
      \Alphabet} q_{k, i}^{\DHW}(b, x)}, \quad \mbox{and} \quad c_{k,
    i}(x) \defn \frac{1 - \rtime_k}{\rtime_k} \sum_{b \in \Alphabet}
  q_{k, i}^{\DHW}(b, x).
\end{align}
These ingredients along with the DHW stepsize yield the kernel
\begin{align}
  \label{EqnDHWKernel}
  \Ker_{k, i}^{\DHW}(\miss \mid x) = (1 - \beta_k)^{c_{k, i}(x)},
  \quad \mbox{and} \quad \Ker_{k, i}^{\DHW}(a \mid x) = \left\{ 1 - (1
  - \beta_k)^{c_{k, i}(x)} \right\} \pi_{k, i}(a, x).
\end{align}
Taking the differences between the expressions~\eqref{EqnBerKernel}
and~\eqref{EqnDHWKernel}, we obtain the kernel differences
\begin{align}
  \label{EqnDHWMaskKernelGap}
  \Ker_{k, i}^{\DHW}(\miss \mid x) - \Ker_{k, i}^{\Ber}(\miss \mid x) & = (1
  - \beta_k)^{c_{k, i}(x)} - (1 - \beta_k), \\
  \label{EqnDHWTokenKernelGap}
  \Ker_{k, i}^{\DHW}(a \mid x) - \Ker_{k, i}^{\Ber}(a \mid x) & =
  \left\{ 1 - (1 - \beta_k)^{c_{k, i}(x)} \right\} \pi_{k, i}(a, x) -
  \beta_k \DenoiseHat_i(a, x).
\end{align}
Thus, the two kernels are very closely related, but not identical in
general.

We claim that when there is no score error, so that $q_{k, i}^{\DHW} =
q_{k, i}^\star$ and $\DenoiseHat_i = \Denoise_i$, then the two
kernels coincide.
First of all, we claim that 
Proposition~6 of Dmitriev et al.~\cite{DmiEtAl26}, when translated
into our notation, shows that
\begin{align}
  \label{EqnWeiExactMaskingScore}
  q_{k, i}^\star(a, x) & \stackrel{(i)}{=} \frac{\rtime_k}{1 -
    \rtime_k} \Denoise_i(a, x), \quad \mbox{and} \quad \sum_{a \in
    \Alphabet} q_{k, i}^\star(a, x) \stackrel{(ii)}{=}
  \frac{\rtime_k}{1 - \rtime_k}.
\end{align}
\end{subequations}
Equation (ii) follows from equation (i), since the denoiser
$\Denoise_i$ is a conditional probability distribution.  To establish
equation (i), if $x$ has coordinate $i$ masked and $y$ is obtained by
filling that coordinate with $a$, then the marginal convention in
Proposition~6 gives $q_0(y) / q_0(x) = \Prob(\Zvar_i = a \mid x) =
\Denoise_i(a, x)$.  Moreover, their CTMC time $u$ is related to our
reveal probability by the transformation $\rtime_k = e^{-u}$, so that
$(e^u - 1)^{-1} = \rtime_k / (1 - \rtime_k)$.

Using equation~\eqref{EqnWeiExactMaskingScore} we see that the exact
score has total mass $\rtime_k/(1 - \rtime_k)$, so that $c_{k, i}(x) =
1$ in equation~\eqref{EqnDHWBasics}. Moreover, normalizing $q_{k,
  i}^\star(a, x)$ by this total mass cancels the common factor
$\rtime_k/(1 - \rtime_k)$, so that $\pi_{k, i}(a, x) = \Denoise_i(a,
x)$.  Thus, the right-hand side of
equation~\eqref{EqnDHWMaskKernelGap} is $(1 - \beta_k) - (1 - \beta_k)
= 0$, while the right-hand side of
equation~\eqref{EqnDHWTokenKernelGap} is $\beta_k \Denoise_i(a, x) -
\beta_k \Denoise_i(a, x) = 0$.

For learned scores, the mask-probability gap is caused by the non-unit
score mass $c_{k, i}(x)$.  The important difference is that rescaling
the DHW score changes its kernel, whereas the Bernoulli reveal
probability remains fixed by the grid.


\section{Proof of~\Cref{LemXORSAT}: XORSAT scaling}
\label{SecXORSAT}

Let us first complete the proof, taking the bounds~\eqref{EqnXORWidth}
as given.  Note that the outer blocks have log-reveal-odds length
$\Order(\log \usedim)$, whereas $\MyLen(\IntMid) = \Order(\sqrt{\log
  \usedim / \usedim})$.  Hence, we can compute
\begin{align*}
  \PartComp(\Partition) & = \left( \sqrt{\MyLen(\IntLeft)
    \Hinfo(\IntLeft)} + \sqrt{\MyLen(\IntMid) \Hinfo(\IntMid)} +
  \sqrt{\MyLen(\IntRight) \Hinfo(\IntRight)} \right)^2 \\ & =
  \Order\left( \left[ \sqrt{\frac{\log \usedim}{\usedim^{10}}} +
    \left( \frac{\log \usedim}{\usedim} \right)^{1/4} \sqrt{\usedim}
    \right]^2 \right) \\ & = \Order\big( \sqrt{\usedim \log \usedim}
  \big),
\end{align*}
where we have used the claimed bound~\eqref{EqnXORWidth}.  Combining
this bound with the scaling $\PartComp([\IntStar]) = \Theta(\usedim
\log \usedim)$ establishes the claim~\eqref{EqnXORSAT}.

We now return to prove the bounds~\eqref{EqnXORWidth}. With the
shorthand $\IntMid = [\sneg, \splus]$ and $c' = \min \big \{ \sneg (1-
\sneg), \spos (1 - \spos) \}$, it follows from the
definition~\eqref{EqnDefnUGC} of $\Hinfo$ that
  \begin{subequations}
\label{EqnRain}    
    \begin{align}
\label{EqnRainOne}
  c' \{ \hfun(\splus) - \hfun(\sneg) \} \leq \Hinfo(\IntMid) \leq
  \frac{\hfun(1) - \hfun(0)}{4}, \quad \mbox{as well as} \\
  \label{EqnRainTwo}
  \Hinfo(\IntLeft) \leq \frac{\hfun(\sneg) - \hfun(0)}{4} \quad
  \mbox{and} \quad \Hinfo(\IntRight) \leq \frac{\hfun(1) -
    \hfun(\splus)}{4}.
    \end{align}
  \end{subequations}
Consequently, we can control the \ugc increments by controlling the
differences in $\hfun$.  From the definition $\hfun(1) = \sum_{i=1}^d
\Info(Z_i; Z_{-i})$ and the binary nature of each $Z_i$, it follows that
$\hfun(1) \leq d \log 2$, and hence that \mbox{$\hfun(1) - \hfun(0) =
  \Order(\usedim)$.}

For controlling the other increments, we study averages of $\hfun$
over the random choice of $\Amat \in \{0,1\}^{\usedim \times \kdim}$.
The increments of the averaged $\hfun$ take a simple form.  For a
positive integer $m \leq \kdim$, let $\Bmat_m \in \{0,1\}^{m \times
  \kdim}$ be a Boolean random matrix with $m$ independent uniform rows
in $\{0, 1\}^\kdim$, and let $b$ be a fresh independent uniform row.
Defining the function
\begin{subequations}
\begin{align}
\label{EqnDefnUfun}  
  \ufun(m) & \defn \Prob\{ b \notin \operatorname{rowspan}(\Bmat_m) \}
\end{align}
we claim that
\begin{align}
  \label{EqnXORSATHfunIdentity}
  \frac{\Exs_\Amat[\hfun(s) - \hfun(t)]}{d \, \log 2} & = \Exs\big[
    \ufun(M_t) - \ufun(M_s) \big], \qquad \mbox{where $M_r \sim
    \operatorname{Bin}(\usedim - 1, r)$ is a binomial RV,}
\end{align}
and also
\begin{align}
  \label{EqnXORSATEndpointEstimates}
  \Exs[\ufun(M_{\sneg})] = 1 - O(d^{-11}) \quad \mbox{and} \quad
  \Exs[\ufun(M_{\splus})] = O(d^{-11}).
\end{align}
\end{subequations}
Taking these two claims as given for the moment, let us complete the
proof of the bounds~\eqref{EqnXORWidth}.

\paragraph{Transition interval:}
Equation~\eqref{EqnXORSATEndpointEstimates} gives $\Exs\big[
  \ufun(M_{\sneg}) - \ufun(M_{\splus}) \big] = 1 - O(d^{-11}) =
\Theta(1)$.  Consequently, equation~\eqref{EqnXORSATHfunIdentity}
gives $\Exs_\Amat[\hfun(\splus) - \hfun(\sneg)] = \Theta(d)$, and
combining with inequality~\eqref{EqnRainOne}, this yields
$\Exs[\Hinfo(\IntMid)] = \Theta(d)$.

\paragraph{Pre- and post-transition intervals:}  For the
pre-transition interval, $M_0 = 0$ deterministically, so that
equation~\eqref{EqnXORSATHfunIdentity} gives $\Exs_\Amat[\hfun(\sneg)]
- \Exs_\Amat[\hfun(0)] = d (\log 2) \big\{ \ufun(0) -
\Exs[\ufun(M_{\sneg})] \big\}$. Since $\ufun$ is non-increasing and $0
\leq \ufun \leq 1$, it follows that
\begin{align*}
  0 & \leq \ufun(0) - \Exs[\ufun(M_{\sneg})] \; \leq \; 1 -
  \Exs[\ufun(M_{\sneg})] = O(d^{-11}),
\end{align*}
and hence $\Exs_\Amat[\hfun(\sneg)] - \Exs_\Amat[\hfun(0)] =
O(d^{-10})$.

Similarly, we have $M_1 = d - 1$ deterministically, and
equation~\eqref{EqnXORSATHfunIdentity} gives $\Exs_\Amat[\hfun(1)] -
\Exs_\Amat[\hfun(\splus)] = d (\log 2) \big\{ \Exs[\ufun(M_{\splus})]
- \ufun(d - 1) \big\}$.  Using monotonicity and non-negativity,
\begin{align*}
  0 & \leq \Exs[\ufun(M_{\splus})] - \ufun(d - 1) \; \leq \;
  \Exs[\ufun(M_{\splus})] = O(d^{-11}),
\end{align*}
and hence $\Exs_\Amat[\hfun(1)] - \Exs_\Amat[\hfun(\splus)] =
O(d^{-10})$.  Combining these two estimates with
inequalities~\eqref{EqnRainTwo}, we conclude that
$\Exs[\Hinfo(\IntLeft)] = \order(d^{-10})$ and
$\Exs[\Hinfo(\IntRight)] = \order(d^{-10})$. \\

To pass from these averaged bounds to a single realization of $\Amat$,
first note that the two outer $\Hinfo$ terms are nonnegative.  Hence,
by Markov's inequality and the preceding expectation bounds, with
probability bounded below by a positive constant, both
$\Hinfo(\IntLeft)$ and $\Hinfo(\IntRight)$ are $\Order(d^{-10})$.  For
the middle interval, the endpoint
estimates~\eqref{EqnXORSATEndpointEstimates} and
identity~\eqref{EqnXORSATHfunIdentity} give $\Exs_\Amat\big[
  \hfun(\splus) - \hfun(\sneg) \big] = d (\log 2) \big\{ 1 -
O(d^{-11}) \big\}$.  Since $0 \leq \hfun(\splus) - \hfun(\sneg) \leq d
\log 2$, Markov's inequality applied to the nonnegative deficit $d
\log 2 - \{ \hfun(\splus) - \hfun(\sneg) \}$ shows that
\begin{align*}
  \Prob_\Amat\left\{ \hfun(\splus) - \hfun(\sneg) < \frac{d \log
    2}{2} \right\} & = O(d^{-11}).
\end{align*}
Thus, by a union bound, for all sufficiently large $d$ there exists a
realization of $\Amat$ for which the two outer bounds hold and
$\hfun(\splus) - \hfun(\sneg) = \Omega(d)$ simultaneously.
Inequality~\eqref{EqnRainOne} then gives $\Hinfo(\IntMid) = \Theta(d)$
for this same realization.


\paragraph{Proof of the identity~\eqref{EqnXORSATHfunIdentity}:}

Fix a coordinate $i$, and condition on coordinate $i$ being masked at
reveal probability $r$.  Let $R \subseteq [\usedim] \setminus \{i\}$
denote the set of revealed coordinates among the remaining
coordinates.  Then $|R| = M_r \sim \operatorname{Bin}(d - 1, r)$.
Coordinate $i$ is associated with the $i^{th}$ row $a_i$ of
$\Amat$, and we use $\Amat_R$ to denote the matrix formed by the rows
indexed by $R$.  All of the following linear-algebraic statements are
over the Boolean field $\mathbb F_2$.

Conditional on $\Amat$ and the revealed values $\Zvar_R$, there are
two possibilities.  If $a_i \in \operatorname{rowspan}(\Amat_R)$, then
$\Zvar_i$ is a linear combination of the revealed coordinates
$\Zvar_R$, and hence is determined by them.  On the other hand, if
$a_i \notin \operatorname{rowspan}(\Amat_R)$, then the linear
functional defined by $a_i$ is nonzero on the nullspace of $\Amat_R$.
Consequently, among the latent assignments compatible with $\Zvar_R$,
exactly half give $\Zvar_i = 0$ and half give $\Zvar_i = 1$.  By this
reasoning, the conditional entropy takes the simple form
\begin{align}
  \Ent(\Zvar_i \mid \Zvar_R, \Amat) & = (\log 2) \, \mathbf{1}\big\{
  a_i \notin \operatorname{rowspan}(\Amat_R) \big\}.
\end{align}

Now condition on the event $|R| = m$.  By independence and permutation
invariance of the rows, the submatrix $\Amat_R$ has the same law as
$\Bmat_m$, while $a_i$ is a fresh independent uniform row.  From the
definition~\eqref{EqnDefnUfun} of $\ufun$, it follows that
$\Exs_\Amat\big[ \Ent(\Zvar_i \mid \Zvar_R, \Amat) \,\big|\, |R| = m
  \big] = (\log 2) \ufun(m)$.  Averaging this identity over the
binomial random variable yields $\Exs_\Amat\big[ \Ent(\Zvar_i \mid
  \Zvar_R, \Amat) \big] = (\log 2) \Exs\big[ \ufun(M_r) \big]$.  The
same argument with no revealed coordinates gives $\Exs_\Amat\big[
  \Ent(\Zvar_i \mid \Amat) \big] = (\log 2) \ufun(0)$.

Combining the pieces, the $i^{th}$ term in the sum~\eqref{EqnDefnHfun}
that defines $\hfun$ has expectation
\begin{align*}
  \Exs_\Amat\big[ \Ent(\Zvar_i \mid \Amat) - \Ent(\Zvar_i \mid
    \Zvar_R, \Amat) \big] = (\log 2) \big\{ \ufun(0) -
  \Exs[\ufun(M_r)] \big\}.
\end{align*}
Summing over the $d$ exchangeable coordinates yields the identity
$\Exs_\Amat[\hfun(r)] = d (\log 2) \big\{ \ufun(0) - \Exs[\ufun(M_r)]
\big\}$.  Taking the difference between $r = s$ and $r = t$ in this
identity yields the claim~\eqref{EqnXORSATHfunIdentity}.


\paragraph{Proof of equation~\eqref{EqnXORSATEndpointEstimates}:}

We first claim that $\ufun$ satisfies the bounds
  \begin{align}
    \label{EqnBone}
    1 - \ufun(m) \; \stackrel{(a)}{\leq} 2^{m - \kdim} \qquad
    \mbox{for $m \leq \kdim$, $\quad$ and} \qquad \ufun(m) \stackrel{(b)}{
      \leq} 2^{\kdim - m} \qquad \mbox{for $m > k$.}
  \end{align}
To prove the first bound~\eqref{EqnBone}(a), we observe that
conditional on $\Bmat_m$, its row span contains
$2^{\operatorname{rank}(\Bmat_m)}$ vectors out of the $2^{\kdim}$
possible rows in $\{0,1\}^{\kdim}$.  Hence a fresh uniform row lies in
this span with probability at most $2^{\operatorname{rank}(\Bmat_m) -
  \kdim} \leq 2^{m - \kdim}$. Combining with the definition of $\ufun$
yields the claim.  As for the bound~\eqref{EqnBone}(b), when $m >
\kdim$, the event that the fresh row lies outside the row span
requires $\operatorname{rank}(\Bmat_m) < \kdim$.  Union-bounding over
nonzero right-null vectors therefore gives $\ufun(m) \leq \Prob \big[
  \operatorname{rank}(\Bmat_m) < \kdim \big] \leq (2^{\kdim} - 1)
2^{-m} < 2^{\kdim - m}$, as claimed.

We now show that $\Exs[\ufun(M_{\sneg})] = 1 - \order(\usedim^{-11})$.
Recalling that $M_{\sneg} \sim \operatorname{Bin}(\usedim - 1,
\sneg)$, its mean is
\begin{align*}
  \mu_- \defn \Exs[M_{\sneg}] & = (\usedim - 1) \sneg = \kdim -
  \frac{\kdim}{\usedim} - c_1 \sqrt{\usedim \log \usedim} + c_1
  \sqrt{\frac{\log \usedim}{\usedim}}.
\end{align*}
For the threshold $T_- \defn \kdim - (c_1 / 4) \sqrt{\usedim \log
  \usedim}$, we therefore have
\begin{align*}
  T_- - \mu_- & = \frac{3c_1}{4} \sqrt{\usedim \log \usedim} +
  \frac{\kdim}{\usedim} - c_1 \sqrt{\frac{\log \usedim}{\usedim}} \geq
  \frac{c_1}{2} \sqrt{\usedim \log \usedim}
\end{align*}
for all sufficiently large $\usedim$.  Defining the bad event $\Event
= \{ M_{\sneg} > T_- \}$, Hoeffding's inequality gives
\begin{align*}
  \Prob(\Event)
\leq \exp\left\{ -\frac{2 (T_- - \mu_-)^2}{\usedim - 1} \right\} \leq
\usedim^{-c_1^2 / 2},
\end{align*}
so that $\Prob(\Event) = \order(\usedim^{-11})$ for $c_1$ sufficiently
large.  On the complement $\Event^c$, the bound~\eqref{EqnBone}(a)
guarantees that $1 - \ufun(M_{\sneg}) \leq 2^{M_{\sneg} - \kdim} \leq
2^{-(c_1 / 4) \sqrt{\usedim \log \usedim}}$.  Since $0 \leq \ufun(m)
\leq 1$, it follows that
\begin{align*}
  1 - \Exs[\ufun(M_{\sneg})] & \leq \Prob(\Event) + 2^{-(c_1 / 4)
    \sqrt{\usedim \log \usedim}} = \order(\usedim^{-11}),
\end{align*}
which proves the claim.

A similar argument applies at the right endpoint $\splus$.  The same
Hoeffding bound guarantees that \mbox{$\Prob \big[ M_{\splus} < \kdim
    + (c_1 / 4) \sqrt{\usedim \log \usedim} \big] = O(\usedim^{-11})$,}
while on the complementary event, we have the bound
\begin{align*}
  \ufun(M_{\splus}) \leq 2^{\kdim - M_{\splus}} \leq 2^{-(c_1 / 4)
    \sqrt{\usedim \log \usedim}}.
\end{align*}
Taking expectations gives $\Exs[\ufun(M_{\splus})] =
O(\usedim^{-11})$. \\

\noindent Combining the two pieces completes the proof of the claimed
bounds~\eqref{EqnXORSATEndpointEstimates}.


\end{document}